\documentclass[reqno, 11pt]{amsart}

\usepackage[utf8]{inputenc} 
\usepackage[T1]{fontenc}    
\usepackage{hyperref}       
\usepackage{url}            
\usepackage{booktabs}       
\usepackage{amsfonts}       
\usepackage{nicefrac}       
\usepackage{microtype}      
\usepackage{xcolor}         

\usepackage{caption,subfigure}
\usepackage{amsmath,mathrsfs,amsthm,amssymb,prettyref,graphicx,float,xfrac,fullpage}

\numberwithin{equation}{section}

\theoremstyle{plain}
\newtheorem{thm}{Theorem}[section]
\newtheorem{lem}{Lemma}[section]
\newtheorem{prop}[thm]{Proposition}
\newtheorem{fact}[thm]{Fact}

\theoremstyle{definition}
\newtheorem{defn}[thm]{Definition}
\theoremstyle{remark}
\newtheorem{rem}{Remark}

\newcommand{\cA}{\mathcal{A}}

\newcommand{\cM}{\mathscr{M}}
\newcommand{\cN}{\mathcal{N}}

\newcommand{\cL}{\mathcal{L}}

\newcommand{\cX}{\mathcal{X}}
\newcommand{\cY}{\mathcal{Y}}

\newcommand{\R}{\mathbb{R}}

\newcommand{\E}{\mathbb{E}}

\newcommand{\tensor}{\otimes}

\newcommand{\op}{\operatorname{op}}

\newcommand{\abs}[1]{\lvert#1\rvert}
\newcommand{\norm}[1]{\lvert\lvert#1\rvert\rvert}

\newcommand{\bfu}{\mathbf{u}}

\newcommand{\tbfu}{\tilde{\mathbf{u}}}

\newcommand{\bg}{\mathbf{g}}
\newcommand{\bu}{\mathbf{u}}
\newcommand{\bff}{\mathbf{f}}
\newcommand{\logit}{\operatorname{logit}}

\title{High-dimensional limit theorems for SGD: \\ Effective dynamics and critical scaling}

\begin{document}
\title{High-dimensional limit theorems for SGD: \\ Effective dynamics and critical scaling}

\author{G\'erard Ben Arous}
\author{Reza Gheissari}
\author{Aukosh Jagannath}

\address[G\'erard Ben Arous]{Courant Institute, New York University}
\email{benarous@cims.nyu.edu}

\address[Reza Gheissari]{Department of Mathematics, Northwestern University}
\email{gheissari@northwestern.edu}

\address[Aukosh Jagannath]{Department of Statistics and Actuarial Science, Department of Applied Mathematics, and Cheriton School of Computer Science, University of Waterloo}
\email{a.jagannath@uwaterloo.ca}

\begin{abstract}
We study the scaling limits of stochastic gradient descent (SGD) with constant step-size in the high-dimensional regime. We prove limit theorems for the trajectories of summary statistics (i.e., finite-dimensional functions) of SGD as the dimension goes to infinity. Our approach allows one to choose the summary statistics that are tracked, the initialization, and the step-size. It yields both ballistic (ODE) and diffusive (SDE) limits, with the limit depending dramatically on the former choices. 
We show a critical scaling regime for the step-size, below which the effective ballistic dynamics matches gradient flow for the population loss, but at which, a new correction term appears which changes the phase diagram.
About the fixed points of this effective dynamics, the corresponding diffusive limits can be quite complex and even degenerate. 
We demonstrate our approach on popular examples including estimation for spiked matrix and tensor models and classification via two-layer networks for binary and XOR-type Gaussian mixture models. 
These examples exhibit surprising phenomena including multimodal timescales to convergence as well as convergence to sub-optimal solutions with probability bounded away from zero from random (e.g., Gaussian) initializations. At the same time, we demonstrate the benefit of overparametrization by showing that the latter probability goes to zero as the second layer width grows.
\end{abstract}

\maketitle

\part{Introduction and main results}

\section{Introduction}
Stochastic gradient descent (SGD) is the go-to method for large-scale optimization problems in modern data science. 
It is often used to train complex parametric models on high-dimensional data.
Since its introduction in~\cite{RobMon51}, there has been a tremendous amount of work in analyzing its evolution.

In fixed dimensions, the asymptotic theory of SGD, and stochastic approximations more broadly, is by now classical. There have been works on path-wise limit theorems, such as functional central limit theorems and even large deviations principles \cite{RobMon51,mcleish, Ljung77, kushner1984asymptotic,dupuis1989stochastic, BMP90, duflo96,Benaim99}. At the core of this line of work is the idea that in the limit where the step-size, or learning rate, tends to zero, the trajectory of SGD with a fixed loss function (appropriately rescaled in time) converges to the solution of gradient flow for the population loss with the same initialization. Recently there has been considerable interest in quantifying the rate of this trajectory-wise convergence to higher order, in terms of a diffusion approximation. Namely, there are many works developing asymptotic expansions of the trajectory in the learning rate~\cite{mandt2017stochastic,li2017diffusion,li2019stochastic,anastasiou2019normal,li2021validity}. Motivated by this, there is a rich line of work bounding the time to equilibrium for the associated diffusion approximation (as well as Langevin--type modifications) under uniform ellipticity assumptions~\cite{mandt2017stochastic,pmlr-v65-raginsky17a,cheng2020stochastic,pmlr-v65-zhang17b}. There is also an interesting line of work obtaining PDE limits in the ``shallow network'' regime where the dimension of the parameter space diverges but the dimension of the data remains constant: see e.g.,~\cite{MMN18,rotskoff2018trainability,chizat2018global,sirignano2020mean,araujo2019mean}.

In recent years, there has been considerable interest in understanding the \emph{high-dimensional setting}, where one is constrained in the amount of data or the run-time of the algorithm due to the high-dimensional nature of the data and the complexity of the model being trained. In these regimes, one cannot simply take the learning rate to be arbitrarily small as this would force an unlimited sample size and run-time. This is a common issue in high-dimensional statistics and the standard analytic approach is to study regimes where the sample size scales with the dimension of the problem \cite{Vershynin,wainwright2019high}. 

For SGD with constant learning rate, there has been recent progress on quantifying the dimension dependence of the sample complexity for various tasks on general (pseudo or quasi-) convex objectives \cite{Bottou99,BottouLeCun04,shamir2016convergence,NeedellSrebroWard14,HLPR-COLT,dieuleveut2020} and special classes of non-convex objectives \cite{escaping-saddles-COLT,TanVershyninKaczmarz,arous2021online}. There has also been important work on scaling limits as the dimension tends to infinity for the specific problems of linear regression \cite{WaMaLu16,paquette2021sgd},  Online PCA \cite{WaMaLu16,online-ICA-NIPS}, and phase retrieval \cite{TanVershyninKaczmarz} from random starts, and teacher-student networks \cite{saad1995dynamics,saad1995line,goldt2019dynamics,veiga2022phase} and two-layer networks for XOR Gaussian mixtures~\cite{refinetti2021classifying} from warm starts.  We also note that the study of high-dimensional regimes of gradient descent and Langevin dynamics have a history from the statistical physics perspective, e.g., in  \cite{crisanti1993sphericalp,CugKur93,sarao2019afraid,mannelli2020marvels,celentano2021high,liang2022high}.

 We develop a unified approach to the scaling limits of SGD in high-dimensions with constant learning rate that allows us to understand a broad range of estimation tasks. 
One of course cannot develop a high-dimensional scaling limit for the full trajectory of SGD as the dimension of the underlying parameter space is growing. On the other hand, in practice, one is rarely interested in the full trajectory; instead one typically tracks the trajectory of various summary statistics of the algorithm's evolution,
such as the loss, the amplitude of various weights, or correlations between the classifier and the ground truth (in a supervised setting). We show in Theorem~\ref{thm:main} that under mild regularity assumptions, the evolution of these summary statistics converges as the dimension grows to the solution of a system of (possibly stochastic) differential equations. 
These \emph{effective dynamics} depend dramatically on the initializations (warm vs.\ random or cold), the parameter regions in which one is developing the scaling limit, and the scaling of the step-size with the dimension. 

In practice, SGD often exhibits two types of phases in training: \emph{ballistic phases} where the summary statistics macroscopically change in value, and \emph{diffusive phases}, where they fluctuate microscopically. (During training, the evolution can start with either, and can even alternate multiple times between these phases.) Our approach allows us to develop scaling limits for both types of phases. 

In ballistic phases, the effective dynamics are given by an ordinary differential equation (ODE) and the finite-dimensional intuition that the summary statistics evolve under the gradient flow for the population loss is correct provided the (constant) learning rate is sufficiently small in the dimension. When the learning rate follows a certain \emph{critical} scaling---matching scalings commonly used in the high-dimensional statistics literature---an additional correction term appears. At this critical scaling, the phase portrait deviates significantly from that of the population gradient flow. Furthermore, in microscopic neighborhoods of the fixed points of this ODE, the effective dynamics become  diffusive and are given by SDEs which can exhibit a wide range of (possibly degenerate) behaviors. We note that the appearance of the correction term in the ballistic phase was first observed in the setting of teacher-student networks in~\cite{saad1995dynamics,saad1995line} and very recently investigated in detail in~\cite{veiga2022phase}.

As a simple, first example of the departure of the effective dynamics in the critical step-size regime from the classical perspective, we study estimation for spiked matrix and tensor models in Section~\ref{sec:matrix-tensor-pca}. In these models, the effective dynamics are exactly solvable and when the step-size scales critically with the dimension, in the ballistic phase the dynamics have additional fixed points as compared to the population gradient flow. The stability of these fixed points exhibit sharp transitions at special signal-to-noise ratios. When initialized randomly, the SGD starts in a microscopic neighborhood of an uninformative such fixed point, within which its effective dynamics become diffusive and exhibit a sharp transition between mean-reverting and mean-repellent Ornstein--Uhlenbeck (OU) processes. 

To demonstrate our approach on more complex classification tasks typically studied using neural networks, 
we study a Gaussian mixture model analogue of the classical XOR problem in Section~\ref{sec:XOR-results}. (The XOR problem is 
arguably the canonical example of a decision boundary requiring at least two-layers to represent \cite{minsky2017perceptrons}.)
Here we find that the natural summary statistics are 22 dimensional, and their (ballistic) effective dynamics exhibit a rich phenomenology between {some 39} connected fixed point regions of varying topological dimension. 
Surprisingly, we find that if we initialize the weights of the network randomly (following a Gaussian distribution), then the algorithm will converge to a classifier with macroscopic generalization error with probability $\sfrac{29}{32}$ and then follow a degenerate diffusion. On the other hand, we demonstrate the benefit of overparametrization, showing that as the width of the second layer grows, the probability of ballistically converging to a Bayes optimal classifier goes to~$1$; this is a mathematically rigorous example of the \emph{lottery ticket hypothesis} of~\cite{LotteryTicket}.

Before delving into the XOR problem, we first analyze the classification of a two component Gaussian mixture model in Section~\ref{sec:binary-gmm}. This task is of course best solved using a one-layer network i.e., logistic regression, but with a two-layer network it exhibits some similar phenomenologies to the XOR problem while being more amenable to finer analysis. 
Here, we again find that if with random initial weights, with probability $1/2$ the SGD will first converge to a classifier with macroscopic generalization error, and then follow a degenerate diffusion in a microscopic neighborhood of that set of unstable fixed points. 
We demonstrate this both empirically for positive signal-to-noise ratio and theoretically in the limit where the SNR tends to zero after the dimension tends to infinity. 

While the above are a few examples that we are able to solve in detail for both their ballistic and diffusive limits, we expect our main theorem to be applicable and lend new insights into a host of other problems including SGD for finite-rank matrix and tensor PCA, and one and two-layer neural networks applied to mixtures of $k$-Gaussians for fixed $k\ge 2$. 
We leave this to future investigation. 
In this paper, we only consider the simplest variant of SGD, namely online SGD; we leave other variants involving batching and re-use to future works.

\section{Main result}
Suppose that we are given a sequence of i.i.d. data $Y_1,Y_2,\ldots $
taking values in $\cY_{n}\subseteq\R^{d_n}$ with law $P_{n}\in\cM_{1}(\R^{d_n})$,
and a loss function $L_{n}:\cX_{n}\times\cY_{n}\to\R$, where here
$\cX_{n}\subseteq\R^{p_{n}}$ is the parameter space. 
Consider online stochastic gradient descent with constant learning rate,
$\delta_{n}$, which is given by 
\[
X_{\ell}=X_{\ell-1}-\delta_{n}\nabla L_n(X_{\ell-1},Y_{\ell})\,,
\]
with possibly random initialization $X_{0}\sim\mu_{n}\in\cM_{1}(\cX_{n})$.
Our interest is in understanding this evolution, $(X_\ell)$, in 
the regime where both $p_n$ and $d_n\to\infty$ as $n\to\infty$.
To this end, suppose that there is a finite collection of summary statistics of $(X_\ell)$
whose evolution we are interested in. 
More precisely, suppose that we are given a sequence of functions
$\mathbf{u}_{n}\in C^{1}(\R^{p_{n}};\R^{k})$  for some fixed $k,$
where $\mathbf{u}_{n}(x)=(u_{1}^{n}(x),...,u_{k}^{n}(x))$,
and our goal is to understand the evolution of $\mathbf{u}_n(X_\ell)$.

To develop a scaling limit, we need some regularity assumptions on the
relationship between how the step-size scales in relation to the loss, its gradients, and the data distribution.
To this end let $$H(x,Y)=L_n(x,Y)-\Phi(x) \qquad \text{where} \qquad \Phi(x)  = \mathbb E [ L_n(x,Y)]\,.$$
In the following, we suppress the dependence of $H$ on
$Y$ and instead view $H$ as a random function of $x$, denoted $H(x)$. We let $V(x)=\E\left[\nabla H(x)\tensor\nabla H(x)\right]$
be the covariance matrix for $\nabla H$ at~$x$.

We make two assumptions on the triple $(\mathbf{u}_n,L_n,P_n)$ and the step size $\delta_n$. The first is an upper bound on the learning rate in terms of the regularity of the summary statistics. The second is our key assumption and asks that the summary statistic evolutions asymptotically close. These assumptions need not hold uniformly over the entire parameter space $\mathbb R^{p_n}$, only uniformly over pre-images of compact sets under $\mathbf{u}_n$. 
We start with the regularity assumption, ensuring tightness of trajectories of the summary statistics. 

\begin{defn}\label{defn:localizable}
A triple $(\mathbf{u}_{n},L_{n},P_{n})$ is \textbf{$\delta_{n}$-localizable}
with localizing sequence $(E_K)_K$ if there is an exhaustion by compacts $(E_{K})_K$ of $\R^{k}$, and constants $C_K$ (independent of $n$)
such that 
\begin{enumerate}
\item $\max_{i} \sup_{x\in\mathbf{u}_{n}^{-1}(E_{K})}\norm{\nabla^{2}u_{i}^n}_{\op}\leq C_K\cdot\delta_n^{-1/2}$,  and $\max_{i} \sup_{x\in\mathbf{u}_{n}^{-1}(E_{K})}\norm{\nabla^{3}u_{i}^n}_{\op}\leq C_K$;
\smallskip
\item $\sup_{x\in\mathbf{u}_{n}^{-1}(E_{K})}\|\nabla\Phi\|\le C_K$, and 
$\sup_{x\in\mathbf{u}_{n}^{-1}(E_{K})}\mathbb{E}[\|\nabla H\|^{8}]\le C_K\delta_{n}^{-{ 4}}$;
\smallskip
\item $\max_{i}\sup_{x\in\mathbf{u}_{n}^{-1}(E_{K})}\E[\langle \nabla H,\nabla u_{i}^n\rangle ^{4}]\leq C_K\delta_{n}^{-2}$, and \hfill 

$\max_{i} \sup_{x\in \mathbf u_n^{-1}(E_{K})} \mathbb E[\langle \nabla^2 u_i^n, \nabla H\otimes \nabla H - V\rangle^2] = o(\delta_n^{-{ 3}})$.
\end{enumerate}
\end{defn}

To help the reader parse this assumption, we provide an in-depth discussion of each of these items, along with examples to have in mind in Remark~\ref{rem:localizability-discussion}.  For now, we make the crucial observation that (1)--(3) are all closed under decreasing the step-size $\delta_n$ so for any reasonable task and family of summary statistics there will be a scaling of the step-size with $n$ below which they will satisfy the conditions of Definition~\ref{defn:localizable}. For concreteness, summary statistics that are good to have in mind are correlations of the parameters with certain ground truth vectors, $\ell^2$ norms of the parameters, and the population loss itself.

We now turn to our second assumption, that the limiting evolution equations for the family of summary statistics chosen close. Define the following first and second-order differential operators, 
\begin{align}\label{eq:A-L-operators}
\cA_n = \sum_{i} \partial_i \Phi \partial_i\,, \qquad \mbox{and} \qquad 
\cL_{n}=\frac{1}{2}\sum_{i,j} V_{ij}\partial_{i}\partial_{j}\,.
\end{align}
Alternatively written,  $\cA_n = \langle \nabla \Phi, \nabla\rangle$ and $\mathcal L_n = \frac{1}{2} \langle V, \nabla^2\rangle$.

\begin{defn}\label{defn:asympotically-closable}
A family of summary statistics $(\mathbf{u}_n)$ are \textbf{asymptotically closable} for learning rate $\delta_n$ if $(\mathbf{u}_n, L_n, P_n)$ are $\delta_n$-localizable with localizing sequence $(E_K)_K$, and furthermore there exist locally Lipschitz functions $\mathbf{h}:\mathbb R^k\to\mathbb R^k$ 
and $\mathbf{\Sigma}:\R^{k}\to \mathbb R^{k\times k}$, such that
\begin{align}
\sup_{x\in \mathbf{u}_n^{-1}(E_K)} \big\| \big( - \cA_n + \delta_n \mathcal L_n\big) \mathbf{u}_n(x) - \mathbf h(\mathbf{u}_n(x))\big\| & \to 0\,,  \label{eq:eff-drift} \\
\sup_{x\in\mathbf{u}_{n}^{-1}(E_{K})}\|\delta_{n}J_{n} V J_{n}^{T}-\mathbf{\Sigma}(\mathbf{u}_{n}(x))\| & \to0\,. \label{eq:diffusion-matrix}
\end{align}
In this case we call $\mathbf{h}$ the \emph{effective drift}, and $\mathbf{\Sigma}$ the \emph{effective volatility}. 
\end{defn}

We are now ready to present our main result. For a function $f$ and measure $\mu$ we let $f_{*}\mu$
denote the push-forward of $\mu$.

\begin{thm}\label{thm:main}
Let $(X_{\ell}^{\delta_{n}})_{\ell}$ be stochastic gradient descent initialized from $X_{0}\sim\mu_{n}$
for $\mu_{n}\in \cM_{1}(\mathbb{R}^{p_{n}})$ with learning
rate $\delta_{n}$ for the loss $L_{n}(\cdot,\cdot)$ and data distribution
$P_{n}$. For a family of summary statistics $\mathbf{u}_n = (u_i^n)_{i=1}^k$, let $(\mathbf{u}_n(t))_t$ be the linear interpolation of $(\mathbf{u}_n(X_{\lfloor t\delta_{n}^{-1}\rfloor}^{\delta_{n}}))_{t}$. 

Suppose that $\mathbf{u}_n$ are asymptotically closable with learning rate $\delta_n$, effective drift $\mathbf{h}$, and effective volatility $\mathbf{\Sigma}$, and that the pushforward of the initial data has $(\mathbf{u}_n)_*\mu_n \to \nu$ weakly for some $\nu\in \cM_1(\mathbb R^k)$. Then $(\mathbf{u}_{n}(t))_{t}\to(\mathbf{u}_t)_t$ weakly as $n\to\infty$, where
$\mathbf{u}_t$ solves
\begin{equation}\label{eq:effective-dynamics}
d\mathbf{u}_{t}= \mathbf{h}(\mathbf{u}_{t}) dt+\sqrt{\mathbf{\Sigma}(\mathbf{u}_{t})}d\mathbf{B}_{t}\,.
\end{equation}
initialized from $\nu$, where $\mathbf{B}_{t}$ is a standard Brownian
motion in $\mathbb{R}^{k}$. 
\end{thm}

The proof of Theorem~\ref{thm:main} is provided in Section~\ref{sec:proof-of-convergence-thm} and can be seen as a version of the classical martingale problem (see~\cite{StroockVaradhan06}) for high-dimensional stochastic gradient descent. 
We call the solution to \eqref{eq:effective-dynamics} the \emph{effective dynamics} of the summary statistics $\mathbf{u}_n$. The fact that $\mathbf{h}, \mathbf{\Sigma}$ are locally Lipschitz ensures that this solution is unique.

We end this subsection with discussion of the various scalings appearing in Definition~\ref{defn:localizable}.  

\begin{rem}\label{rem:localizability-discussion}
The kinds of summary statistics that we most frequently have in mind for application are (1) linear functions of the parameter space $\mathcal X_n$, for instance the correlation with a unit vector, or some ground truth; (2) radial statistics, like the $\ell^2$-norm of the parameters, or some subset of the parameters; and (3) rescaled versions (usually blown up by $\delta_n^{-1/2}$ of these near their fixed points, as described in Section~\ref{subsec:rescaled}. Regarding the item (1) in Definition~\ref{defn:localizable}, for linear functions, it trivially holds; for radial statistics, the Hessian is a block identity matrix, so item (1) holds as long as $\delta_n$ is $O(1)$; therefore item (1) is most restrictive for rescalings of non-linear statistics, e.g., $u(x) = \delta_n^{-\alpha} (\|x\|^2 - 1)$ where it prevents consideration of this statistic with $\alpha>1/2$.

Turning to item (2) of Definition~\ref{defn:localizable}, we comment that the regularity assumptions made on $\Phi, L$ here are 
less restrictive than uniform Lipchitz assumptions common 
to the literature. In particular, we do not assume
the population loss is Lipschitz everywhere, as we may have that $\bigcup_{K}\mathbf{u}_{n}^{-1}(E_K)$ does not cover $\mathcal X_n$, 
nor does it imply uniform smoothness of $H$ (and in turn $L$) as we may (and will) be taking $\delta_{n} \to 0$ with $n$. 

Let us lastly motivate the scalings appearing in item (3), which ensure there is some independence between $H$ and the values of $\nabla u$ and $\nabla^2 u$ at $x$. As a testbed, suppose that $\nabla H(x)$ is a random vector with i.i.d.\ entries all of order $1$. If $u$ is a rescaled linear statistic, e.g., $\delta_n^{-1/2} \langle x,e_1\rangle$ then the first bound of item (3) is saturated, and the second of course is trivial due to the linearity of $u$. The second bound is saturated by taking a rescaling of a radial statistic, e.g., $\delta_n^{-1/2} \|x\|^2$, again assuming for maximal simplicity that $\nabla H$ is an i.i.d. random vector with order one entries. In fact, the second part of item (3) could be dropped at the expense of more complicated diffusion coefficients in limiting SDE's: see Remark~\ref{rem:weakening-localizability}. 
\end{rem}

\begin{rem}\label{rem:weakening-localizability}
While we discussed above the reasons for which the various scalings of Definition~\ref{defn:localizable} were selected, it is interesting to ask what changes in Theorem~\ref{thm:main} should certain of the assumptions of Definition~\ref{defn:localizable} be violated. Most of the assumed bounds in the definition of localizability are used to establish tightness and ensure higher order terms in Taylor expansions vanish in the $n\to\infty$ limit. In principle the second assumption in item (3) of Definition~\ref{defn:localizable} could be dropped; in that case, the same quantity is still ensured to be $O(\delta^{-3})$ by the other localizability assumptions. Then
Theorem~\ref{thm:main} would still apply, but the limiting diffusion matrix would be the $n\to\infty$ limit (assuming it exists) of 
\begin{align*}
	\delta J V J^T & +  \delta^2  \mathbb E[ \langle \nabla H, J \rangle \otimes \langle \nabla^2 \mathbf{u}, \nabla H\otimes \nabla H - V\rangle]+  \delta^2 \mathbb E[\langle \nabla^2 \mathbf{u}, \nabla H\otimes \nabla H - V\rangle\otimes \langle \nabla H, J \rangle] \\ 
	& +  \delta^3 \mathbb E[\langle \nabla^2 \mathbf{u},\nabla H \otimes \nabla H - V\rangle^{\otimes 2}]\,,
\end{align*}
as opposed to simply the limit of $\delta J V J^T$. 
\end{rem}

Generically, the choice of summary statistics to which to apply Theorem~\ref{thm:main} depends both on the quantities one is interested in, and the specifics of the task. In our examples, the choices are natural: correlations with ground truth vectors, finite numbers of final layer weights, and $\ell^2$ norms of the parameters. In less structured settings,  the choice of summary statistics may be more open-ended. 
One could start with a summary statistic of interest, like the projection in a principal subspace of an empirical matrix (a covariance or, as suggested experimentally in e.g.,~\cite{SagunEtAl,PapyanPMLR}, a Hessian), or the population loss itself. Then from that statistic, one would determine the other statistics needed to build an asymptotically closed family per Definition~\ref{defn:asympotically-closable}.

\subsection{Comparison to fixed dimensional perspective: critical v.s. subcritical step-sizes}\label{subsec:comp-fixed-dim}

Let us compare this with the classical limit theory 
of SGD in fixed dimension. For the sake of this discussion, suppose that not only does~\eqref{eq:eff-drift} hold, but each of the two terms $\cA_n \mathbf u$ and $\delta_n \mathcal L_n \mathbf{u}$ (recall~\eqref{eq:A-L-operators}) individually admit $n\to\infty$ limits: namely that there exists $\mathbf{f},\mathbf{g}:\mathbb R^k \to\mathbb R^k$ such that 
\begin{align}
\sup_{x\in\mathbf{u}_{n}^{-1}(E_{K})}\| \cA_n \mathbf{u}_{n}(x) -\mathbf{f}(\mathbf{u}_{n}(x))\| & \to0\,, \label{eq:population-drift}\\
\sup_{x\in\mathbf{u}_{n}^{-1}(E_{K})}\|\delta_{n}\mathcal{L}_{n} \mathbf{u}_{n}(x)-\mathbf{g}(\mathbf{u}_{n}(x))\| & \to0\,, \label{eq:corrector}
\end{align}
in which case, evidently~\eqref{eq:eff-drift} holds with $\mathbf{h} = -\mathbf{f} + \mathbf{g}$. When~\eqref{eq:population-drift} and~\eqref{eq:corrector} both hold, we call $\mathbf{f}, \mathbf{g}$ and $\Sigma$ the \textbf{population drift}, the \textbf{population corrector}, and the \textbf{diffusion matrix} of $\mathbf u$ respectively. 
From the fixed dimensional perspective, when~\eqref{eq:population-drift} holds, one predicts $\mathbf{u}$ to solve 
\begin{align}\label{eq:population-dynamics}
	d\mathbf{u}_t = -\mathbf{f}(\mathbf{u}_t)dt\,,
\end{align}
with initial data $\mathbf{u}_0 \sim \mathbf{u}_*\mu$. 
as this is its evolution under gradient descent on the population loss $\Phi$. 
Evidently this perspective
only applies in the high-dimensional limit of Theorem~\ref{thm:main} if both the population corrector $\mathbf{g}$ and the diffusion matrix $\Sigma$ are zero.
We find that for any triple $(\bu_n, L_n,P_n)$, there is a scaling of the learning rate $\delta_n$ with $n$ below which $\mathbf{g} = \Sigma=0$, and the effective dynamics agree with the population dynamics~\eqref{eq:population-dynamics} (we call this the \textbf{sub-critical} scaling regime, where the classical perspective applies), and a \textbf{critical} scaling regime in which $g$ and $\Sigma$ may be non-zero, and the high-dimensionality induces non-trivial corrections to~$\mathbf{f}$. (In the case of teacher--student networks, the terms $\mathbf{f}$ and $\mathbf{g}$ can be compared to the ``learning" and ``variance" terms in Eq.~(14a) of~\cite{veiga2022phase}.)

To see this, notice that if the triple $(\bu_n,L_n,P_n)$ is  $\delta_n$-localizable 
for some $\delta_n\to0$, then 
it is also $\delta_n'$-localizable for every sequence $\delta'_n=O(\delta_n)$. 
If furthermore \eqref{eq:diffusion-matrix} and~\eqref{eq:population-drift}--\eqref{eq:corrector} hold for $\delta_n$ with some $\bff,\bg$ 
and $\Sigma$,
then these limits also exists for $\delta_n'=o(\delta_n)$ with the same $\bff$ but with $\bg= \Sigma=0$. As such, there can be exactly one scaling of $\delta_n$ with $n$ at which $\bg$ or $\Sigma$ may be non-zero, and for all smaller scales of $\delta_n$, the fixed-dimensional perspective of~\eqref{eq:population-dynamics} applies.\footnote{Note that if $\delta_n=o(\delta_n')$, then limiting $\bg,\Sigma$ may not exist for $\delta'_n$, so there is no super-critical regime.}

\subsection{Ballistic vs.\ diffusive behavior of effective dynamics}\label{subsec:rescaled}

In all of our examples, the diffusion matrix for the effective dynamics of the most natural choice of summary statistics is zero even in
the critical scaling regime where $\mathbf{h}\ne \mathbf{f}$.
We call this the \textbf{ballistic limit}.
In this case, the effective dynamics of the summary statistics is given by the ODE system
\begin{equation}\label{eq:ballistic}
d\bfu_t =\mathbf{h}(\mathbf{u}_t)dt\,.
\end{equation}
In these settings, the phase portrait of the summary statistics
is asymptotically that of this flow. 

Note that by construction of the scaling limit, the phase portrait
of the ballistic limit only describes the evolution of summary statistics on length-scales
that are order 1 and number of iterations that are order $1/\delta_n$.
If one is then interested
in the evolution of $\bfu_n$ in microscopic $o(1)$ neighborhoods of the fixed points of the ballistic effective dynamics of~\eqref{eq:ballistic},
Theorem~\ref{thm:main} also allows one to develop separate \textbf{diffusive limits} there. 

To study diffusive regimes, one must apply Theorem~\ref{thm:main} to re-centered and re-scaled summary statistics,
$\tilde{\bu}_n(t) = \delta^{-\alpha}_n(\bfu_n(t)-\bfu_\star)$ where $\bfu_\star$ is a fixed point of~\eqref{eq:ballistic}.\footnote{One might also wish to rescale time like $\delta_n^{-\beta}$, where $\beta$ may depend on $t$; we leave this to future work.}

To apply Theorem~\ref{thm:main},  $\alpha$ must be chosen appropriately so that the triple
$(\tilde{\bu}_n(t),L_n,P_n)$  is $\delta_n$-localizable and to pick out the next order drifts for $\tilde u$---the first order term being zero microscopically close to $\mathbf{u}_\star$---and such that the initial data still converges $(\tilde{\bfu}_n)_*\mu_n\to \tilde{\nu}$.

 This then leads to the \textbf{rescaled effective dynamics} of the summary statistics $\bu_n$ near $\bfu_\star$:
\begin{align}\label{eq:rescaled-effective-dynamics}
d\tilde{\bfu}_t = \tilde{\mathbf{h}}(\tilde{\mathbf{u}}_t) dt + \tilde{\Sigma}^{1/2}(\tbfu_t)d\mathbf{B}_t \qquad \mbox{with $\tilde{\mathbf u}_0 \sim \tilde \nu$}\,.
\end{align}
The rescaled effective dynamics are
similar in spirit to diffusion one typically finds for 
the evolution of SGD near critical points in fixed dimensions. However, we note two important differences as compared to this perspective. 
Firstly, since this is a high-dimensional limit of general summary statistics,~\eqref{eq:rescaled-effective-dynamics} applies in a neighborhood
of a fixed point of the effective ODE system~\eqref{eq:ballistic}, rather than the
population dynamics~\eqref{eq:population-dynamics}. Secondly, in many examples (indeed all the
ones we study) the SDE's we get are degenerate,
so that uniform ellipticity assumptions typically used to understand
hitting and mixing times in these regimes do not apply. {The degeneracies can take various forms, with $\tilde \Sigma$ sometimes being rank deficient in the entire $\tbfu$-space, and sometimes vanishing completely as $\tilde \Sigma$ approaches certain distinguished points, for instance $\mathbf{u}_\star$. The implications of such degeneracies can be severe, as degenerate diffusions can be absorbed for arbitrarily long times by their \emph{unstable} fixed points (c.f.\ the simple case of a 1D geometric Brownian motion).}

\begin{rem}[Training at the edge of stability and critical scaling]\label{rem:edge-of-stability}
In~\cite{EdgeOfStability-GD}, it was empirically observed that the best training for neural networks does not occur when step-sizes are small enough for the classical gradient flow approximation to be valid. Instead, it occurs \emph{at the edge of stability} where the step size is just small enough for the training to remain stable. Here, the loss fluctuates for some time before eventually converging to lower values than it would with smaller step size. This critical step size scaling is defined via the \emph{sharpness}, namely the largest eigenvalue of the training loss Hessian. For a selection of recent theoretical investigations of this phenomenon see, e.g.,~\cite{EoS-Ahn-Bubeck,EoS-Damian,EoS-Arora,EoS-Ge}. 

While sharpness and edge of stability do not have direct analogues in the context of online SGD, a qualitatively similar phenomenon can be seen by taking the population loss as a summary statistic. The critical scaling of the learning rate with dimension discussed in Sections~\ref{subsec:comp-fixed-dim}--\ref{subsec:rescaled} constrains the step size in terms of the top eigenvalue of the Hessian of the loss. With this scaling, the population loss fluctuates near critical regions of its ballistic flow, allowing it to escape the critical region, whereas with a sub-critical learning rate the population loss stays stuck. We leave more detailed investigation of this connection to edge-of-stability phenomena for SGD to future investigation. 
\end{rem}

\part{Examples}
In the following sections, we demonstrate Theorem~\ref{thm:main} on a range of popular examples
of high-dimensional statistical tasks. We begin first in Section~\ref{sec:matrix-tensor-pca}  by presenting an application to a widely studied problem of high-dimensional estimation: namely, de-noising a rank one tensor
that has been corrupted additively by Gaussian noise.
We then turn to classification. Our aim in these examples is to demonstrate
the applicability of our result to the analysis of multi-layer neural networks. 
To this end we analyze the training dynamics of a two-layer neural network for  two canonical classification tasks, namely classification of a symmetric, binary gaussian mixture model (Section~\ref{sec:binary-gmm}) and classification of a Gaussian analogue of the XOR problem of Minsky--Papert (Section~\ref{sec:XOR-results}).

\section{Matrix and Tensor PCA}\label{sec:matrix-tensor-pca}
\subsection{Model and background}
Consider the problem of de-noising a rank one tensor that has been 
corrupted additively by Gaussian noise via SGD. 
A popular statistical model of this task is the 
 \emph{spiked tensor model} \cite{MR14}.
Suppose that we are given i.i.d. samples of data of the form 
\[
Y^\ell = \lambda v^{\tensor k} + W^\ell
\]
where $W^\ell $ are i.i.d. copies of a $k$-tensor whose entries are  i.i.d. standard Gaussians, $v\in \R^n$ is a unit vector,
and $\lambda=\lambda_n>0$ is the signal-to-noise ratio.  Our goal is to infer~$v$.

In the case $k=2$, this is a version of the well-known spiked matrix model of PCA \cite{johnstone2000distribution}
for which there is, by now, a substantial literature regarding the statistical thresholds.
For a necessarily small selection see, e.g., \cite{baik2005phase,paul2007asymptotics,mont15,perry2018optimality,el2020fundamental}. For related work on online learning
in this context see, e.g., \cite{shamir2016convergence}. Of particular interest in this direction is the
well-known phase transition at $\lambda=1$ for estimation in this problem,
which was determined first for Wishart ensembles in \cite{baik2005phase} and subsequently for this setting in \cite{feral2007largest,capitaine2009largest,benaych2011eigenvalues}.
As we will see in Section~\ref{sec:dynamical-bbp} below, we find a dynamical analogue of this transition at $\lambda=1$.

The case $k\geq 3$ was introduced by Montanari and Richard \cite{MR14}  as a natural generalization of the 
spiked matrix models for estimation (and testing) problems where the data
has multiple indices or requires higher moments. Here there has been a large
literature on the statistical thresholds for estimation and testing, see, e.g.,  \cite{mont15,perry2016statistical,BMMN17,LMLKZ17,JLM20}.
In this setting, there has also been a tremendous literature
on the computational aspects of this problem as it is viewed as a important example
of a model with a \emph{statistical-computational gap}, namely,  a setting where there 
is a gap between the regimes of statistical and computational tractability. See, e.g., \cite{MR14,hopkins2015tensor,hopkins2016fast,LMLKZ17,kim2017community,BGJ18b,arous2021online}. 

We begin with these examples as their effective dynamics are particularly simple to analyze. 
In particular, they are  are exactly solvable and only require two
summary statistics, a correlation observable and a radial term.
Even with this relative simplicity, we encounter a wide range of ODE and SDE limits.
In particular, as mentioned above, we find dynamical phase transitions 
corresponding to the aforementioned thresholds in these models.
For our analysis we will focus exclusively on the most interesting, critical step-size scaling
which corresponds to the \emph{proportional asymptotics} regime from the random matrix 
theory literature.

\subsection{Analysis}

We take as loss the (negative) log-likelihood\footnote{Note that one might also add additional penalty terms. The case of a ridge penalty is treated in Section~\ref{app:matrix-tensor-pca}.} 
namely, 
\[
L(x,Y)=\norm{Y-x^{\tensor k}}^{2}.
\] 
The pair  
\[m= m(x) :=\langle x,v\rangle \qquad \mbox{and}\qquad r_{\perp}^{2}= r_\perp^2(x) :=\|x-m v\|^{2}=\|x\|^{2}-m^{2}\]
 are such that $\Phi(x)=- 2\lambda m^{k}+ (r_{\perp}^{2}+m^{2})^{k}+c$,
and the law of $L$ only depends on them.

In our normalization with $\lambda>0$ fixed, the regime $\delta_n = o(1/n)$ is sub-critical and the regime $\delta_n = \Theta(1/n)$ is critical. \footnote{Note that with different scalings of $\lambda_n$, the critical learning rate changes.}
We focus our presentation on the most interesting regime, namely the critical scaling regime of $\delta_n = c_\delta/n$ for some constant $c_\delta$.
Recalling the relation between number of samples and step-size, we see that this regime corresponds to the proportional asymptotics
regime most studied in the random matrix theory literature where the above-mentioned transition for the top eigenvalue occurs.
Note, however, that the limits in the subcritical regime are in all cases recovered by taking the $c_\delta \downarrow 0$ limits of the ODE's/SDE's of the critical regime. 

For notational simplicity, let $R^2 := m^2 + r_\perp^2$. We consider the pair $\mathbf{u}_{n}=(u_{1},u_{2})=(m,r_{\perp}^{2})$, 
for which Theorem~\ref{thm:main} yields the following effective
dynamics.
\begin{prop}\label{prop:matrix-tensor-effective-dynamics}
Fix $k\ge2$, $\lambda>0$, $c_\delta >0$ and let $\delta_{n}=\sfrac{c_{\delta}}{n}$. Then $\mathbf u_n(t)$ converges as $n\to\infty$ to the solution of 
the following ODE initialized from $\lim_{n\to\infty}(\mathbf{u}_{n})_{*}\mu_{n}$:
\begin{align}\label{eq:matrix-tensor-ODE}
d m=  2 m (\lambda k m^{k-2}-kR^{2(k-1)})dt \,,\qquad\qquad
d r_\perp^2 =  - 4 kR^{2(k-1)}(r_\perp^2 - c_\delta ) dt\,.
\end{align}
\end{prop}
We are able to identify and classify the set
of fixed points of this effective dynamics.  
We focus on 
the critical step-size regime with $c_{\delta}=1$ where one sees from~\eqref{eq:matrix-tensor-ODE} that $r_\perp^2 \to1$, 
where the problem in the matrix case is most directly related to an eigenvalue problem (see Section~\ref{app:matrix-tensor-pca} for the generic $c_{\delta}$
dependencies). Throughout the following, we use the following notion of stability/unstability of a set of fixed points of an ODE.

\begin{defn}
We call a set of fixed points $U$ for an ODE \emph{stable} if for every $\epsilon>0$, for every $u\in B_\epsilon(U)$, the solution of the ODE with initialization $u$ converges to some point in $U$ as $t\to\infty$. Otherwise, we call $U$ unstable.  
\end{defn}

\begin{prop}\label{prop:matrix-tensor-fixed-points}
Eq.~\eqref{eq:matrix-tensor-ODE} has isolated fixed points classified as follows. Let $\lambda_c(k)$ be as in~\eqref{eq:lambda-c-k} and $m_\dagger(k,\lambda) \le m_\star(k,\lambda)$ be as in~\eqref{eq:m-dagger-m-star} (if $k=2$, $\lambda_c= 1$ and $m_\dagger  = m_\star = \sqrt{\lambda - 1}$): 
\begin{enumerate}
\item An unstable fixed point at $(0,0)$ and a fixed point at $(0,1)$; if $k=2$, $(0,1)$ is stable if $\lambda<\lambda_c(2)$ and unstable if $\lambda>\lambda_c(2)$; if $k>2$ $(0,1)$ is always stable.
\item If $\lambda>\lambda_c(k)$: when $k=2$, two stable fixed points at $(\pm m_\star(2), 1)$. When $k\ge 3$, two unstable fixed points at $(\pm m_\dagger(k),1)$ and two stable fixed points at $(\pm m_\star(k),1)$.
\end{enumerate}
\end{prop}
\begin{rem}
The presence of \emph{two} pairs of fixed points when $k\ge 3$ with non-zero correlation with $v$ may seem surprising---indeed it indicates that even some warm starts will fail to attain good correlation with the signal when $\lambda$ is finite. This is an interesting consequence of the corrector in~\eqref{eq:matrix-tensor-ODE} and if one tracks the $c_\delta$ dependence in the above, the fixed point $m_\dagger$ goes to zero as $c_\delta \to 0$ and this barrier to recovery from warm starts vanishes as one approaches sub-critical step-sizes.
\end{rem}

\subsection{A dynamical analogue of the BBP transition}\label{sec:dynamical-bbp}
Let us now consider a rescaling of $\mathbf u_n$
in a microscopic neighborhood of
the saddle set $m=0$. This captures the initial phase from a random start: 
if $\mu_n \sim \cN(0,I_n/n)$, then $(\mathbf u_n)_*\mu_n\to\delta_{(0,1)}$ weakly.
Now rescale and let $\tilde{\mathbf{u}}_{n}= (\tilde m, r_\perp^2) = (\sqrt{n}m,r_{\perp}^{2})$.
Evidently, $\tilde \nu = \lim_{n}(\tilde{\mathbf{u}}_{n})_{*}\mu_{n} = \mathcal{N}(0,1)\otimes \delta_1$. 
\begin{prop}\label{prop:matrix-tensor-rescaled-effective-dynamics}
Fix $k\ge 2$, $\lambda>0$ and $\delta_{n}=1/n$. Then $\tilde{\mathbf{u}}_n(t)$ converges as $n\to\infty$ to the solution of the following SDE initialized from $\tilde\nu$:
\begin{equation}\label{eq:rescaled-tensor-pca}
d\tilde{m} = 2 \tilde m (2\lambda\mathbf{1}_{k=2}-k r_\perp^{2(k-1)}) dt + 2(k r_\perp^{2(k-1)})^{\sfrac{1}{2}}dB_{t}
\quad\qquad
 d r_\perp^2   = - 4 k r_\perp^{2(k-1)} (r_\perp^2 -1)dt\,.
\end{equation}
\end{prop}

\begin{figure}
    \centering
   \subfigure[]{     \includegraphics[width=.23\textwidth]{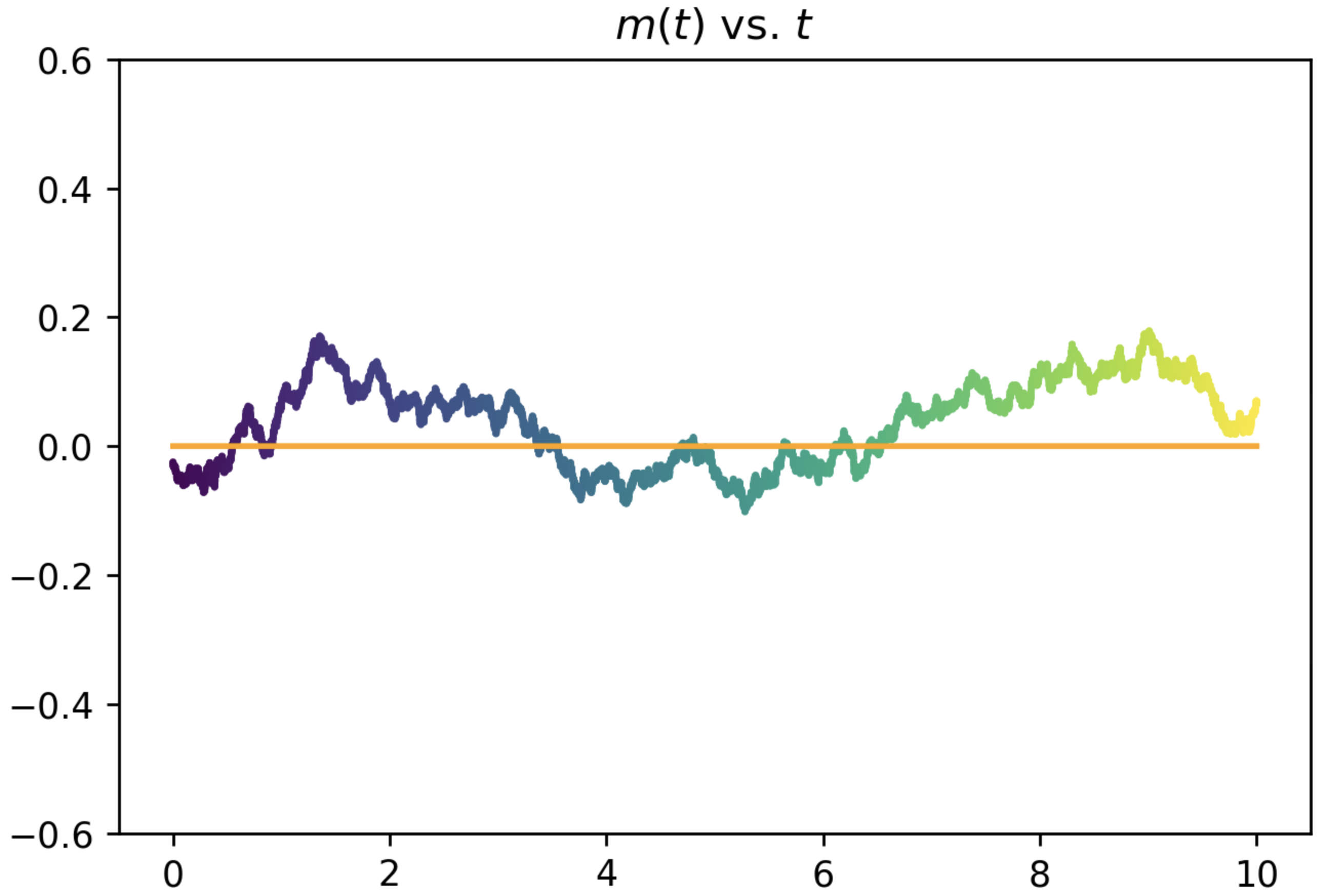}}
\subfigure[]{        \includegraphics[width=.23\textwidth]{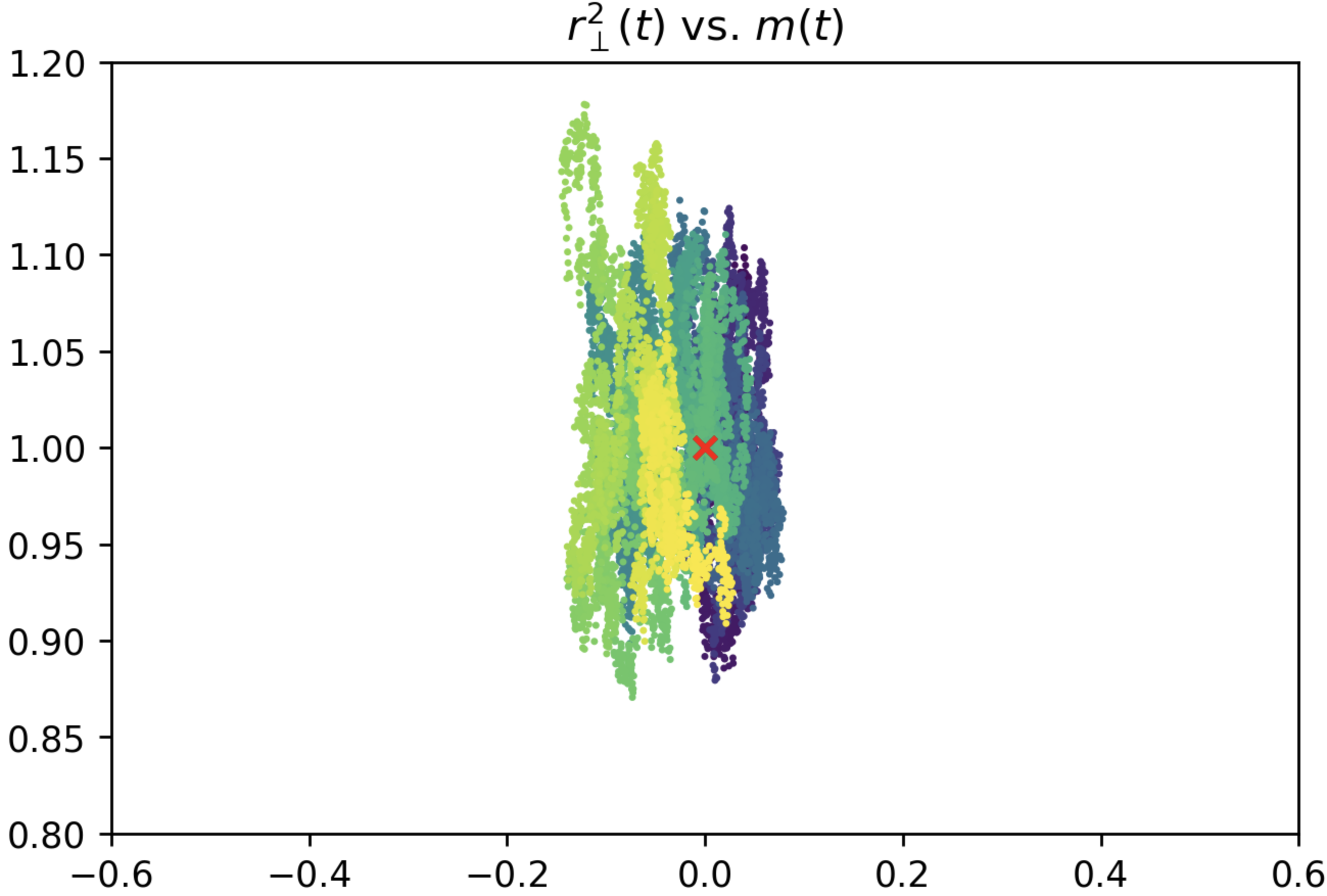}}
\subfigure[]{        \includegraphics[width=.23\textwidth]{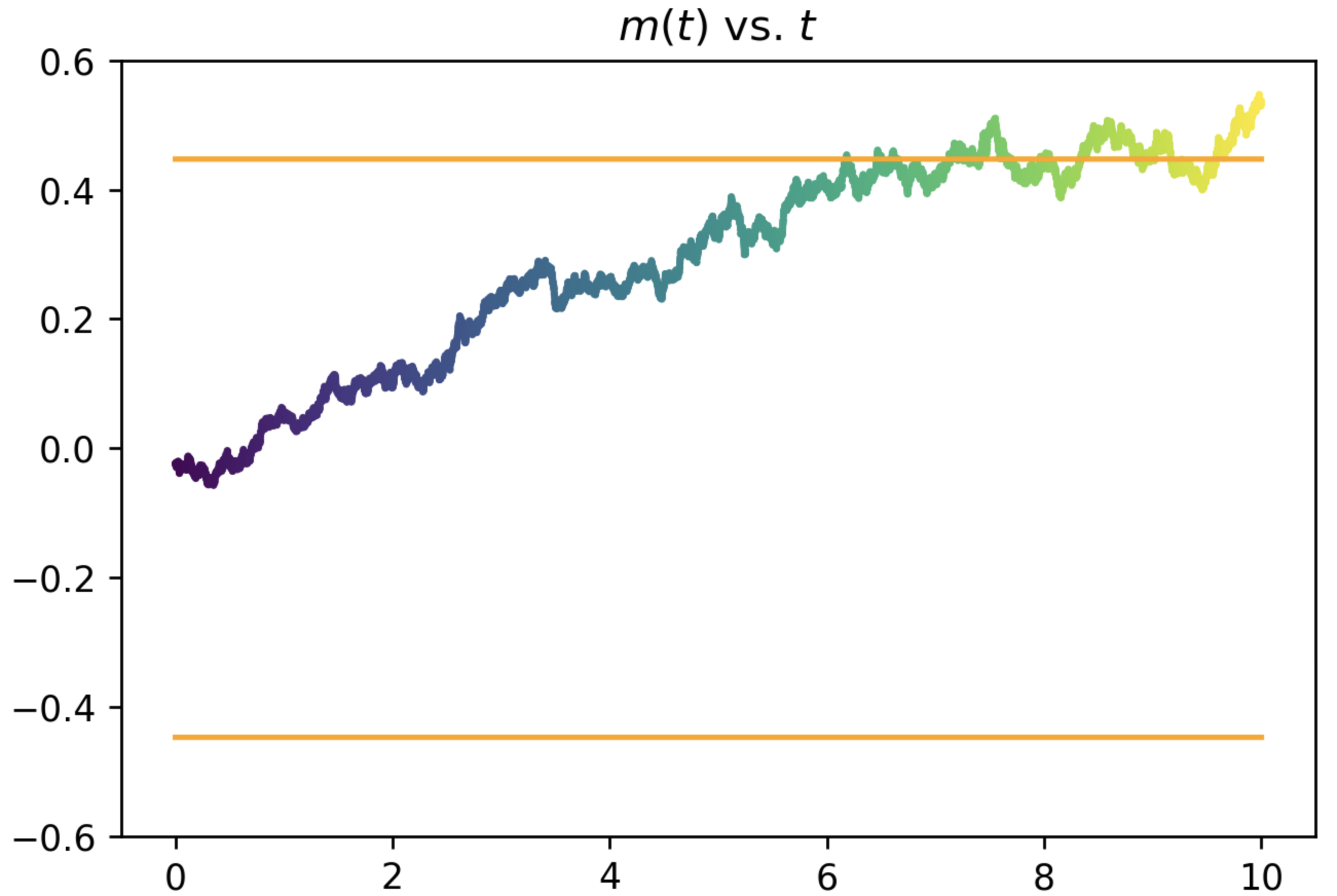}}
       \subfigure[]{ \includegraphics[width=.23\textwidth]{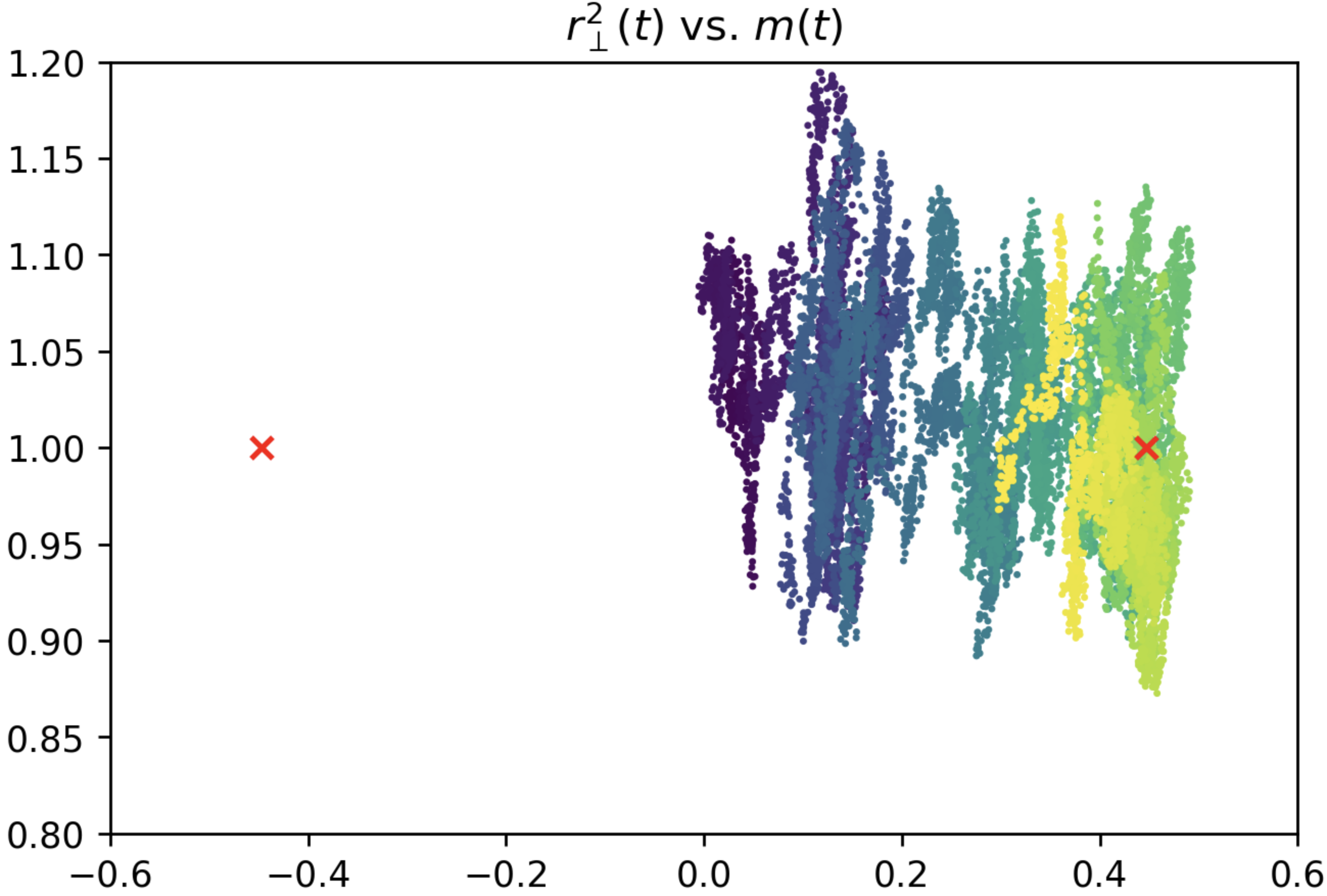}}
       \vspace{-.1cm}
      \caption{Matrix PCA summary statistics in dim.\ $n = 1500$ run for $10n$ steps at $\lambda = 0.8<\lambda_c$ in (a)--(b) and $\lambda = 1.2 >\lambda_c$ in (c)--(d). Here, {\color{red}$\times$} and ${\color{orange}-}$ mark the stable fixed points of the systems. (a) and (c) demonstrate the mean-reverting and mean-repellent OU processes that arise as diffusive limits of the $m$ variable, and (b) and (d) depict the trajectories in $(m,r_\perp^2)$ space.}
    \label{fig:matrixPCA}
\end{figure}

We see that $r_\perp^2$ now solves an autonomous ODE which converges exponentially to $1$. When $k=2$, as $t$ tends to $\infty$, the equation for $\tilde m$ behaves like  
$$d\tilde m = 4(\lambda-1)\tilde m dt+2\sqrt{2}dB_{t}\,.$$ 
This is an OU process which is mean-reverting when $\lambda<1$ and mean-repellent when $\lambda>1$. By stitching together the prelimits of these OU processes at a sequence of scales interpolating between that of $\tilde{\mathbf{u}}_n$ and $\mathbf u_n$, we expect that one could show that for any $\lambda>1$, SGD reaches the stable fixed points at $(\pm m_\star(2),1)$ in $O(n\log n)$ steps (with precise asymptotics, etc.), while when $\lambda<1$, the mean-reverting nature of the OU suggests it needs a much larger number of samples in order to correlate with the vector~$v$. See Figure~\ref{fig:matrixPCA} for an overview, and Figures~\ref{fig:MatrixPCA-r-v-m}--\ref{fig:MatrixPCA-m-t} for more refined numerical verification of this intuition. 

\subsection{On the sample complexity of tensor PCA}
When $k\ge 3$, SGD is known to require a polynomially diverging sample complexity or $\lambda$ in order to solve the tensor PCA problem~\cite{arous2021online}. Accordingly, when $\lambda$ is kept finite in $n$,  the expression for $\tilde m$ in \eqref{eq:rescaled-tensor-pca} is \emph{always} a mean-reverting OU-type process.
Interestingly, one can also capture the (diverging) signal-to-noise threshold for SGD to recover $v$ in tensor PCA by our methods. Indeed, for $k\ge 3$ if one considers $\lambda_n = \Lambda n^{(k-2)/2}$ (matching the predicted gradient-based algorithm threshold from~\cite{BGJ18b}), $\tilde{\mathbf u}_n$ would instead converge to the solution of 
\begin{align}\label{eq:diffusive-m-ballistic-r}
d\tilde m  & = 2\tilde m (k \Lambda - k r_\perp^{2(k-1)})dt+2 (k r_\perp^{2(k-1)})^{\sfrac{1}{2}}dB_{t}\quad
& d r_\perp^2 &  = -4   k r_\perp^{2(k-1)} (r_\perp^2-1)dt\,,
\end{align}
which transitions between mean-reverting and mean-repellent at $\Lambda_c(k)=1$, as in $k=2$. 

We only considered a few specific choices of summary statistics in the above, and the strength of Theorem~\ref{thm:main} derives from its general applicability. As demonstrations, let us mention a few other examples that we would expect to be of interest in the study of SGD for matrix and tensor PCA. The first example is a limiting ballistic limit theorem for the evolution of the population loss $\Phi(x)$. {The population loss can be taken added to the family of summary statistics in our $\delta_n$-localizable triple; in the case of $k$-tensor PCA, this yields,
\begin{align}\label{eq:Phi-fixed-points}
	d \Phi = \Big(- 4k^2  m^2 \big( \lambda^2 m^{2(k-2)} - 2\lambda m^{k-2} R^{2k-2} + R^{4k-4}\big) -  4 k^2 R^{4(k-1)} ( r_\perp^2 - c_\delta)\Big) dt\,.\end{align}

\subsection{A finer diffusive limit theorem at a random start}

The second example is a diffusive limit theorem near the fixed point $(m,r_\perp^2) = (0,1)$ (as opposed to~\eqref{eq:diffusive-m-ballistic-r} where we blew up only the $m$ variable about the saddle set $m=0$ and therefore only $\tilde m$ was moving diffusively). 
In order to do so, we consider the scaling limit of the pair $(\tilde{m},\tilde r_\perp)  = (\sqrt n m,\sqrt n(r_\perp^2 -1))$ and find the following limit:  
\begin{align}\label{eq:matrix-tensor-PCA-double-diffusive-limit}
	d \tilde m & = 2 k (\lambda \mathbf 1_{k=2}\tilde m^{k-1} - 1) dt + 2\sqrt{k} dB_t^{(1)}\,, \quad &  d\tilde r_\perp^2 & = -4k \tilde r_\perp^2  dt+ 2\sqrt{k(k-1)} dB_t^{(2)}\,.
\end{align} 
Interestingly, with this double rescaling, the $n\to\infty$ limit yields a pair of OU processes that are decoupled, namely, each of their drifts are autonomous and their stochastic parts independent. This pair of independent OU processes is depicted in Figures~\ref{fig:MatrixPCA-m-t-zoom}--\ref{fig:MatrixPCA-r-t}. 
}

\begin{figure}
    \centering
  \subfigure[$\lambda=0.8$]{   \includegraphics[width=.23\linewidth]{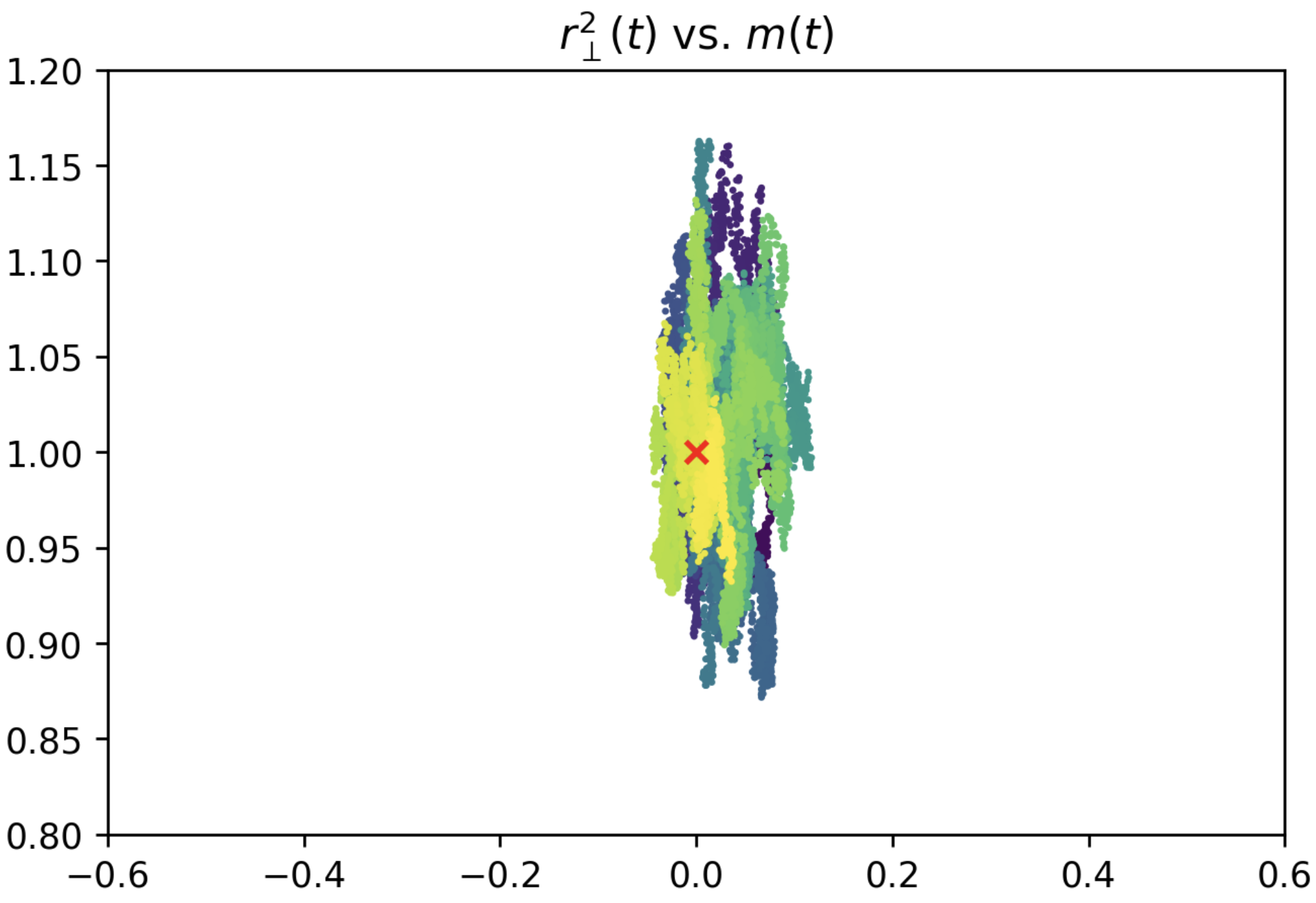}}
  \subfigure[$\lambda =0.9$]{       \includegraphics[width=.23\linewidth]{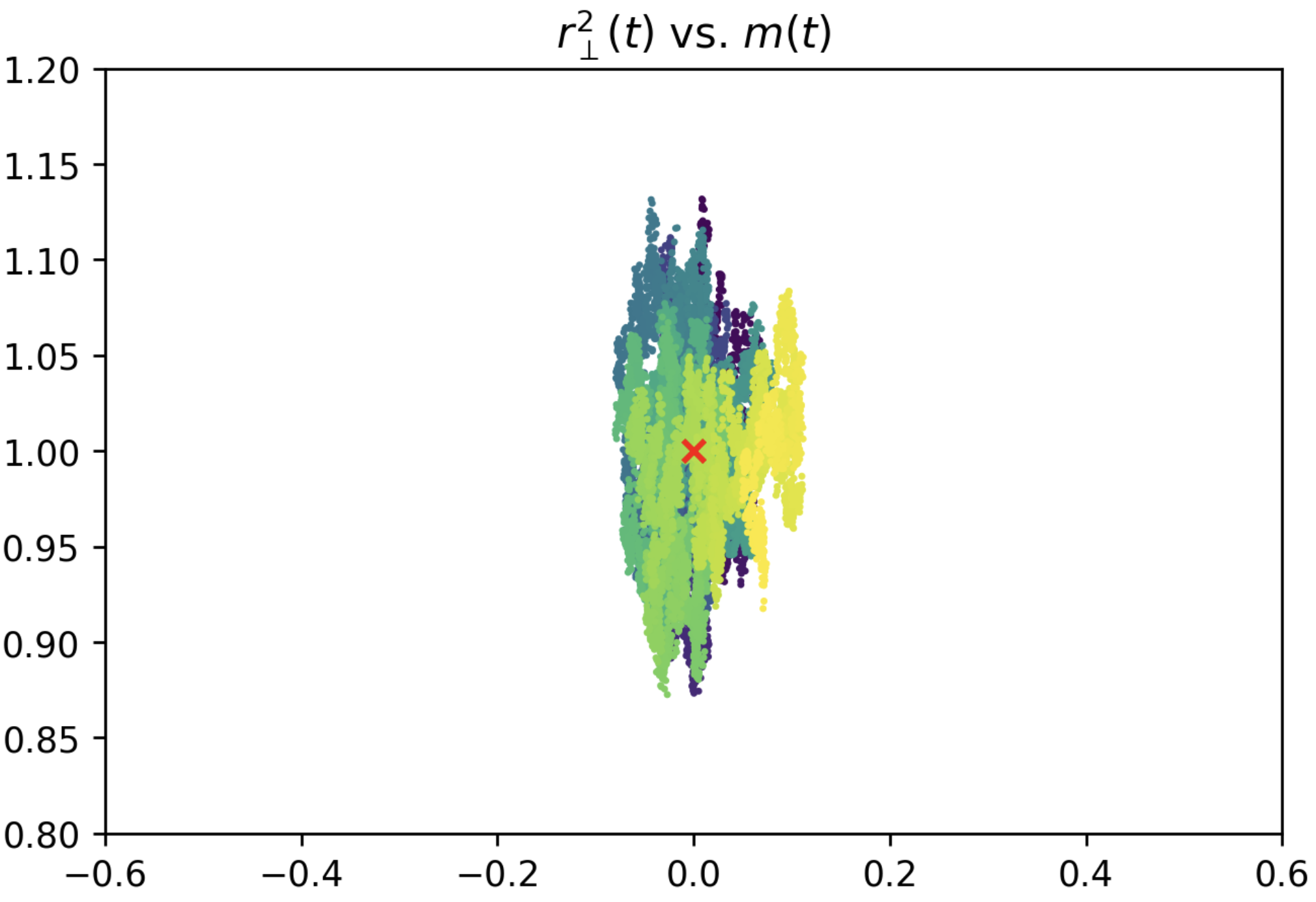}}   
    \subfigure[$\lambda = 1.1$]{       \includegraphics[width=.23\linewidth]{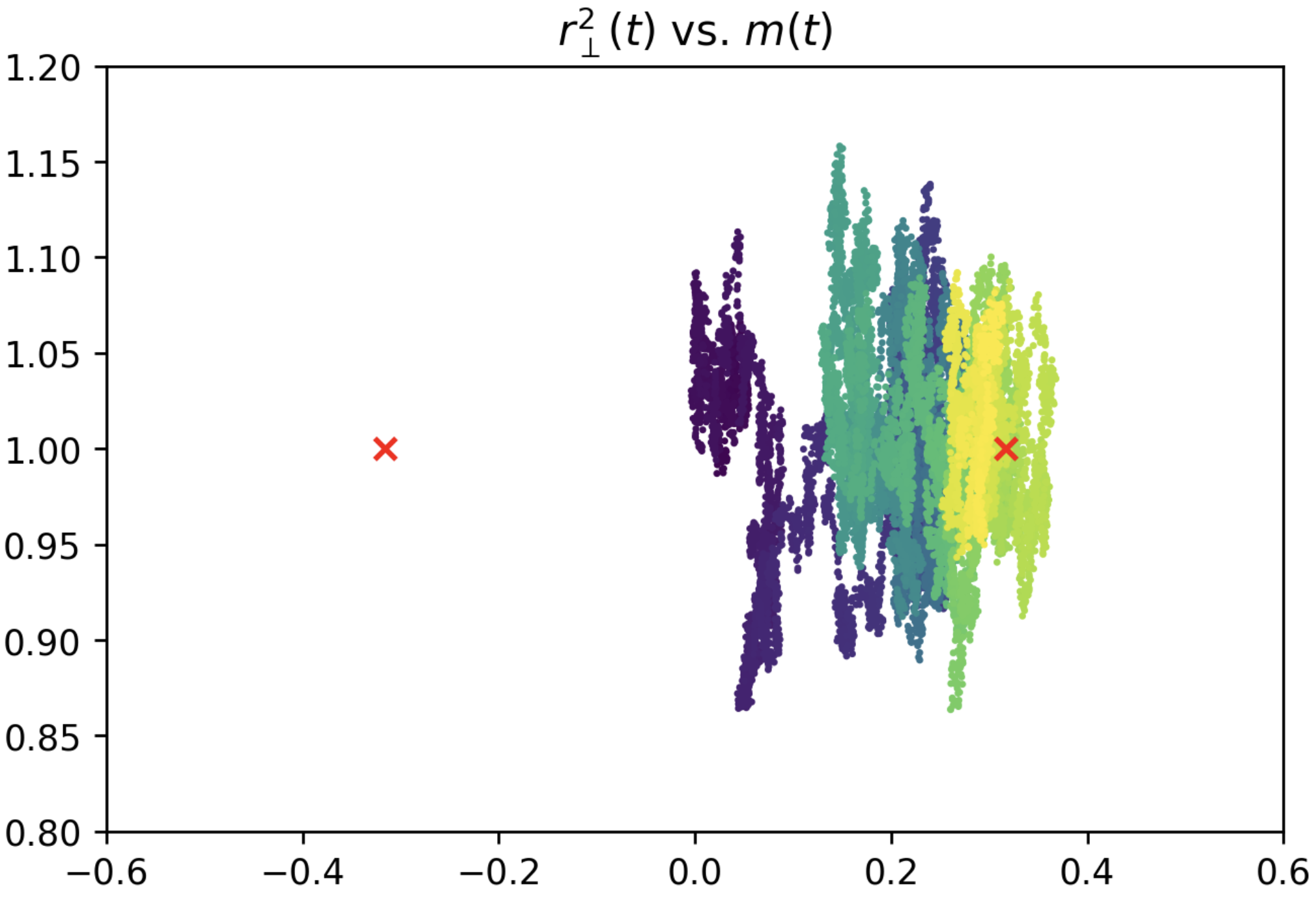}}    
  \subfigure[$\lambda = 1.2$]{       \includegraphics[width=.23\linewidth]{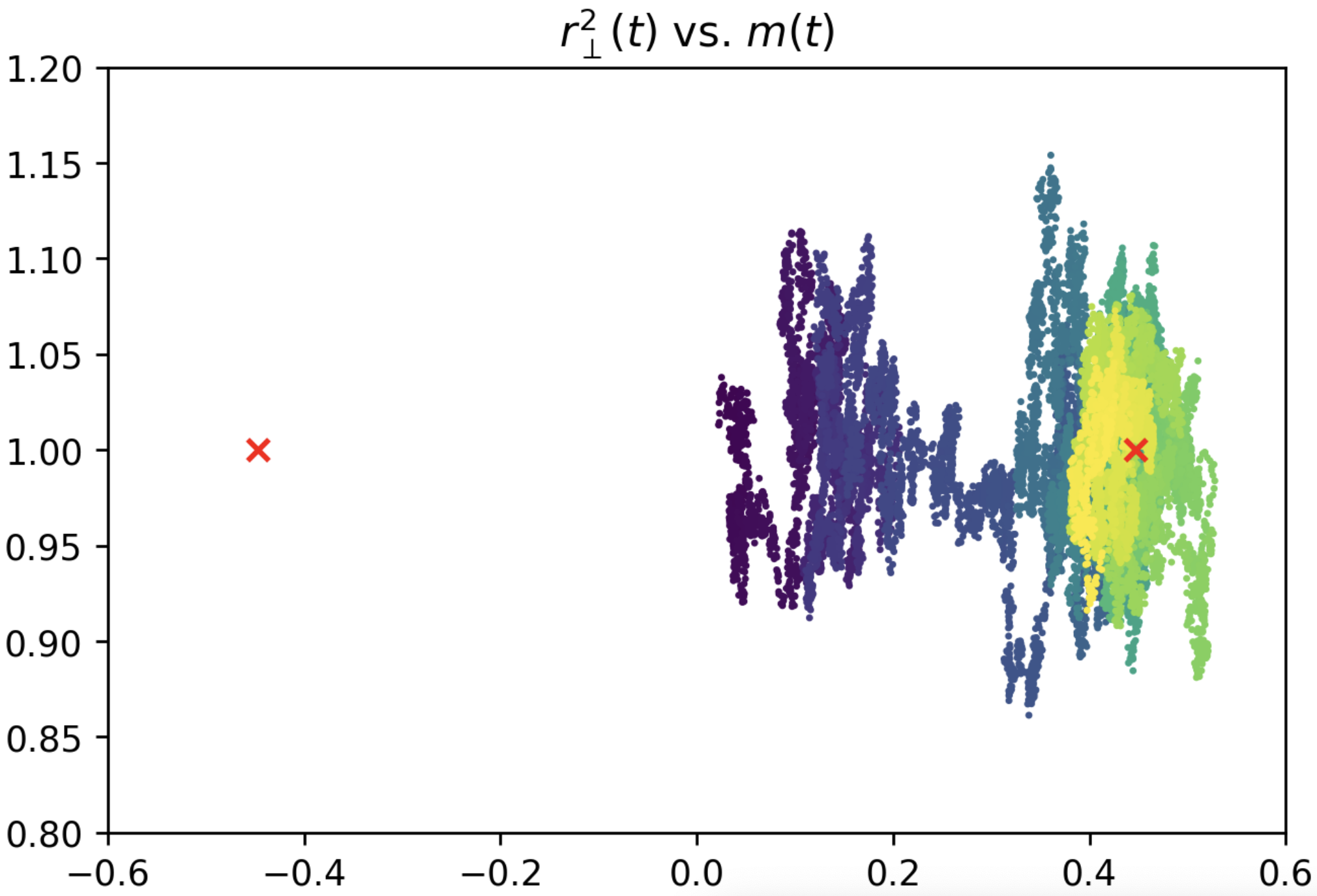}}
  \vspace{-.1cm}
  \caption{Matrix PCA in dimension $n=2000$ with various values of $\lambda$ near the critical $\lambda=1$. Depicted is the evolution of summary statistics $(m,r_\perp^2)$ for $10n$ steps of SGD initialized randomly.}\label{fig:MatrixPCA-r-v-m}
\end{figure}

\vspace{-.2cm}
\begin{figure}
    \centering
  \subfigure[$\lambda=0.8$]{   \includegraphics[width=.23\linewidth]{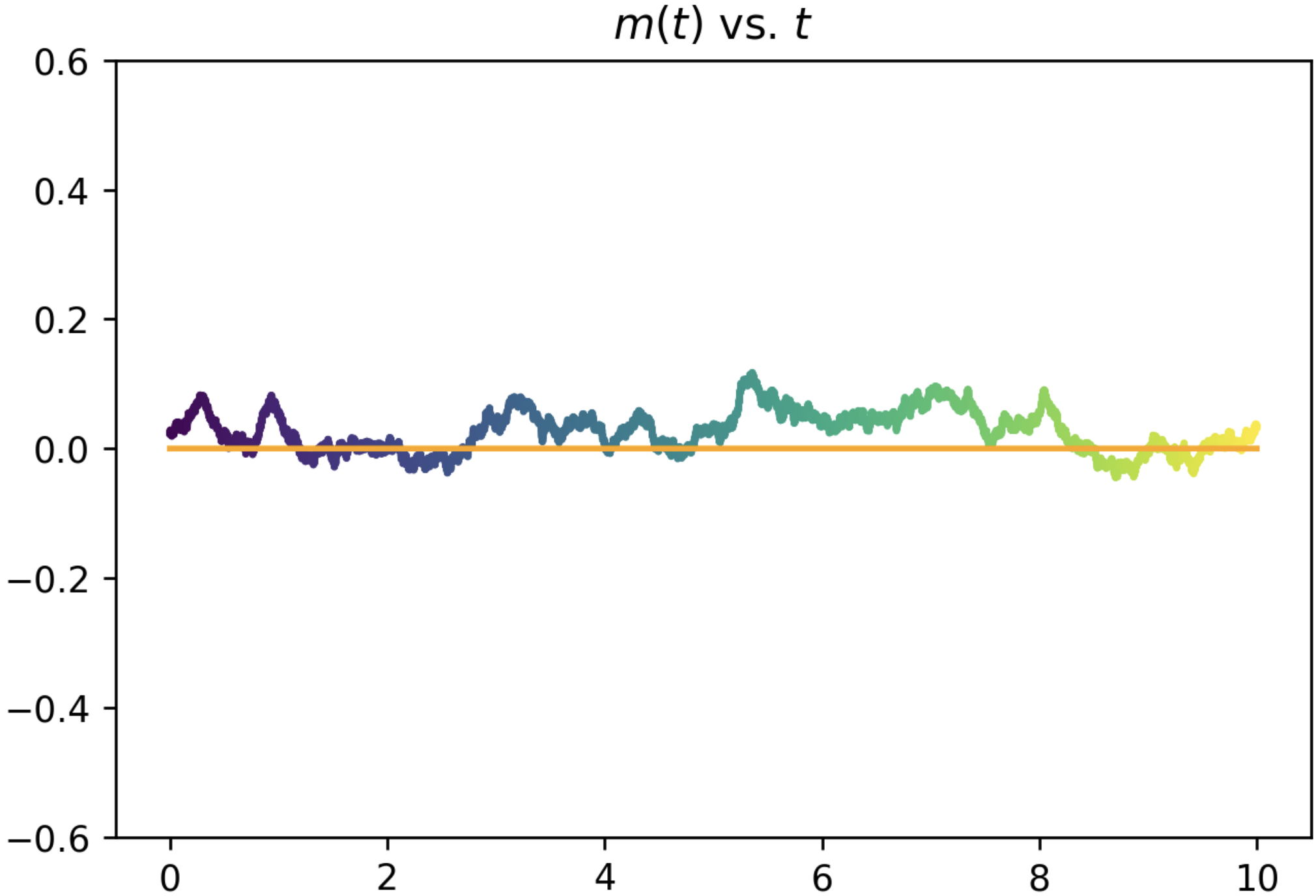}}
  \subfigure[$\lambda =0.9$]{       \includegraphics[width=.23\linewidth]{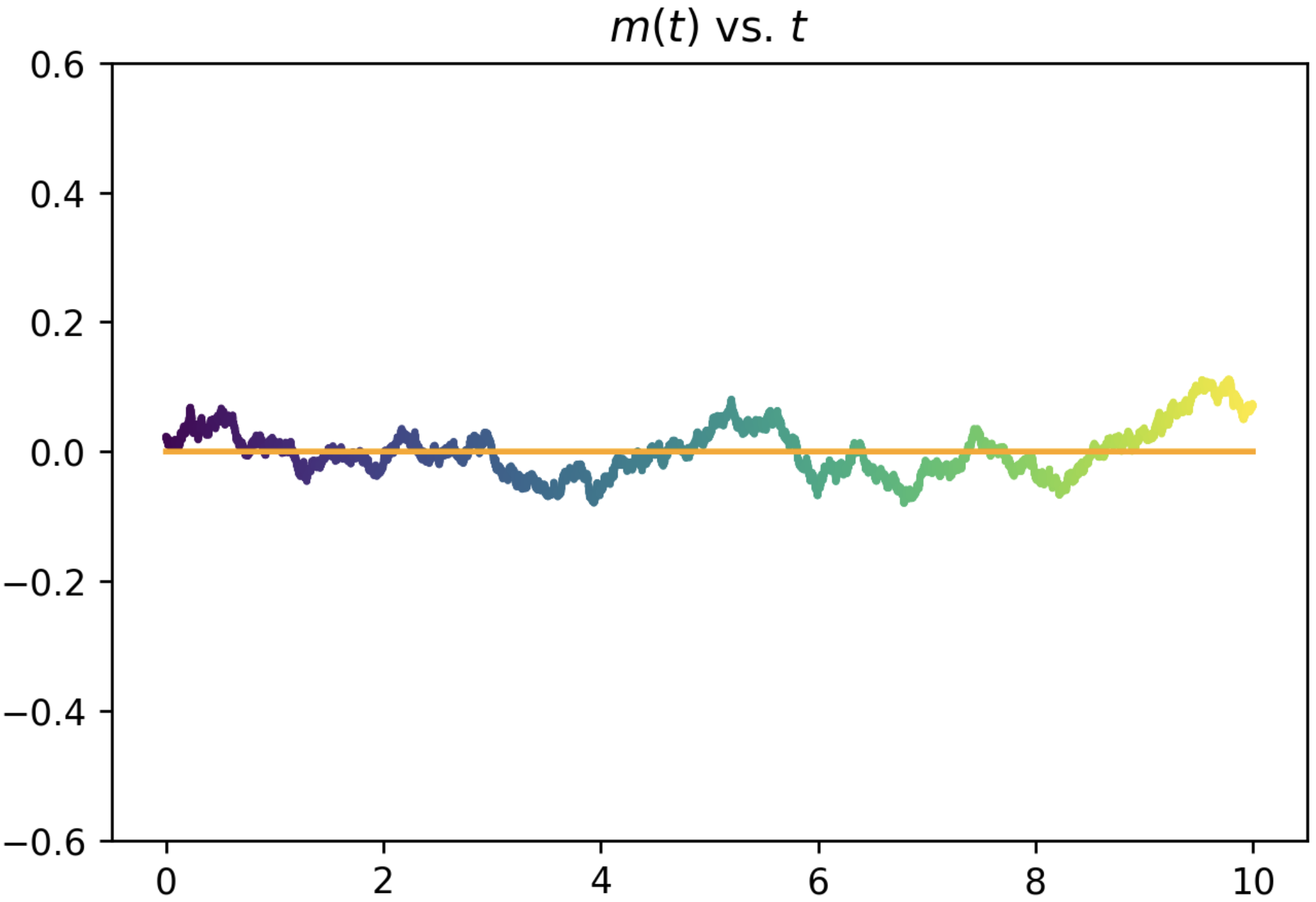}}   
    \subfigure[$\lambda = 1.1$]{       \includegraphics[width=.23\linewidth]{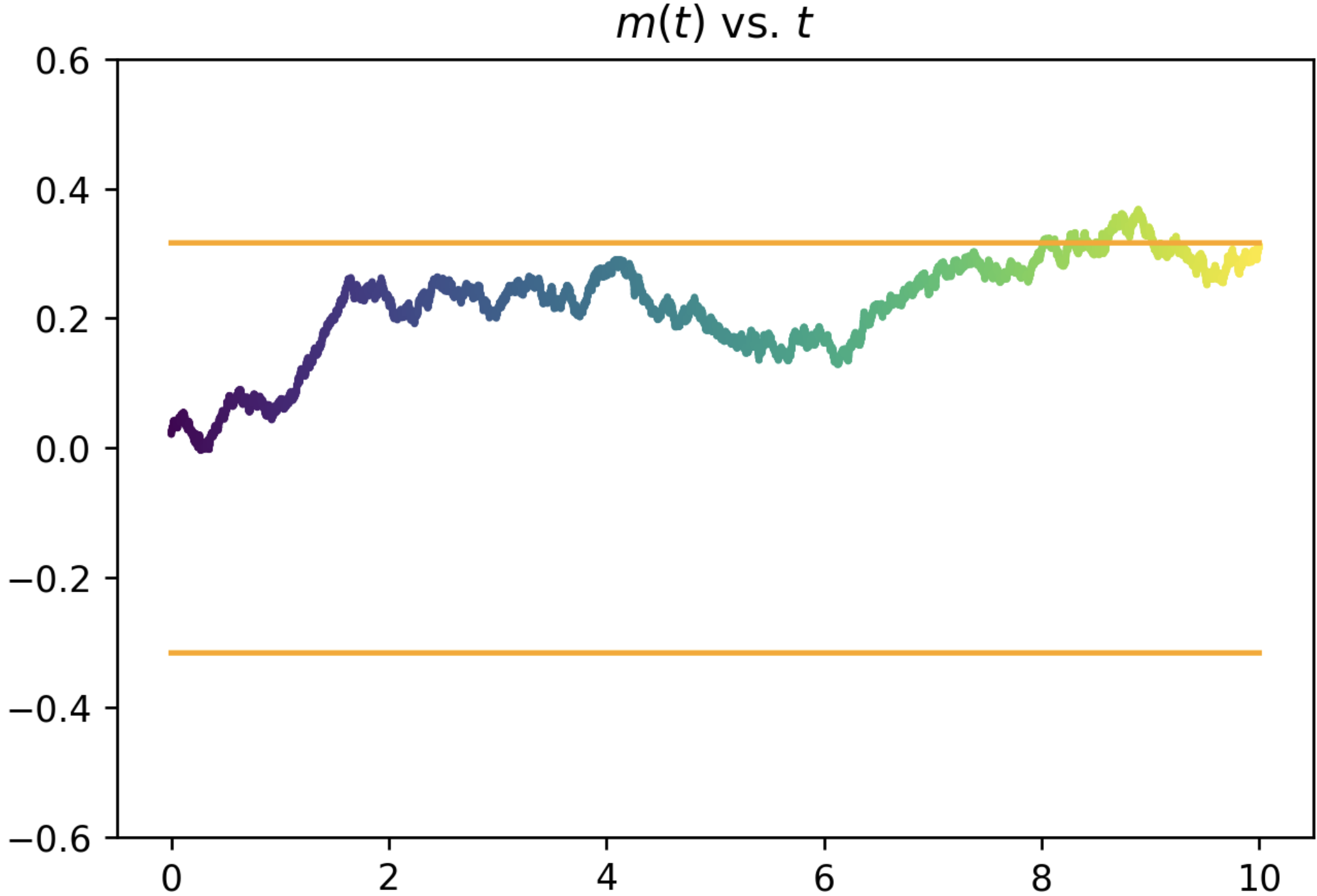}}    
  \subfigure[$\lambda = 1.2$]{       \includegraphics[width=.23\linewidth]{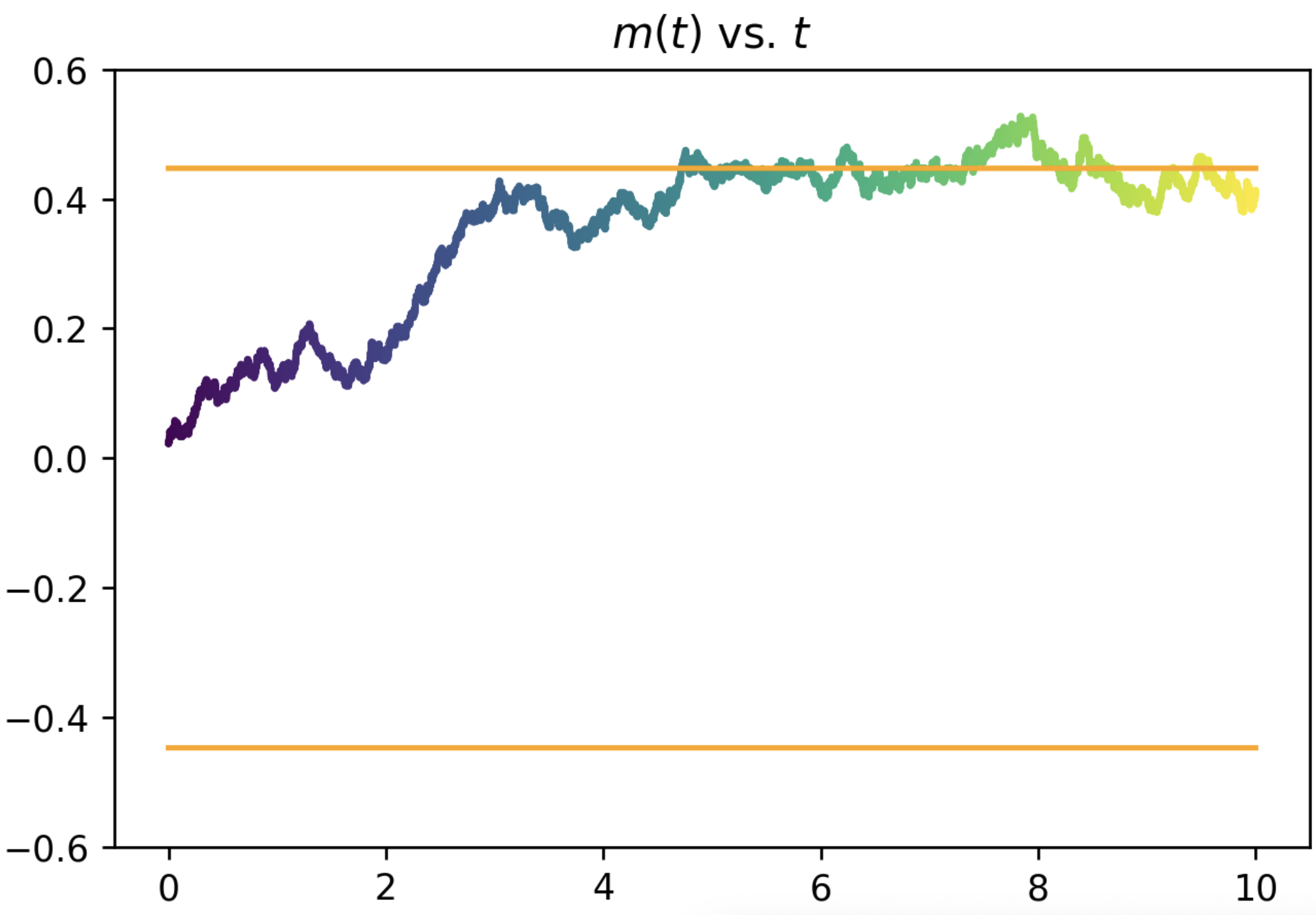}} 
  \vspace{-.1cm}   
  \caption{Matrix PCA in dimension $n=2000$ with various values of $\lambda$ near the critical $\lambda=1$. Depicted is the evolution of $m(t)$ for $10n$ steps of SGD initialized randomly.}\label{fig:MatrixPCA-m-t}
\end{figure}

\begin{figure}
    \centering
  \subfigure[$\lambda=0.8$]{   \includegraphics[width=.23\linewidth]{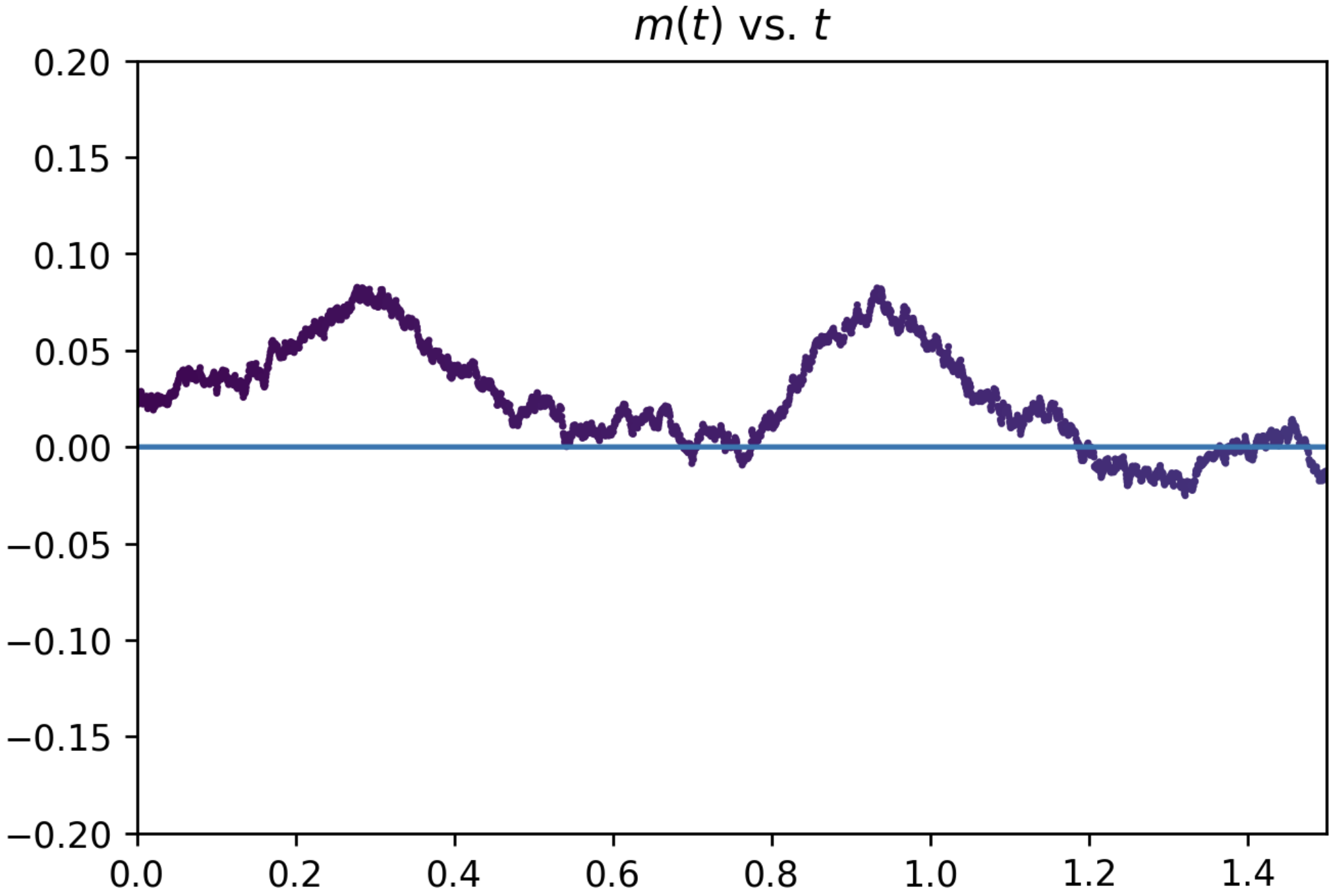}}
  \subfigure[$\lambda =0.9$]{       \includegraphics[width=.23\linewidth]{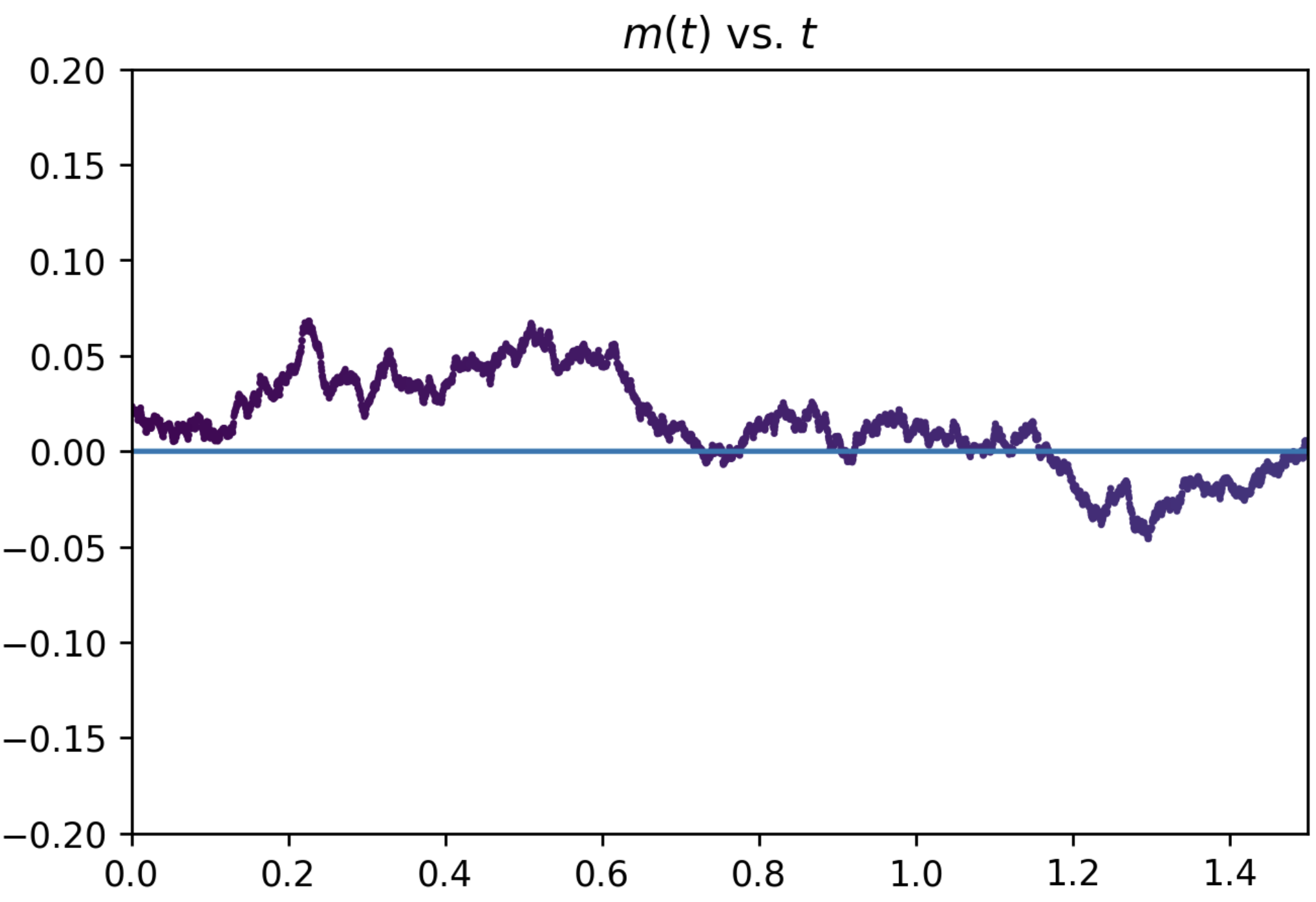}}   
    \subfigure[$\lambda = 1.1$]{       \includegraphics[width=.23\linewidth]{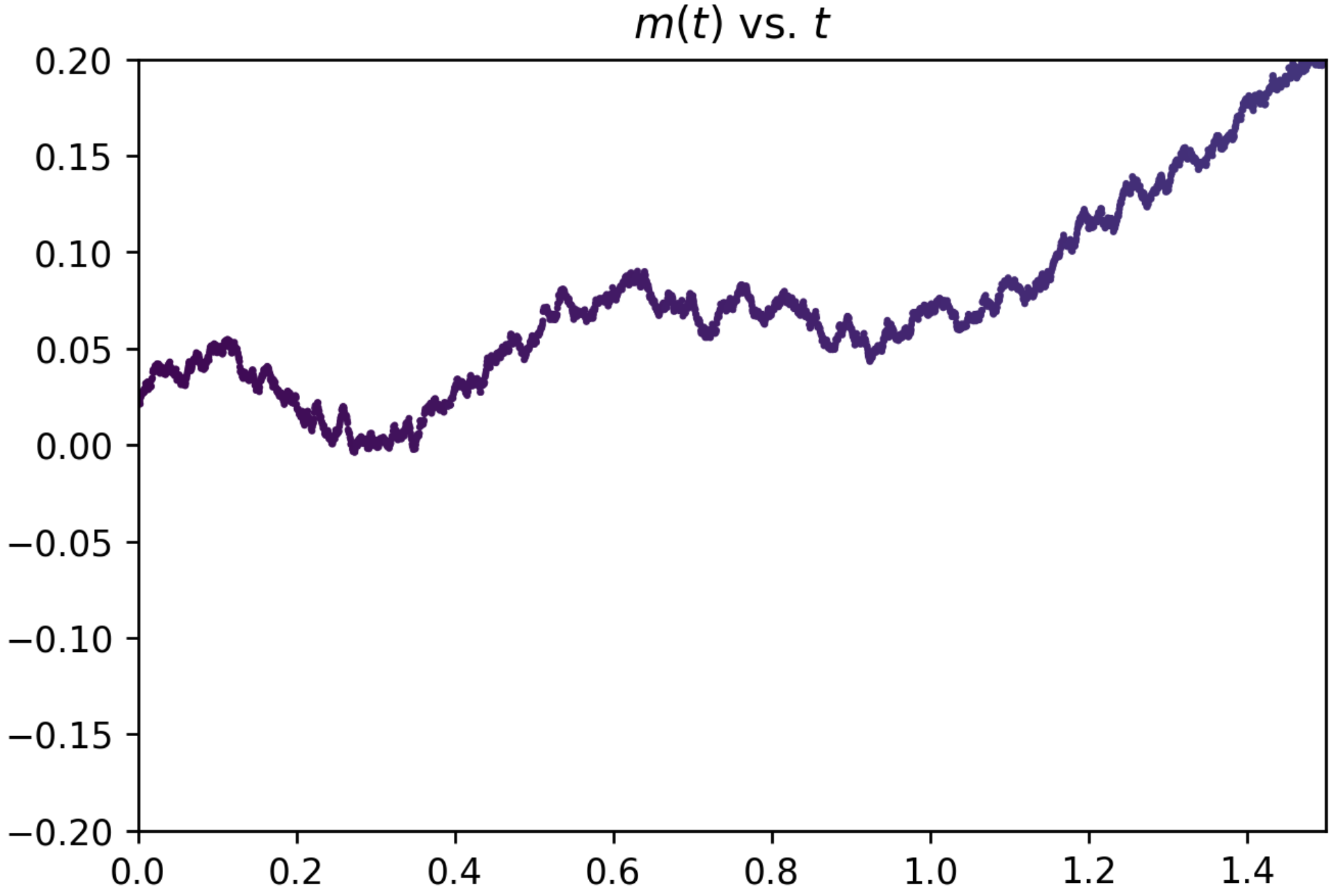}}    
  \subfigure[$\lambda = 1.2$]{       \includegraphics[width=.23\linewidth]{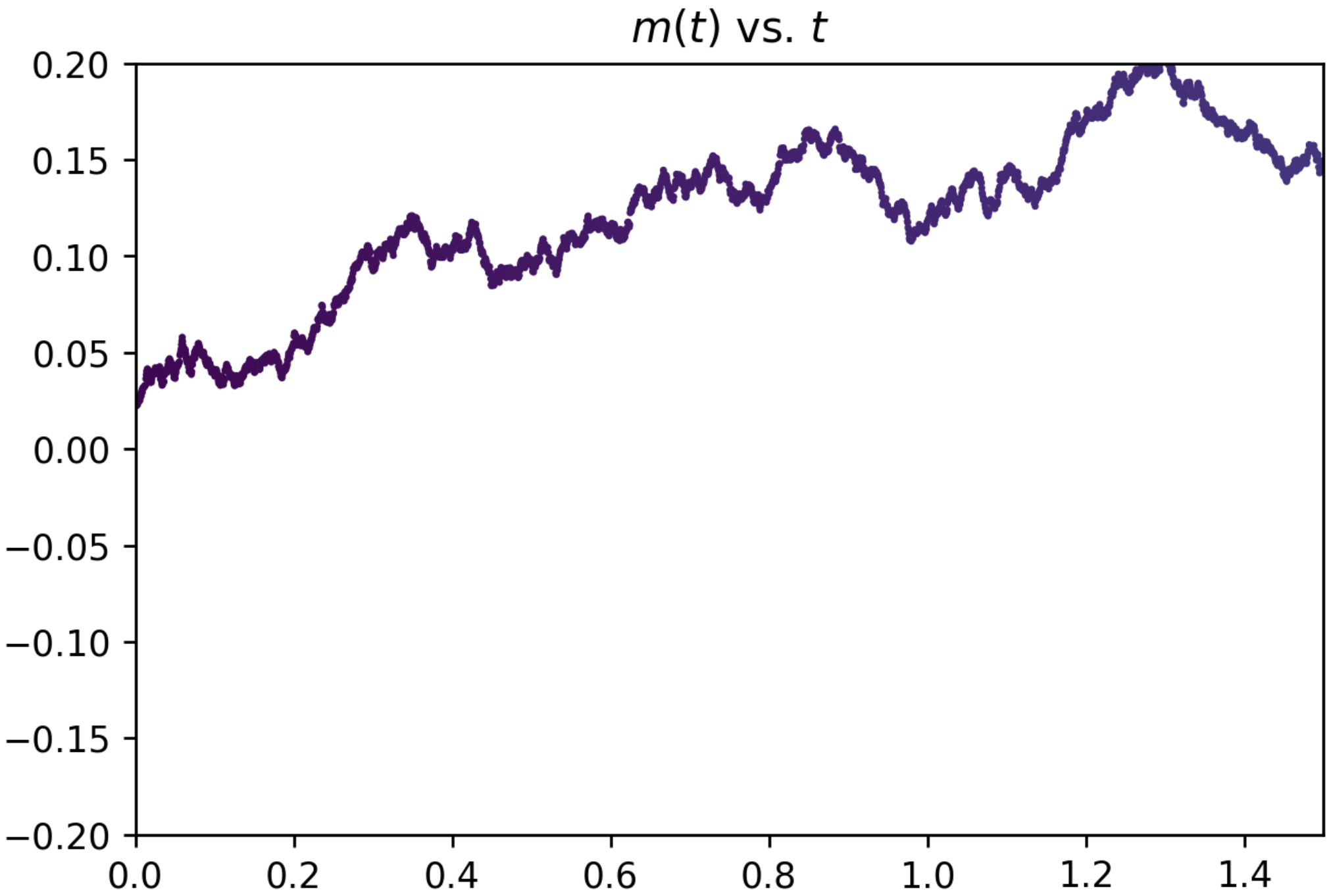}}    
  \caption{Matrix PCA in dimension $n=2000$ with various values of $\lambda$ near the critical $\lambda=1$. Depicted is the evolution of summary statistic $m(t)$ zoomed in about an $O(n^{-1/2})$ window of $m=0$ for $1.5*n$ steps of SGD initialized randomly. In (a)--(b) one sees stable OU processes, and in (c)--(d) one sees unstable OU processes.}\label{fig:MatrixPCA-m-t-zoom}
\end{figure}

\begin{figure}
    \centering
  \subfigure[$\lambda=0.8$]{   \includegraphics[width=.23\linewidth]{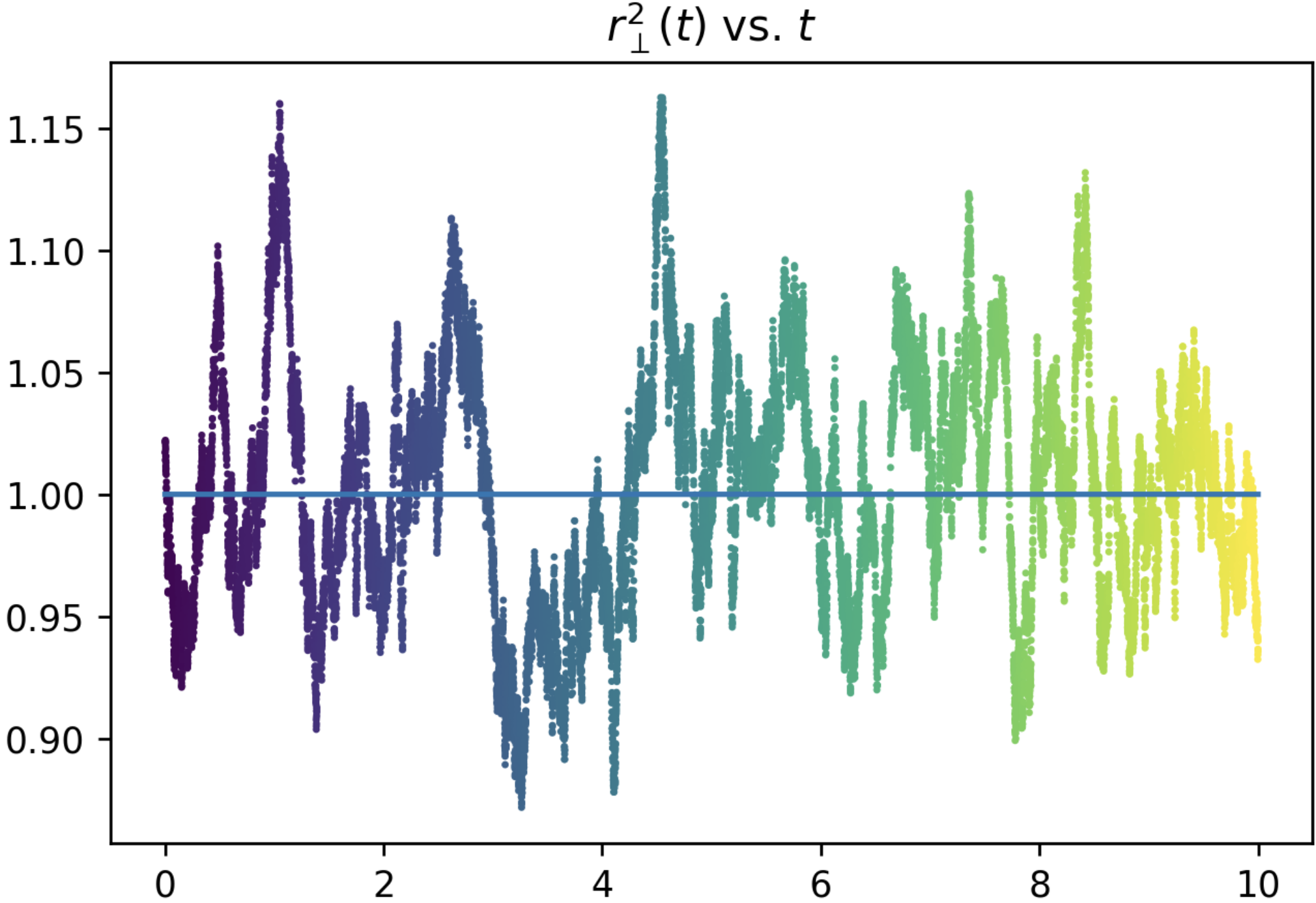}}
  \subfigure[$\lambda =0.9$]{       \includegraphics[width=.23\linewidth]{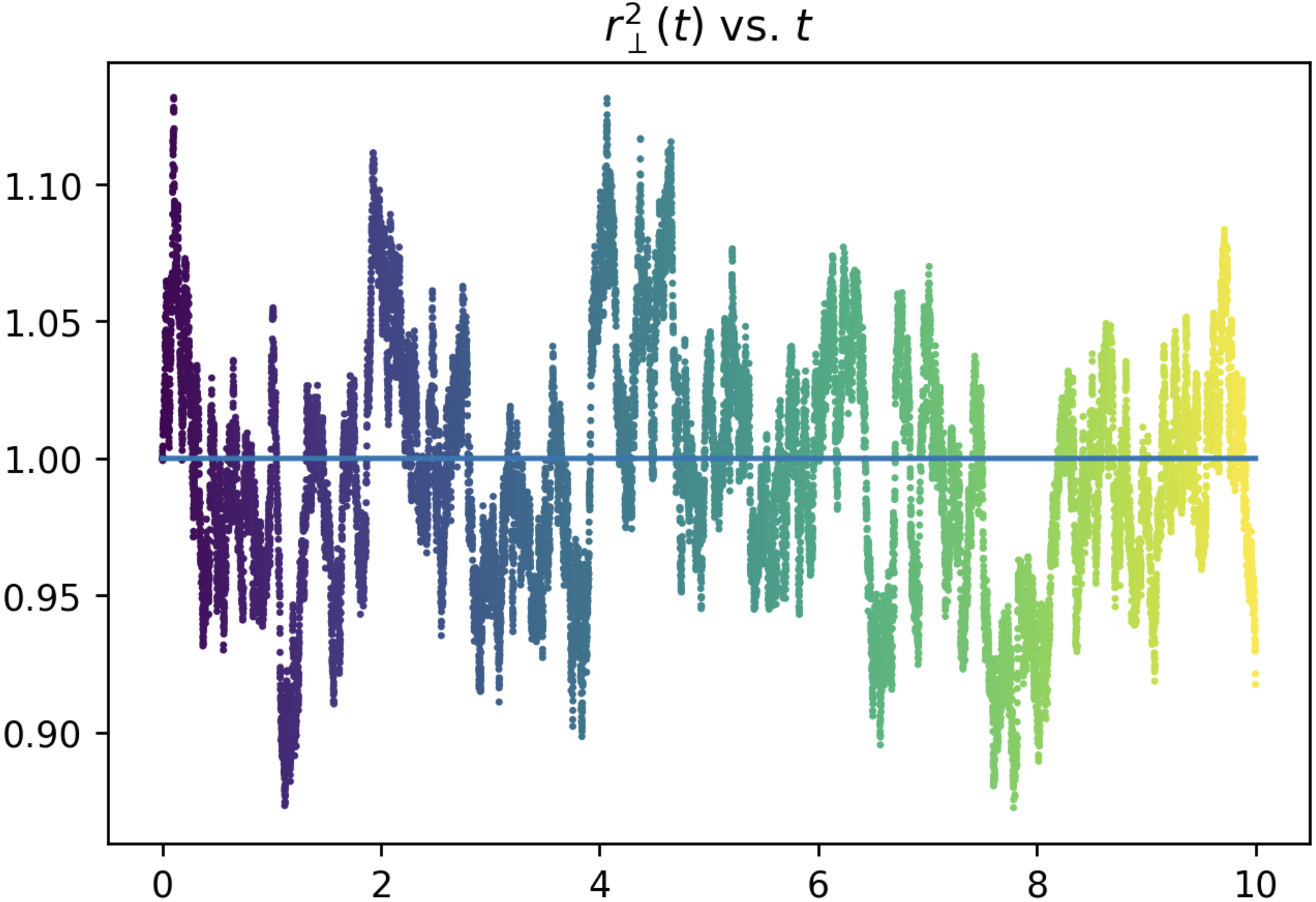}}   
    \subfigure[$\lambda = 1.1$]{       \includegraphics[width=.23\linewidth]{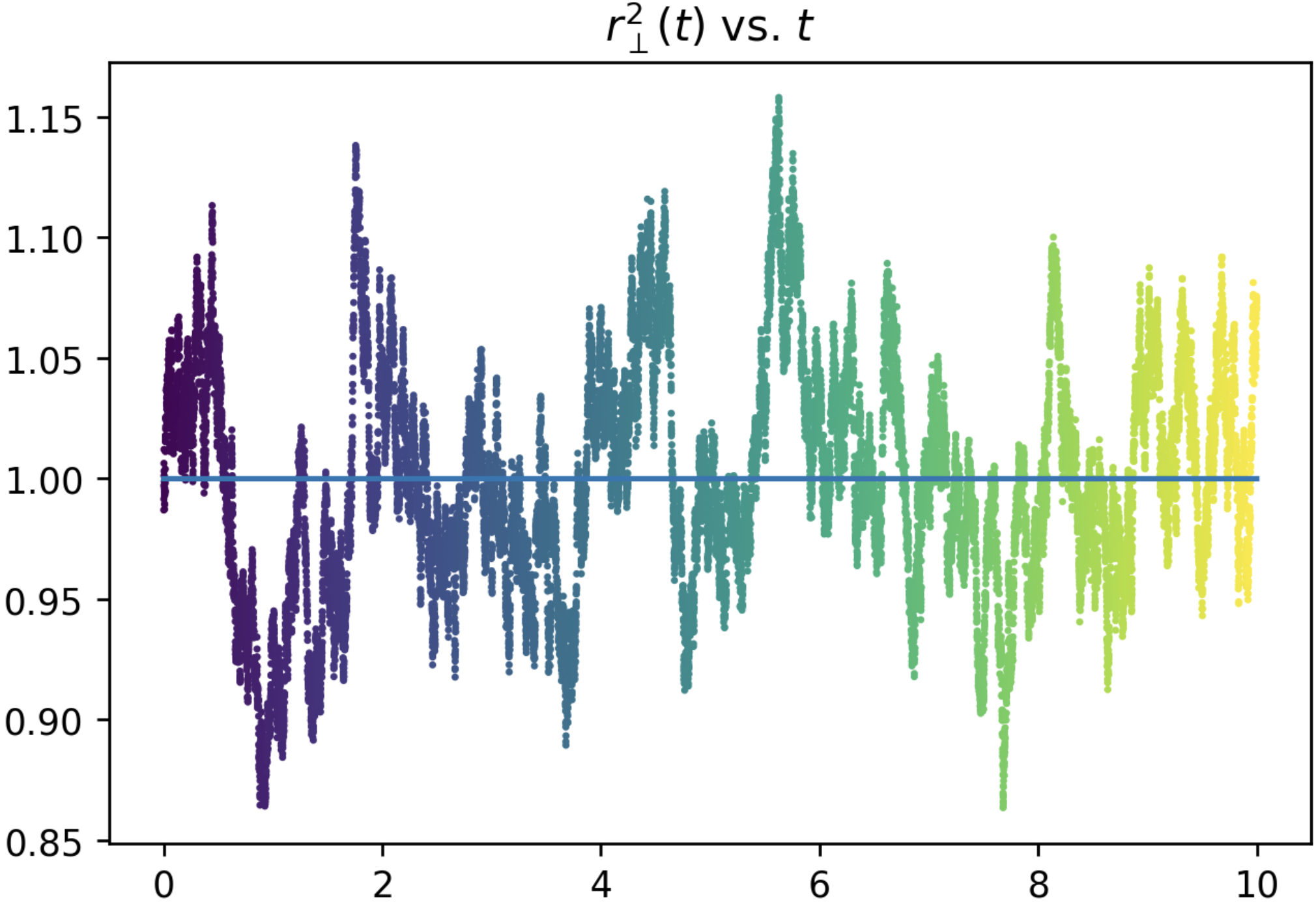}}    
  \subfigure[$\lambda = 1.2$]{       \includegraphics[width=.23\linewidth]{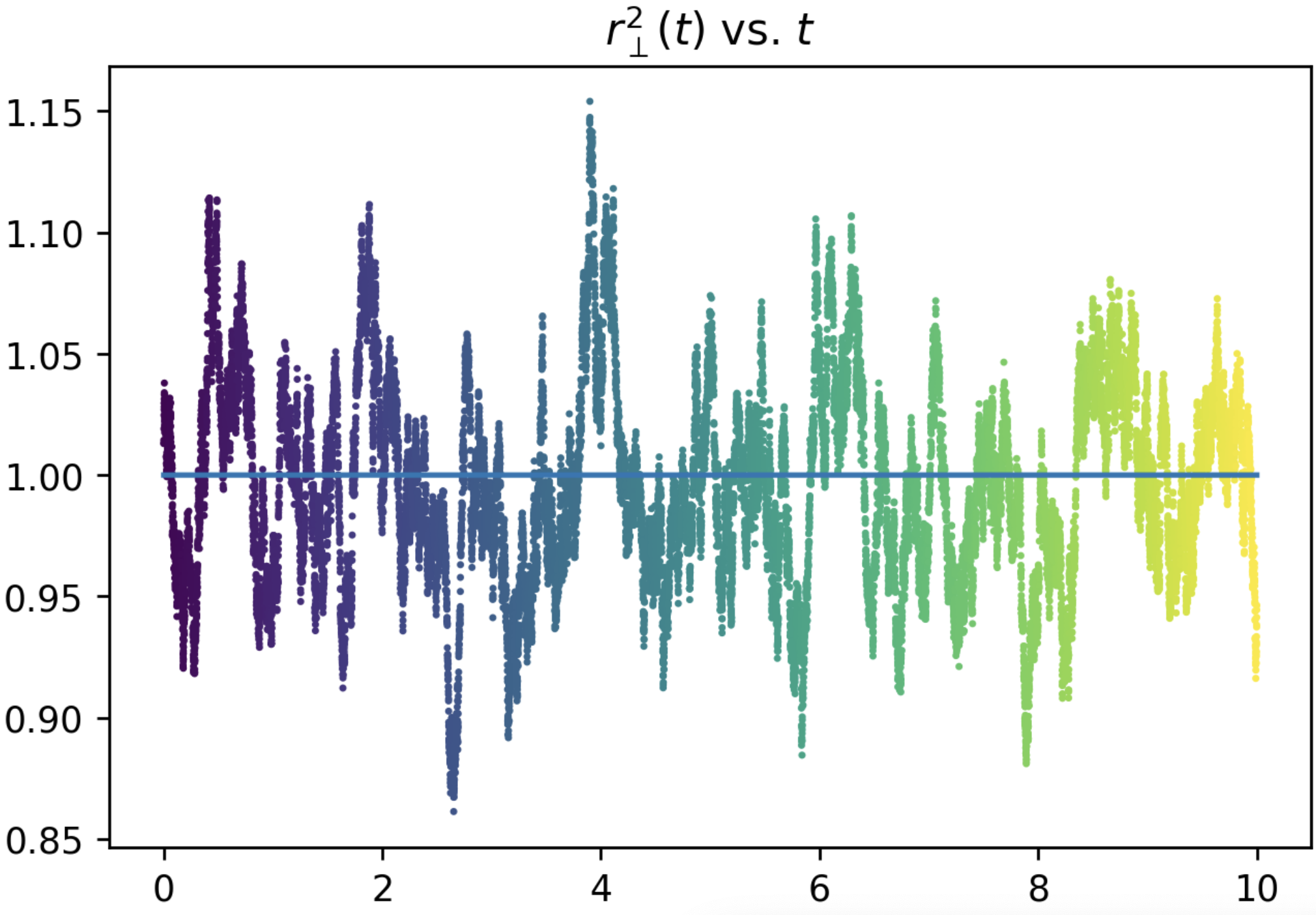}}    
  \caption{Matrix PCA in dimension $n=2000$ with various values of $\lambda$ near the critical $\lambda=1$. Depicted is the evolution of summary statistic $r_\perp^2(t)$ for $10n$ steps of SGD initialized randomly. This follows a stable OU process independent of $\lambda$.}\label{fig:MatrixPCA-r-t}
\end{figure}

\section{Two-layer networks for classifying a binary Gaussian mixture}\label{sec:binary-gmm}

\subsection{Model and Background}
As a warm-up to the XOR problem that we will consider in Section~\ref{sec:XOR-results}, we 
consider the problem of supervised classification of a binary Gaussian mixture model (binary GMM)
which is defined as follows.  Suppose that we are given i.i.d. samples of the form $Y=(y ,X)$, where $y$ is
a $\{0,1\}$-valued $Ber(1/2)$ random variable and, conditionally on $y$, we have 
\[
X \sim \cN((2y-1) \mu, I/\lambda)\,,
\]
where $\mu\in\R^N$ is a fixed unit vector, $I$ is the identity on $\R^N$,
and $\lambda>0$ is the signal-to-noise ratio.  Here, $y$ is called the class label and $X$ is called the data.
Our goal is to construct an estimator, $\hat y = \hat y(X)$, of the class label, $y$, which depends on the data, $X$, alone.

It is classical \cite{anderson1962introduction} that the Bayes optimal estimator in this setting is given by $\hat{y}=\operatorname{sgn}( \mu\cdot x)$.
Furthermore, this estimator can be achieved by (a rounding of) the output of a single layer neural network trained using
the binary-cross-entropy loss \eqref{eq:Loss-bgmm}. This is also called logistic regression. The single-layer setting can be easily analyzed via our framework. Our focus here, however, is to demonstrate our analysis on multi-layer neural networks.

To that end, we consider now the same setting, except that we will estimate the class labels using a simple
two-layer neural network. (Note that the Bayes' optimal estimator is still expressible by this architecture.)  At first glance, this may seem an elementary setting with little to say. However as we will see, even in this simple setting surprising behaviour can occur in the high-dimensional setting which runs counter to common intuition.  Furthermore, as we will see in Section~\ref{sec:XOR-results}, the phenomena occurring here also appear in richer problems such as the XOR problem.

\subsection{Analysis}
For the sake of concreteness, we consider classification via the following architecture
(though our techniques generalize to other settings \emph{mutatis mutandis}):
The first layer has weights $(W_{1},W_{2})\in \R^N\times\R^N$
and ReLu activation, $g(x)=x\vee0$; 
and the second layer has weights
$v_{1},v_{2}\in\mathbb{R}$ and sigmoid activation, $\sigma(x)=1/(1+e^{-x})$. 
The output of the multi-layer network is then $\sigma(v\cdot g(WX))$
Our parameter space is then  $\cX_n=\mathbb{R}^{2N+2}$ and we therefore take $n=2N+2$ when applying
Theorem~\ref{thm:main}. 

As we are interested in supervised classification, we take the usual \emph{binary cross-entropy
loss} with $\ell^{2}$ regularization. In our setting, this reduces to optimizing
\begin{align}\label{eq:Loss-bgmm}
L\big((v_{i},W_{i})_{i\in\{1,2\}};(y,X)\big)=-yv\cdot g(WX)+\log(1+e^{v\cdot g(WX)})+p(v,W)\,,
\end{align}
where $g$ is applied component wise 
and $p(v,W):=(\alpha/2)(\norm{v}^{2}+\norm{W}^{2})$.

It can be shown {(see  Lemma~\ref{lem:summary-stats-bGMM}) 
that the law of the loss at a given point, $(v,W)\in\cX_n$,
depends only on the $7$ summary statistics, 
\begin{equation}\label{eq:u-bgmm}
\bu_n = (v_1,v_2,m_1,m_2, R_{11}^\perp, R_{12}^{\perp}, R_{22}^{\perp}),
\end{equation}
where $m_i = W_i\cdot \mu$ and $R_{ij}^\perp = W_i^\perp\cdot W_j^\perp$ 
with $W_{i}^{\perp}=W_{i}-m_{i}\mu$ denoting the part of $W_{i}$ orthogonal to $\mu$.
For a point, $(v,W)\in \cX_n$, let 
\begin{align}
\mathbf{A}_{i}^\mu=\mathbb{E}[X\!\cdot\!\mu\mathbf{1}_{W_{i}\cdot X\ge0}&(\sigma(v\!\cdot\!g(WX))-y)]\,, \qquad
\mathbf{A}^\perp_{ij}=\mathbb{E}[X\!\cdot\! W_j^\perp\mathbf{1}_{W_{i}\cdot X\ge0}(\sigma(v\!\cdot\!g(WX))-y)]\,, \nonumber\\
\mathbf{B}_{ij}&=\mathbb{E}[\mathbf{1}_{W_{i}\cdot X\ge0} \mathbf{1}_{W_{j}\cdot X\ge0}(\sigma(v\! \cdot \!g(WX))-y)^{2}]. \label{eq:bgmm-A-B}
\end{align}
By similar reasoning to Lemma~\ref{lem:summary-stats-bGMM}, it can be seen that these are functions only of $\bu_n$, and we denote them as such, e.g.,  $\mathbf A_i^\mu= \mathbf A_i^\mu(\bu_n)$. See Section~\ref{sec:binary-gmm-proofs}.
The critical scaling for $\delta$ is then of order $\Theta(1/n)$ and we obtain the following.

\begin{prop}\label{prop:bgmm-ballistic}
Let $\mathbf{u}_{n}$ be as in \eqref{eq:u-bgmm} and fix any $\lambda>0$ and
$\delta_{n}=\sfrac{c_{\delta}}{N}$. Then 
$\mathbf{u}_{n}(t)$ converges to the solution of the ODE system,
$\dot{\mathbf{u}}_{t}=  -\mathbf{f}(\mathbf{u}_{t})+\mathbf{g}(\mathbf{u}_{t})$,
initialized from $\lim_{n\to\infty}(\mathbf{u}_{n})_{*}\mu_{n}$,
with:
\begin{align*}
f_{v_i} & =m_i \mathbf{A}_{i}^\mu(\bu)+\mathbf{A}^\perp_{ii}(\bu)+\alpha v_{i},  \qquad \qquad \qquad  f_{m_i}  =v_i \mathbf{A}_{i}^\mu(\bu)+\alpha m_{i}, \\
& \qquad \qquad \qquad  f_{R_{ij}^{\perp}}  = v_i \mathbf A_{ij}^\perp (\bu) + v_j \mathbf{A}_{ji}^\perp (\bu)  + 2\alpha R_{ij}^\perp\,, 
\end{align*}
and correctors ${g}_{v_i}=g_{m_i}=0$,
${g}_{R_{ij}^\perp}=c_{\delta}\frac{v_i v_j }{\lambda}\mathbf{B}_{ij}$ for $i,j = 1,2$. 
\end{prop}

\subsection{Low variance asymptotics}
Due to the Gaussian integrals defining $\mathbf f,\mathbf g$, it is difficult to 
analyze the ODE system defined by Proposition~\ref{prop:bgmm-ballistic}, let alone any rescaled effective dynamics. 
For ease of analysis, we next  send $\lambda\to\infty$ corresponding to a small
noise regime for the Gaussian mixture. We emphasize that this limit
is taken after $n\to\infty$ and therefore is still approximately
on the critical scale of $\lambda=\Theta(1)$ at which there is a transition
in the existence of any fixed point which is a good classifier. In particular, if $\lambda = \lambda_n$ is any diverging sequence, then the limiting effective dynamics would exactly match that attained by now sending $\lambda\to\infty$. 
In Figure~\ref{fig:GMM-endpoints-various-lambda}, we demonstrate numerically that the following predicted fixed points from the $\lambda\to\infty$ limit match those arising at finite large $N$ and $\lambda>0$.\footnote{For large $\lambda$, this is indeed a quantitative approximation as $\mathbf f,\mathbf g$ exhibit locally Lipschitz dependence on $\lambda^{-1}$, so
the corresponding dynamics converges as $\lambda\to\infty$ by classical well-posedness results (see, e.g., \cite{teschl2012ordinary})}

\begin{prop}\label{prop:bgmm-ballistic-noiseless}
The $\lambda\to\infty$ limit of the ODE system of Proposition~\ref{prop:bgmm-ballistic} is given by 
\begin{align*}
\dot{m}_{i} & =\begin{cases}
\tfrac{v_{i}}{2}\sigma(-v\cdot m)-\alpha m_{i} & m_{1}m_2 >0\\
\tfrac{v_{i}}{2}\sigma(-v_{i}m_{i})-\alpha m_{i} & else\\
\end{cases}\,, & \dot{v}_{i} & =\begin{cases}
\tfrac{m_{i}}{2}\sigma(-v\cdot m)-\alpha v_{i} &m_{1}m_2>0\\
\tfrac{m_{i}}{2}\sigma(-v_{i}m_{i})-\alpha v_{i} & else
\end{cases}\,,
\end{align*}
and $\dot R_{ij}^\perp = -2\alpha R_{ij}^\perp$. 
The fixed points of this system are classified as follows. All fixed points
have $R_{ij}^\perp=0$ and $m_{i}=v_{i}$ for $i,j=\{1,2\}$. In $(v_{1},v_{2})$, the coordinates are classified by 
\begin{enumerate}
\item A fixed point at $(v_{1},v_{2})=(0,0)$ that is stable if $\alpha>\sfrac{1}{4}$;
\item If $\alpha<\sfrac{1}{4}$, two unstable sets of fixed points at the quarter-circles given by $(v_{1},v_{2})$
having $v_1v_2>0$ such that $v_{1}^{2}+v_{2}^{2}=C_{\alpha}$ for $C_{\alpha}:=\log(1-2\alpha)-\log(2\alpha)$.
\item If $\alpha<\sfrac{1}{4}$, two stable fixed points at $(v_{1},v_{2})$ equals $(\sqrt{C_\alpha},-\sqrt{C_\alpha})$
and $(-\sqrt{C_{\alpha}},\sqrt{C_{\alpha}})$.
\end{enumerate}
If $\mu_{n}$ is e.g., given by $(v_{1},v_{2})\sim\mathcal{N}(0,I_2)$
and $W_{1},W_{2}\sim\mathcal{N}(0,I_{N}/(\lambda N))$ then $\nu :=\lim (\mathbf{u}_{n})_{*}\mu_{n}$ is $\mathcal{N}(0,I_{2})$ in the $v_{1},v_{2}$ coordinates, and is in
the basin of attraction of the quarter-circles of item (2) with probability
$\sfrac 1 2$ and the basin of attraction of the stable fixed points of (3)
with probability $\sfrac 1 2$. 
\end{prop}

\begin{figure}
    \centering
  \subfigure[$\lambda=1$]{\includegraphics[width=.24\linewidth]{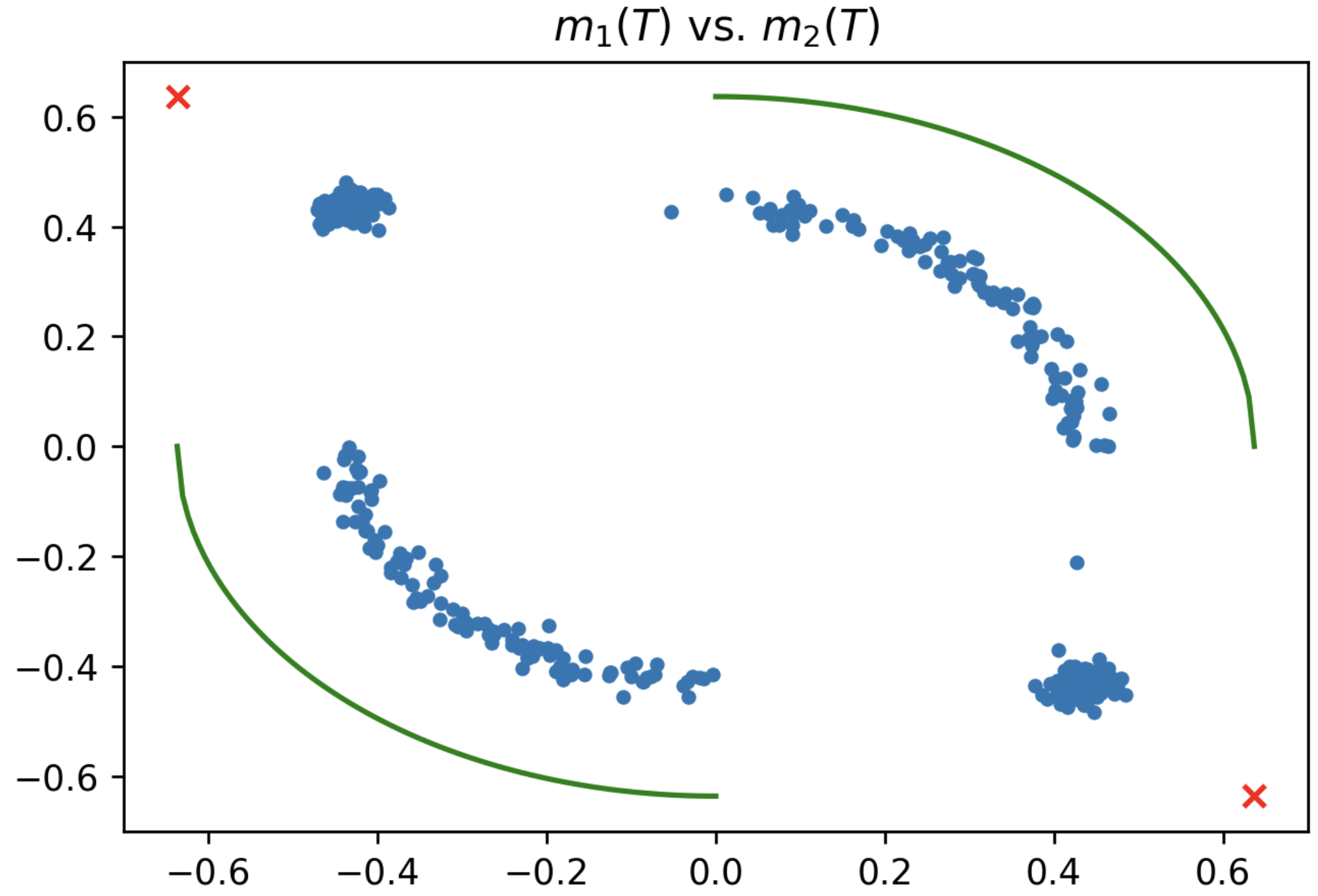}}
  \subfigure[$\lambda =5$]{\includegraphics[width=.24\linewidth]{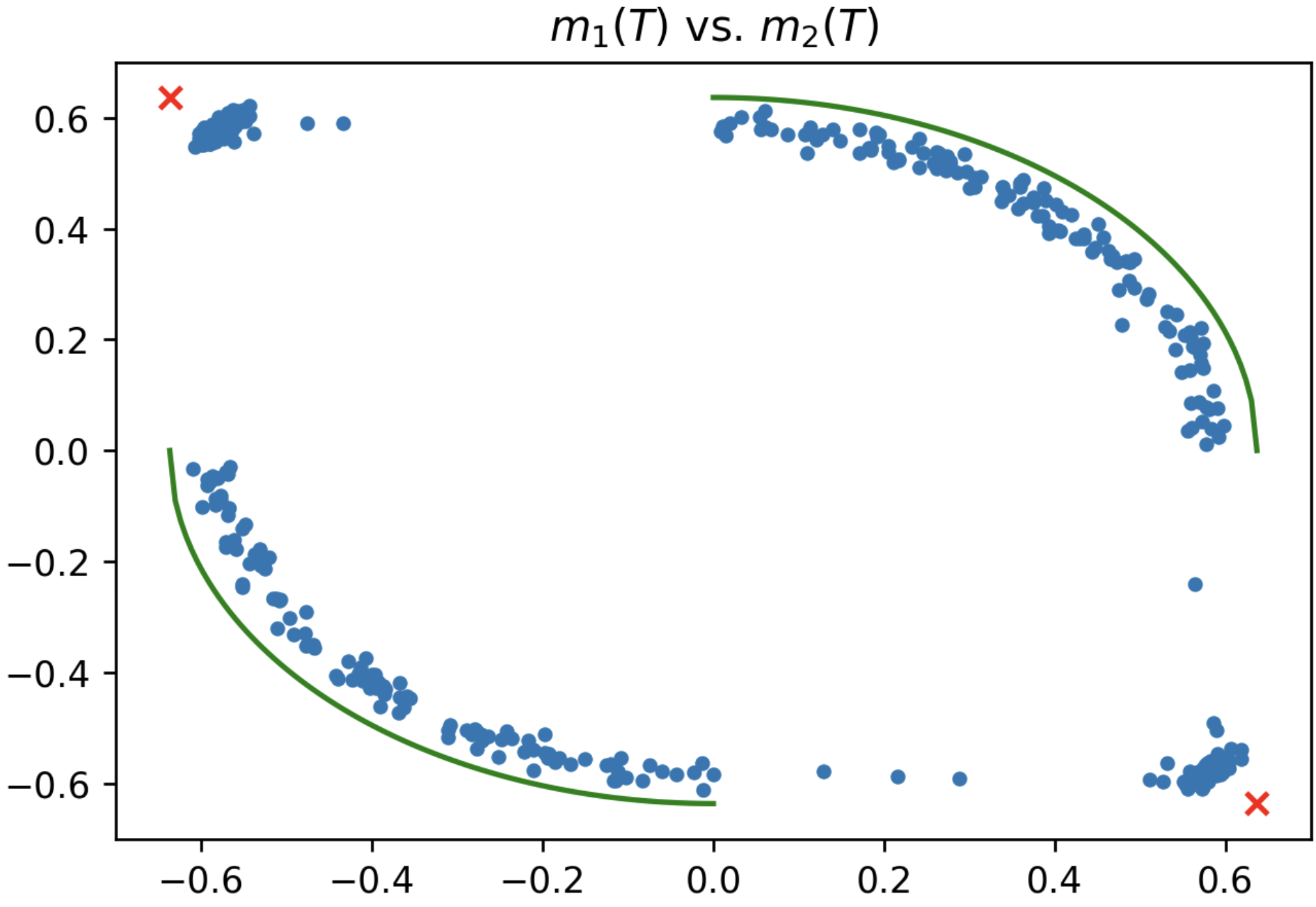}}   
    \subfigure[$\lambda = 10$]{\includegraphics[width=.24\linewidth]{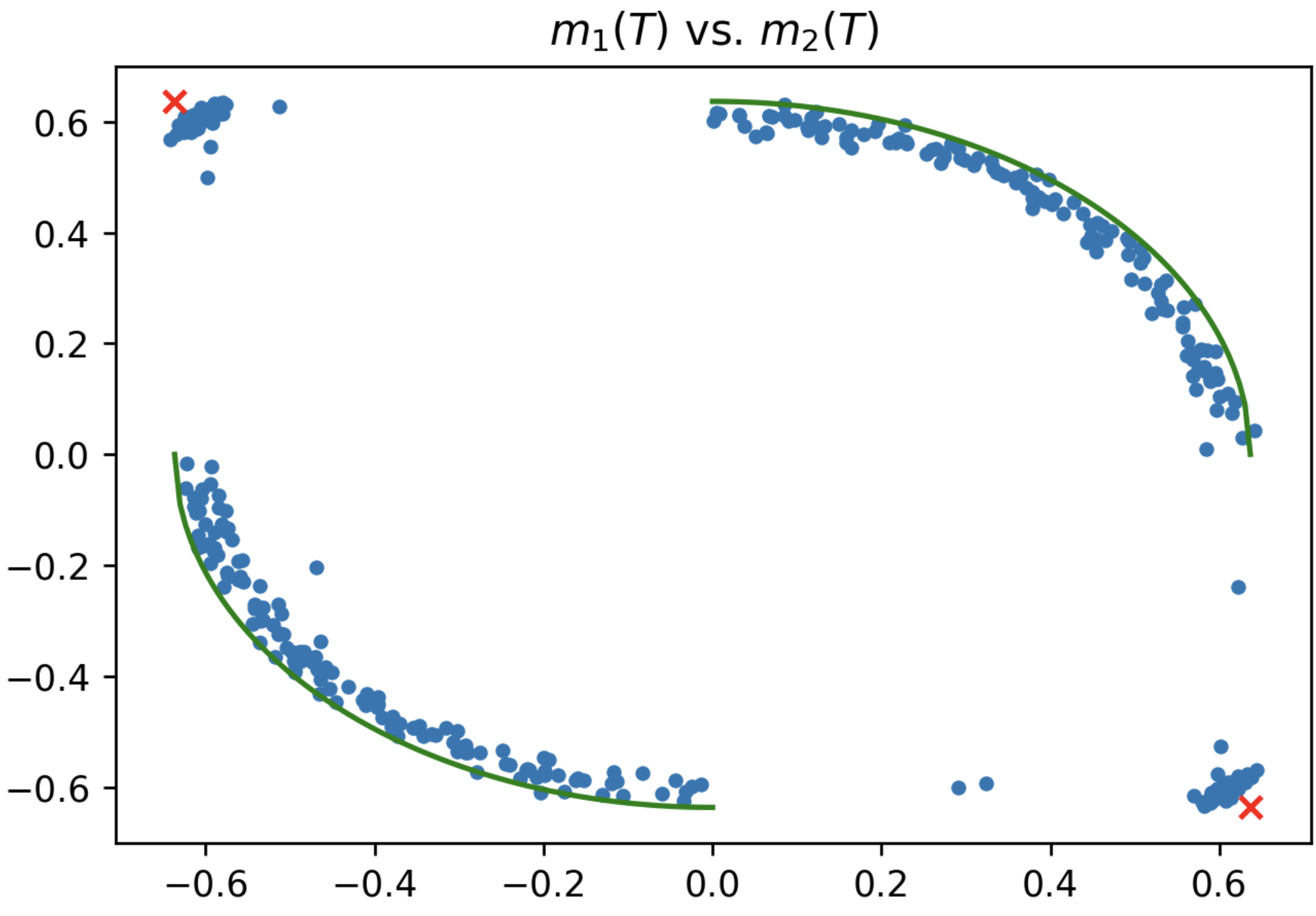}}    
  \subfigure[$\lambda = 100$]{\includegraphics[width=.24\linewidth]{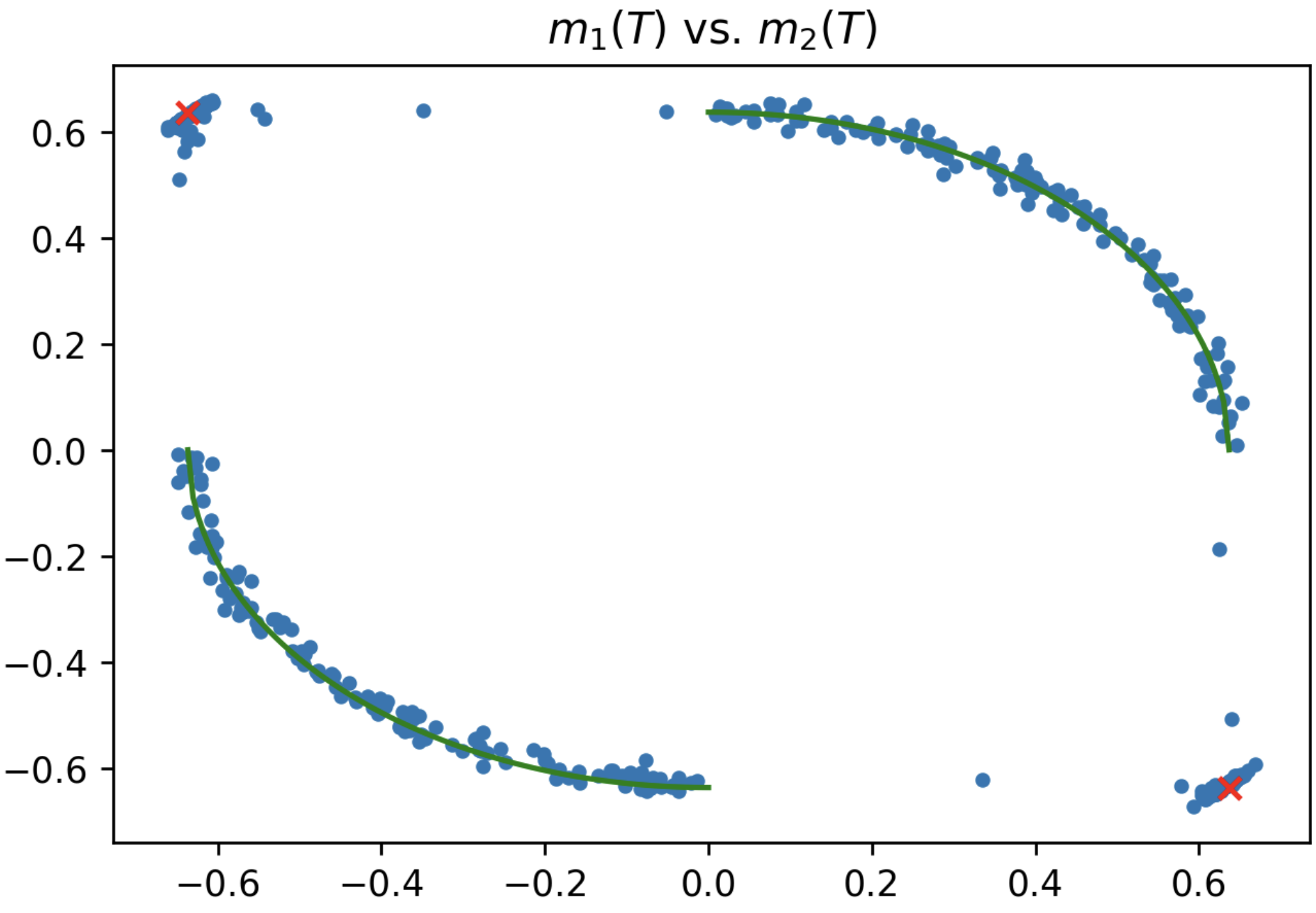}}    
  \caption{GMM in dimension $N=500$ with $\alpha = 0.1$ at various values of $\lambda$. Depicted are $(m_1,m_2)$ values that 500 runs of SGD converge to after $100N$ steps from a random Gaussian initialization. The ${\color{green!50!black}-}$ and ${\color{red}\times}$ are the unstable and stable fixed points of the $\lambda = \infty$ ballistic effective dynamics. The fixed points of the limiting effective dynamics have the same structure at finite $\lambda$ as $\lambda = \infty$, and that as $\lambda$ gets large quantitatively approach the $\lambda=\infty$ ones.}\label{fig:GMM-endpoints-various-lambda}
\end{figure}

\begin{figure}
    \centering
  \subfigure[$\lambda=1$]{   \includegraphics[width=.23\linewidth]{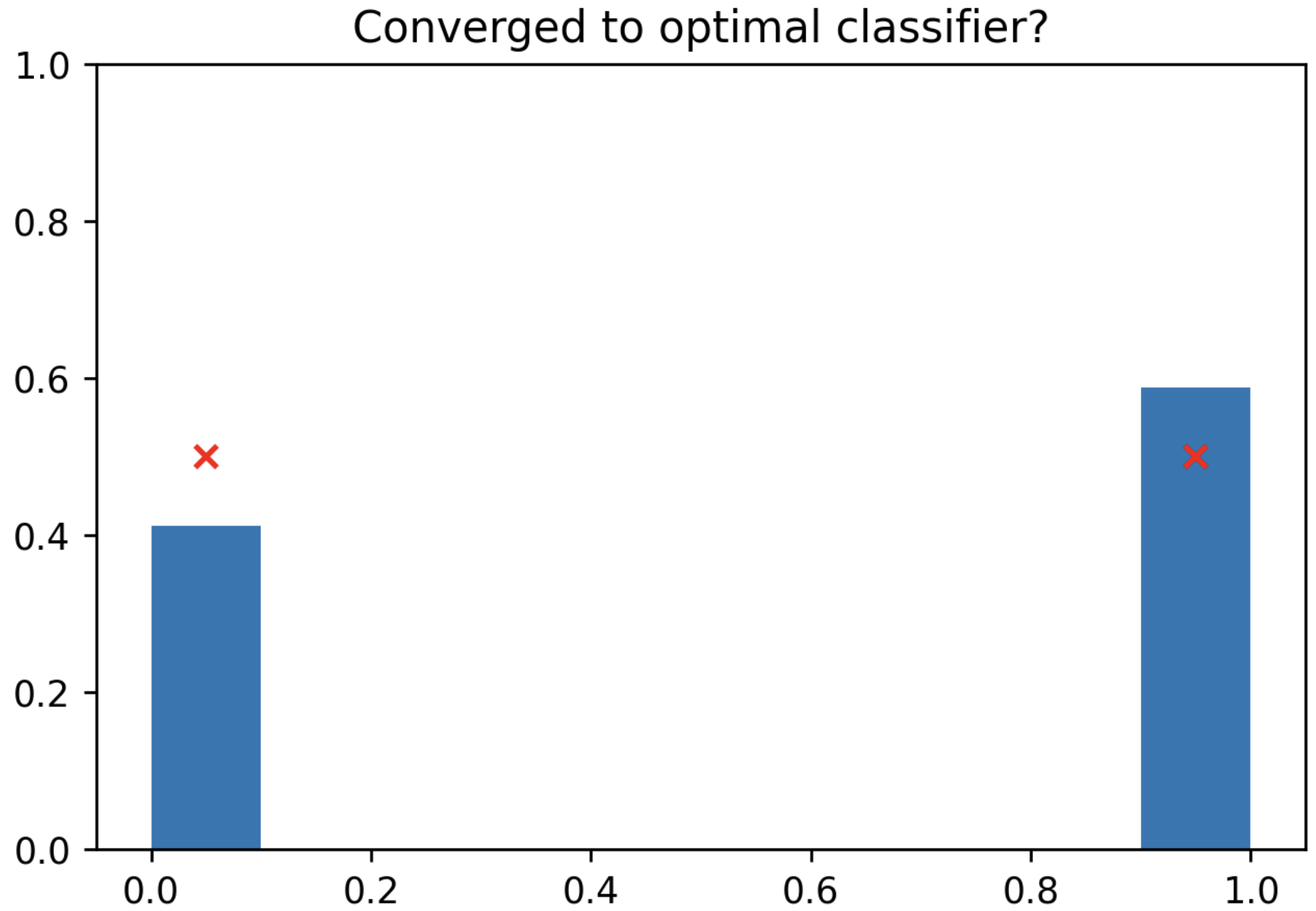}}
  \subfigure[$\lambda =5$]{       \includegraphics[width=.23\linewidth]{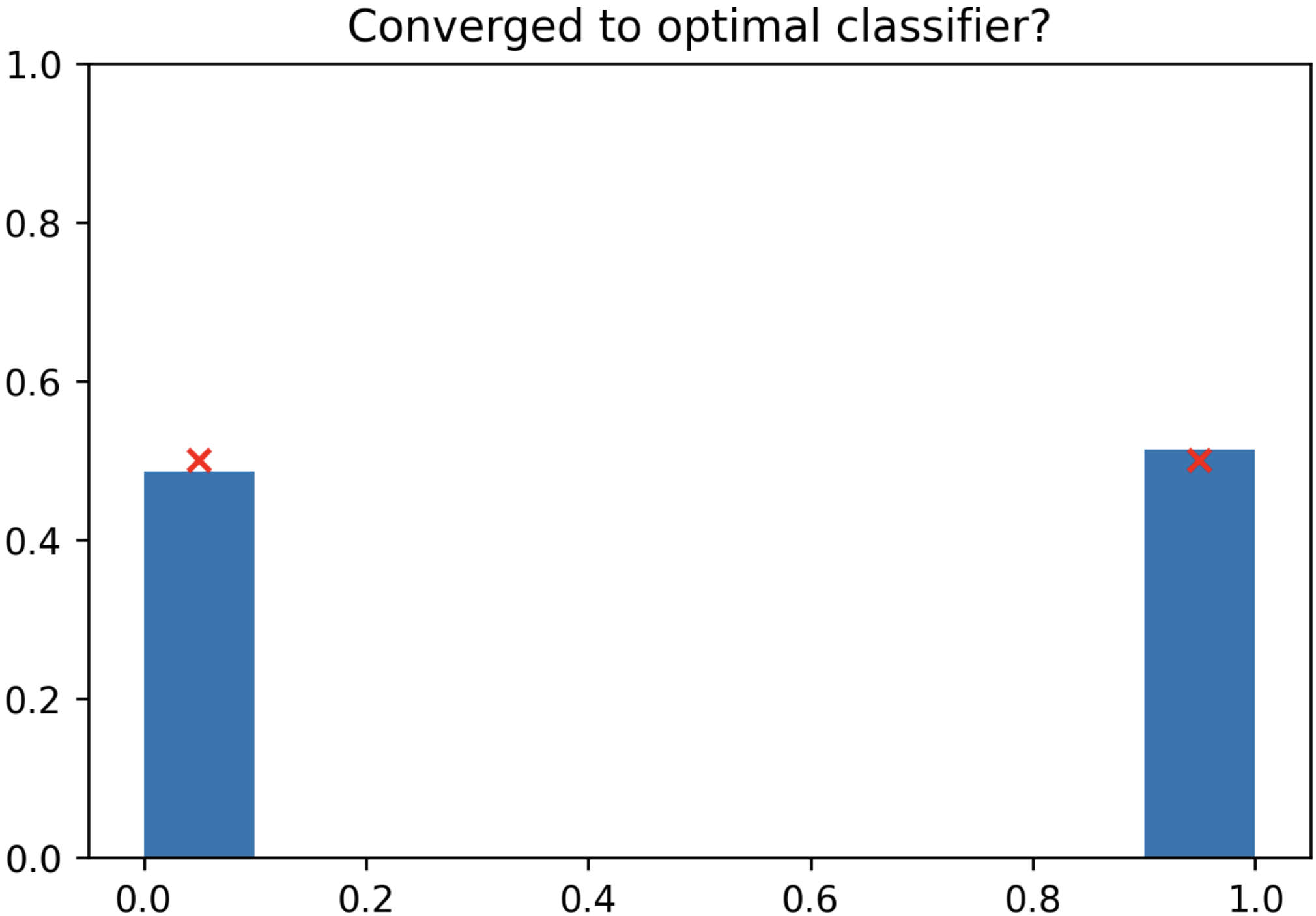}}   
    \subfigure[$\lambda = 10$]{       \includegraphics[width=.23\linewidth]{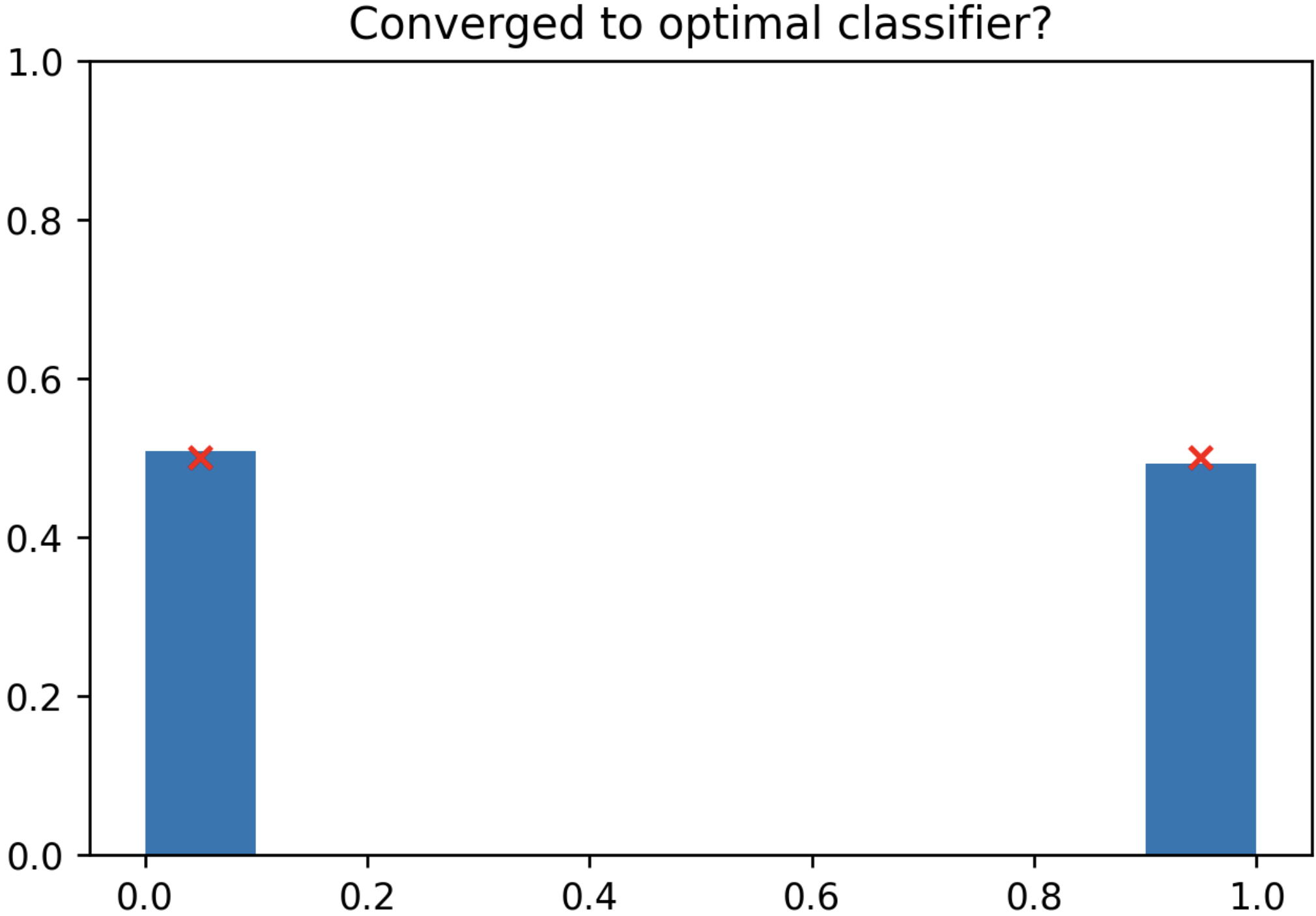}}    
  \subfigure[$\lambda = 100$]{       \includegraphics[width=.23\linewidth]{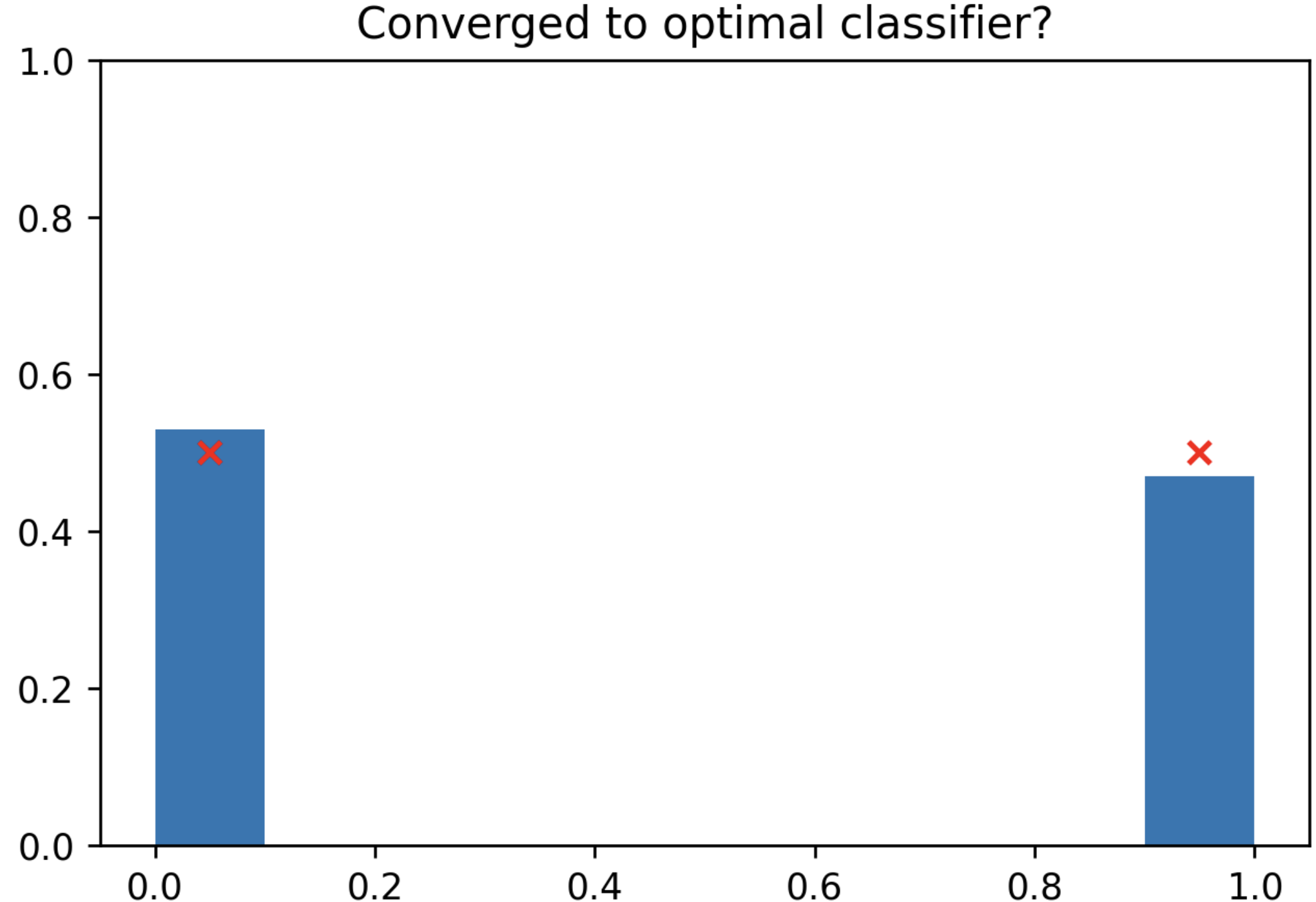}}    
  \caption{GMM in dimension $N=500$ with $\alpha = 0.1$. Depicted is the fraction of endpoints (SGD after $100N$ steps from a random Gaussian initialization) with $m_1 m_2<0$, corresponding to the stable fixed points of the $\lambda = \infty$ dynamics; it matches the predicted $\frac{1}{2}-\frac{1}{2}$ fraction.}\label{fig:GMM-isgood-various-lambda}
\end{figure}

\subsection{Convergence to spurious solutions}
Let us pause to interpret this result.
The stable fixed points when $\alpha<1/4$ are the optimal classifiers,
whereas the unstable set of fixed points given by item (2) misclassify
half of the data. Therefore, the above indicates that
when solving the above task with randomly initialized weights,
one of the following two scenarios occur, each with probability $1/2$
(with respect to the initialization):
the algorithm will converge to the optimal classifier
in linear time or it will appear to have converged to a macroscopically sub-optimal classifier on the same timescale, 
see Figures~\ref{fig:GMM-endpoints-various-lambda}--\ref{fig:GMM-isgood-various-lambda} for numerical verification of this at finite $N$ and $\lambda$.

\subsection{Degeneracy of diffusive limits}
It is then natural to ask about the behaviour of the SGD
in the latter regime, after it converges to the sub-optimal classifiers
which lie on the aforementioned quarter-circles. 
Proposition~\ref{prop:bgmm-ballistic-noiseless} rigorously justified the exchange of $n\to\infty$ and $\lambda\to\infty$ limits in the ballistic phase. 
In the diffusive phase, one could in principle find the quarter circle of fixed points of the ODE in Proposition~\ref{prop:bgmm-ballistic} and consider rescaled observables $\tilde v_i, \tilde m_i$ corresponding to blowing up $v_i,m_i$ in diffusive $O(n^{-1/2})$ neighborhoods about them to get SDE limits from Theorem~\ref{thm:main}. In order to have explicit formulae, in what follows, we consider the diffusive limits obtained when taking $\lambda = \infty$, for which we know the precise locations of these fixed points from Proposition~\ref{prop:bgmm-ballistic-noiseless}. This also captures the limit obtained by taking any $\lambda_n$ diverging faster than $O(n^{1/2})$; the numerics of Figure~\ref{fig:bgmm} demonstrate its qualitative consistency with the behavior in microscopic neighborhoods of fixed points at $\lambda$ finite.

\begin{prop}\label{prop:bgmm-diffusive-noiseless}
Let $\delta_n = \sfrac{1}{N}$,  $(a_{1},a_{2})\in\mathbb{R}_{+}^{2}$ be such that $a_{1}^{2}+a_{2}^{2}=C_{\alpha}$ and let $\tilde v_i = \sqrt{N}(v_i -a_i)$ and $\tilde m_i = \sqrt N(m_i - a_i)$.
When $\lambda = \infty$, the SDE system obtained 
by applying Theorem~\ref{thm:main} to $\tilde{\mathbf{u}}_{n}$
is
\begin{align*}
d\tilde{v}_{i}= & \alpha(\tilde{m}_{i}-\tilde{v}_{i})+a_{i}(\alpha-2\alpha^{2})\sum a_{k}(\tilde{v}_{k}+\tilde{m}_{k})+\tilde{\Sigma}^{\sfrac 1 2}d\mathbf{B}_{t}\cdot e_{v_{i}}\,, && d{R}_{ii}^\perp  =-2\alpha R_{ii}^\perp dt\,,\\
d\tilde{m}_{i}= & \alpha(\tilde{v}_{i}-\tilde{m}_{i})+a_{i}(\alpha-2\alpha^{2})\sum a_{k}(\tilde{v}_{k}+\tilde{m}_{k})+\tilde{\Sigma}^{\sfrac 1 2}d\mathbf{B}_{t}\cdot e_{m_i}\,,&&d {R}_{ij}^{\perp}=-2\alpha R_{ij}^{\perp} dt\,,
\end{align*}
where $\tilde{\Sigma}$ is a matrix whose only non-zero entries are 
$\tilde{\Sigma}_{\tilde v_{i}\tilde v_{j}}=\tilde{\Sigma}_{\tilde m_{i}\tilde m_{j}}=\tilde{\Sigma}_{\tilde v_{i}\tilde m_{j}}=  \alpha^{2}a_{i}a_{j}\,.$
\end{prop}
Notice that this diffusion matrix is rank 1, so this diffusion is non-trivial but degenerate even in the rescaled coordinates $(\tilde v_i,\tilde m_i)$.
Moreover, the entries of $\tilde \Sigma$ vanish on the axes $a_1=0$ or $a_2 =0$. In particular, crossing from the unstable quarter ring into the quadrants $v_1 v_2 <0$ where the stable fixed points lie is \emph{impossible} in the noiseless setting, and happens on a much larger timescale at finite $\lambda$.

 \begin{figure}
    \centering
    \subfigure[]{       \includegraphics[width=.23\linewidth]{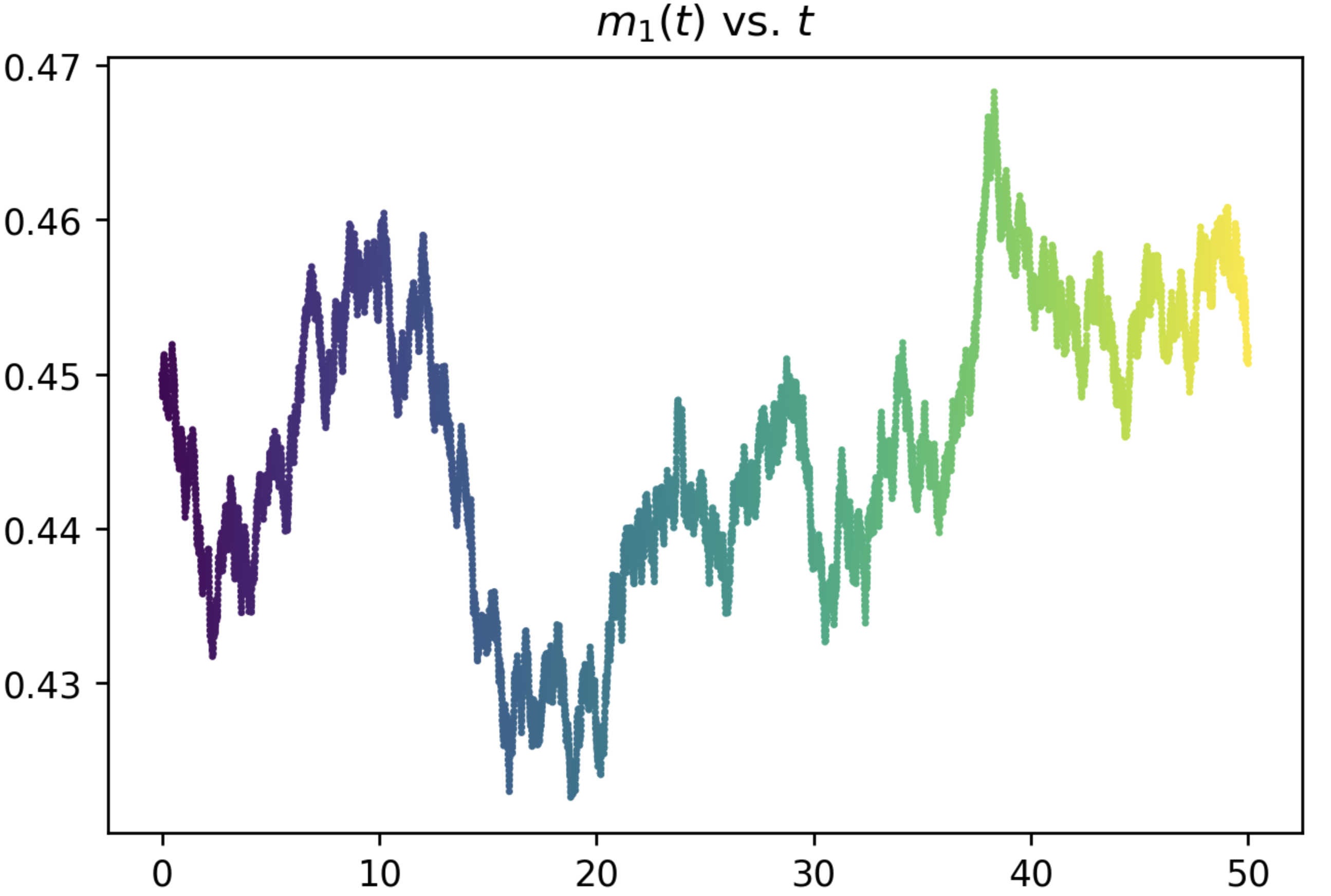}}    
  \subfigure[]{       \includegraphics[width=.23\linewidth]{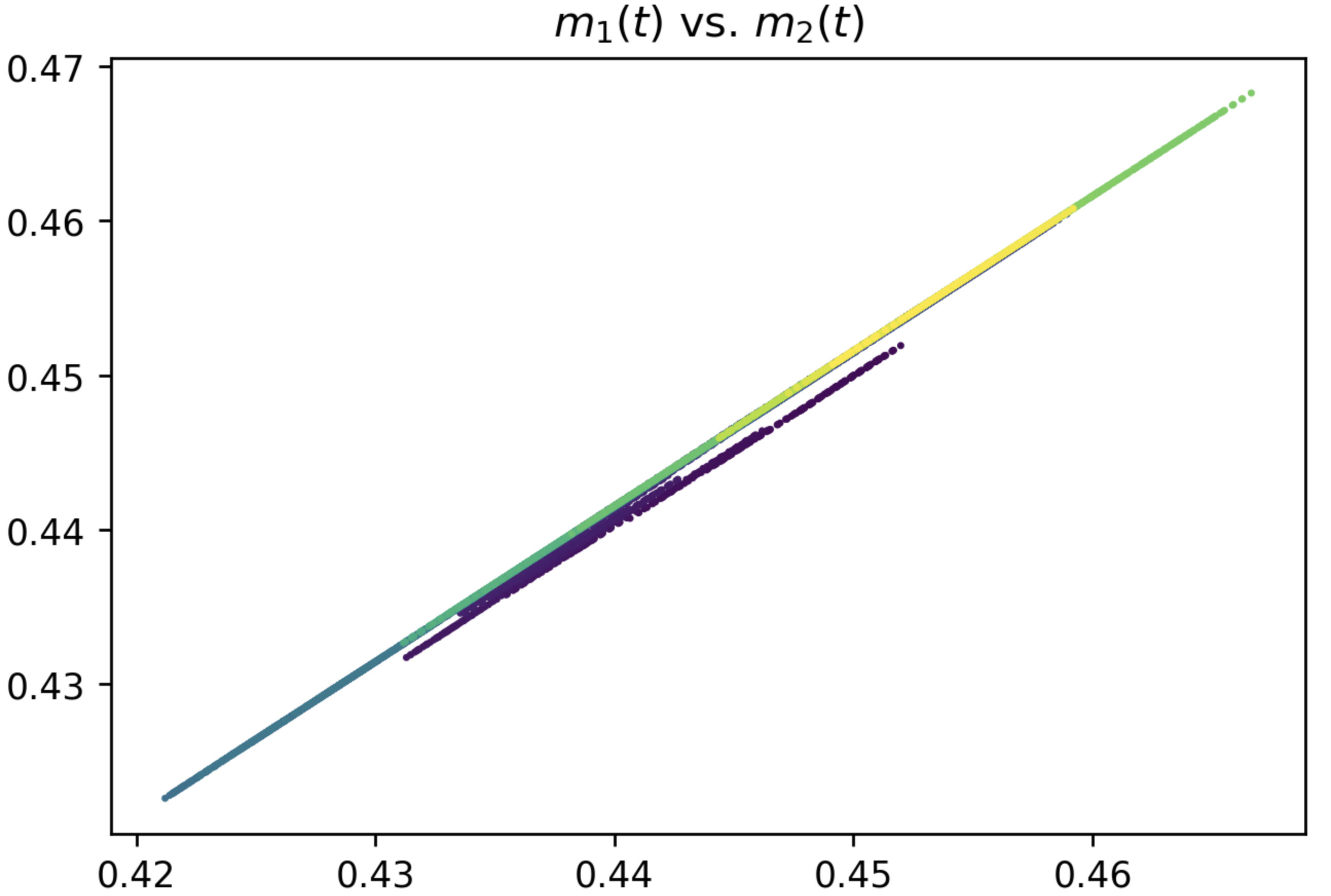}}    
    \subfigure[]{   \includegraphics[width=.23\linewidth]{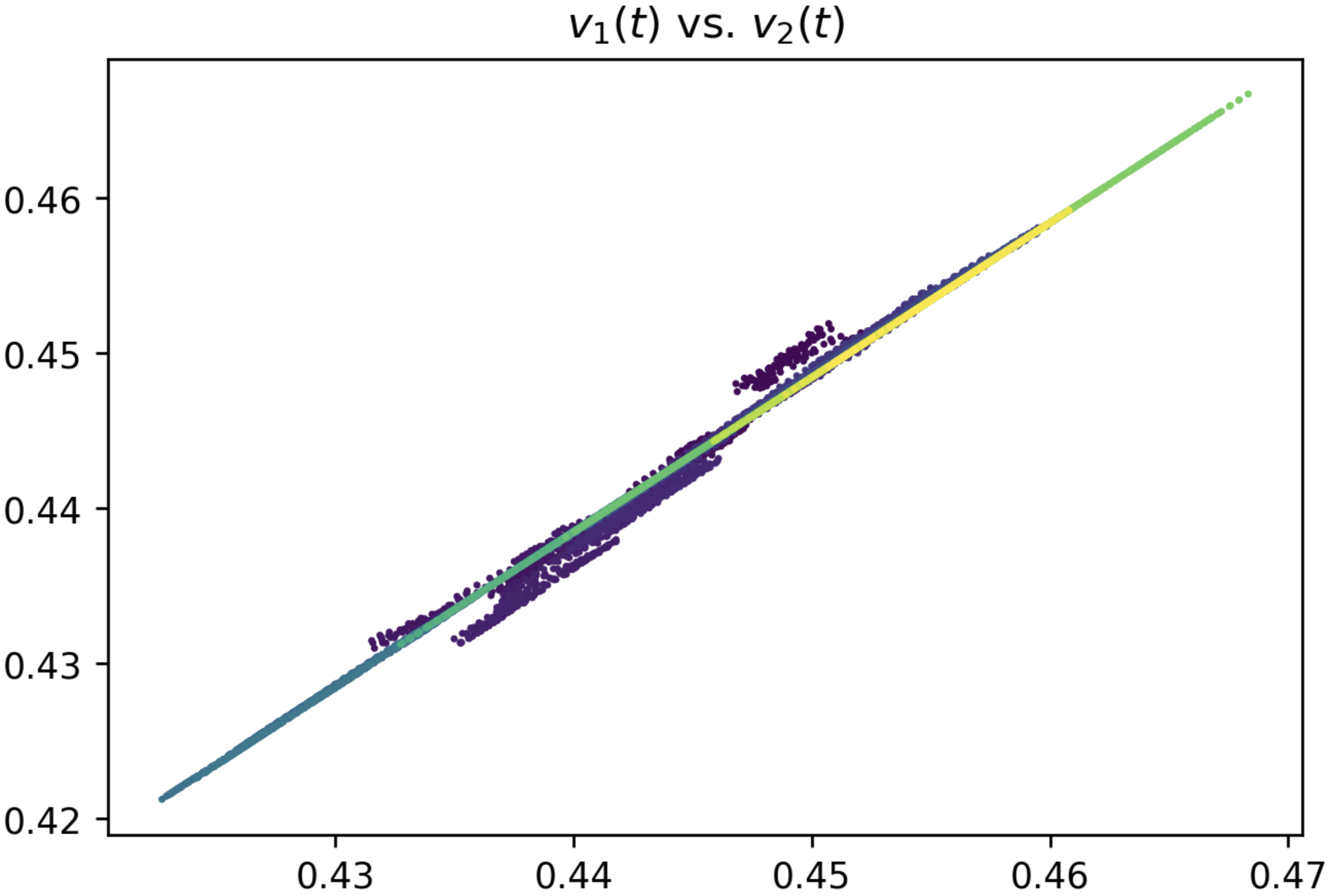}}
  \caption{Binary GMM in dim.\ $N=250$ with $\lambda = 100$ and $\alpha = 0.1$. Diffusive limits for (a) $m_1$ individually, and (b)--(c) the pairs $(m_1,m_2)$ and $(v_1,v_2)$ where the diffusions can be seen to not be of full rank.}\label{fig:bgmm}
\end{figure}

\section{Two-layer networks for the XOR Gaussian mixture}\label{sec:XOR-results}

\subsection{Model and Background}

For our final example,  consider the problem of supervised learning for an XOR-type Gaussian mixture model in $\mathbb R^N$.
Suppose that we are given i.i.d.\ samples of the form $Y=(y,X)$,
where $y$ is $Ber(1/2)$ and
$X$ has the following distribution: if $y=1$ then $X$ is a $\sfrac12$-$\sfrac12$ mixture of $\cN(\mu,I/{\lambda})$ and $\cN(-\mu,I/{\lambda})$ and if $y = 0$ it is a $\sfrac12$-$\sfrac12$ mixture of $\cN(\nu,I/{\lambda})$ and $\cN(-\nu,I/{\lambda})$, 
where $\lambda>0$, and $\mu,\nu$ are orthogonal unit
vectors. Here, $y$ is the class label and $X$ is the data.

This data model is a Gaussian mixture model analogue of the (in)famous XOR problem of Minsky--Papert \cite{minsky2017perceptrons}. 
In particular, it is easy to see that the optimal decision boundary is not expressible by a single-layer neural network
as the data is not linearly separable. That said, it is also straightforward to see that this decision boundary
is realizable by simple two-layer networks.\footnote{In the notation of the following subsection,
this can be realized by taking $K=4$, $W_1 = - W_2 = \mu$, $W_3 = W_4 = \nu$, and  $v_i=c$ for  $i=1,\ldots,4$ for some $c>0$.}

We focus on this example as a demonstration of the applicability of our techniques to the analysis of the training dynamics for two-layer neural networks
on natural data models. While this model is arguably the simplest model requiring a multi-layer network to solve, it nevertheless exhibits very complex
phenomenology. We mention that some of these complexities were also observed in a very similar setup in~\cite{refinetti2021classifying} where ballistic limits from warm starts were derived.

\subsection{Analysis}
Consider the corresponding classification problem
 using a two-layer neural network, taking as our estimator of the class label $\hat{y}(X)$
 to be the natural rounding of $\sigma(v\cdot g(WX))$,
where $\sigma$ and $g$ are the sigmoid and ReLU as in Section~\ref{sec:binary-gmm}.
We take $W$ to be a $K\times N$ matrix and $v$ to be a $K$-vector.

To train the network, we again consider the binary cross-entropy loss with $\ell_2$-penalty. This loss is identical to \eqref{eq:Loss-bgmm} \emph{mutatis mutandis}. For the readers convenience, we recall that the loss is of the form
\[
L\big((v_{i},W_{i})_{i\le K};(y,X)\big)=-yv\cdot g(WX)+\log(1+e^{v\cdot g(WX)})+p(v,W)\,,
\]
where again $\sigma,g$ are applied component wise and again $p(v,W):=(\alpha/2)(\norm{v}^{2}+\norm{W}^{2})$.

In Lemma~\ref{lem:XOR-summary-stats} below, we show that the law of the loss at a point $(v,W)$  
depends only on the following $4K + \binom{K}{2}$ variables: for $1\le i\le j\le K$, 
\begin{align}\label{eq:XOR-GMM-summary-stats}
v_{i}\,,\qquad m_{i}^{\mu}  =W_{i}\cdot\mu\,, \qquad m_{i}^{\nu}  =W_{i}\cdot\nu\,, \qquad
R_{ij}^{\perp} =W_{i}^{\perp}\cdot W_{j}^{\perp}
\end{align}
where $W_i^\perp = W_i - m_i^\mu \mu - m_i^\nu \nu$ is the part perpendicular to $\mu,\nu$. 
Furthermore, this lemma shows that, if $\mathbf{u}_{n}$ given by these variables, then for any fixed $\lambda>0$, the localizability criterion of Definition~\ref{defn:localizable} holds as long as $\delta_{n}=O(1/n)$.
We can then apply Theorem~\ref{thm:main} to obtain limits in both the ballistic and diffusive phases. 
To this end,
we need to define the following auxiliary functions analogous to \eqref{eq:bgmm-A-B} above.
For a point $(v,W) \in \mathbb R^{K+KN}$, define the quantity 
$$
	\mathbf{A}_i = \mathbb E\big[X\mathbf 1_{W_i \cdot X\ge 0} \big(-y+\sigma(v\cdot g(WX))\big)\big]\,,
$$
and let 
\begin{align*}
	\mathbf{A}_i^\mu  = \mu \cdot \mathbf{A}_i\,, \qquad \mathbf{A}_i^\nu = \nu\cdot \mathbf{A}_i\,, \qquad \mathbf{A}_{ij}^\perp = W_j^\perp \cdot \mathbf{A}_i\,.
\end{align*}
Furthermore, let 
\begin{align*}
	\mathbf{B}_{ij} = \mathbb E\big[\mathbf 1_{W_i \cdot X \ge 0} \mathbf 1_{W_j\cdot X\ge 0} \big(-y+\sigma(v\cdot g(WX))\big)^2\big]\,.
\end{align*}
By similar reasoning, it can be shown that these functions are expressible as functions of $\bu_n$ alone (see Section~\ref{sec:xor-proofs} below).
We then find the following effective ballistic dynamics.
\begin{prop}\label{prop:XOR-effective-dynamics-finite-lambda}
	Let $\bu_n$ be as in~\eqref{eq:XOR-GMM-summary-stats} and fix any $\lambda>0$ and $\delta_n = c_\delta/N$. Then $\bu_n(t)$ converges to the solution of the ODE system $\dot{\bu}_t = -\mathbf{f}(\bu_t) + \mathbf{g}(\bu_t)$, initialized from $\lim_n (\bu_n)_* \mu_n$ with 
	\begin{align*}
		f_{v_i} &  = m_i^\mu \mathbf A_i^\mu (\mathbf u) + m_i^\nu \mathbf{A}_i^\nu (\bu) + \mathbf{A}_{ii}^\perp (\bu) + \alpha v_i\,, &  f_{m_i^\mu} & = v_i \mathbf{A}_i^\mu + \alpha m_i^\mu\,, \\
		f_{R_{ij}^{\perp}}  & = v_i \mathbf{A}_{ij}^\perp(\bu) +v_j \mathbf{A}_{ji}^\perp(\bu) + 2\alpha R_{ij}^\perp\,, & f_{m_i^\nu} & = v_i \mathbf{A}_i^\nu + \alpha m_i^\nu\,.
	\end{align*}
	and correctors $g_{v_i} = g_{m_i^\mu} = g_{m_i^\nu} = 0$, and $g_{R_{ij}^\perp} = c_\delta \frac{v_i v_j}{\lambda} \mathbf B_{ij}$ for $1\le i\le j\le K$. 
\end{prop}

\begin{figure}
    \centering
  \subfigure[$\lambda=10$]{\includegraphics[width=.24\linewidth]{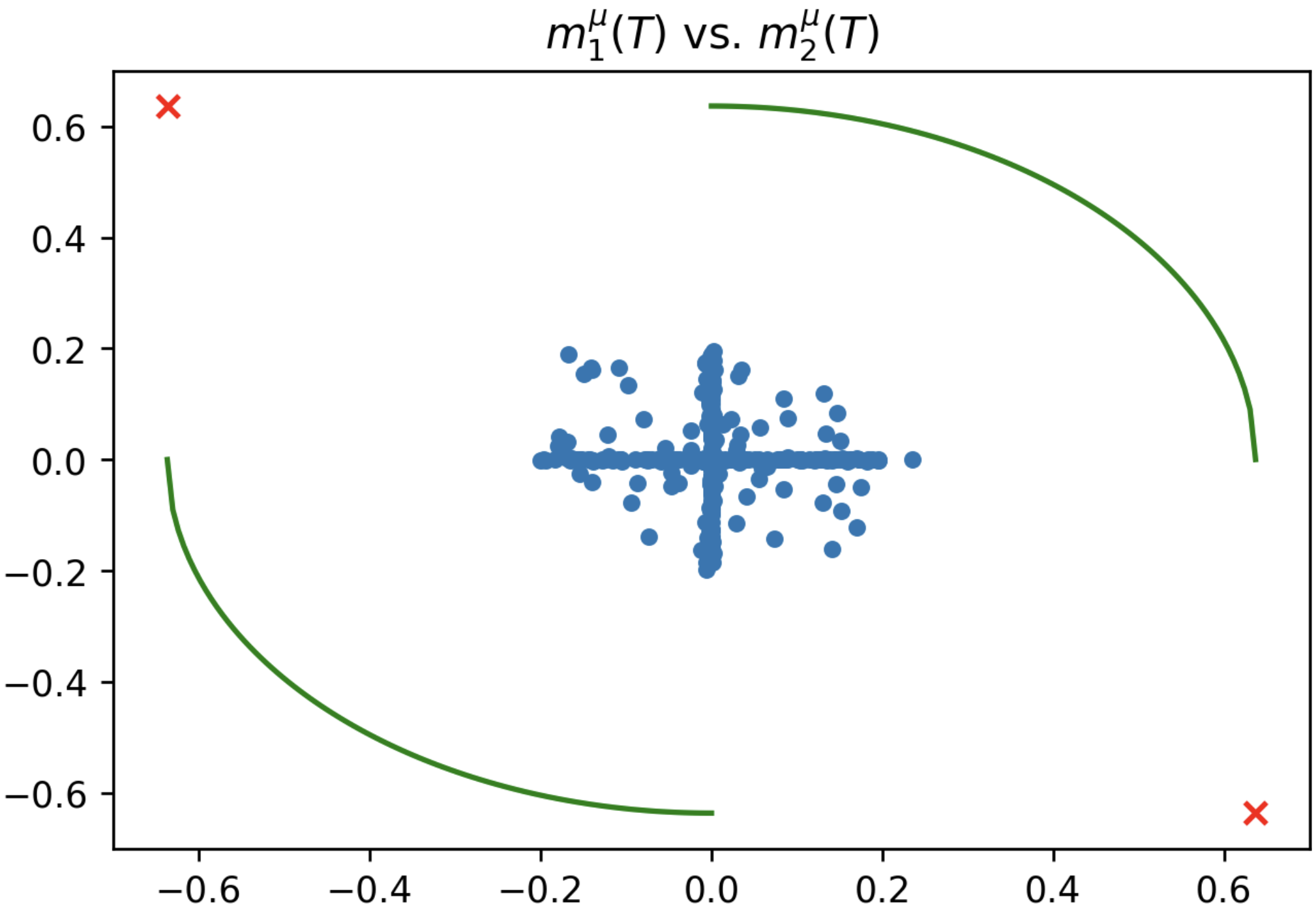}}
  \subfigure[$\lambda =100$]{\includegraphics[width=.24\linewidth]{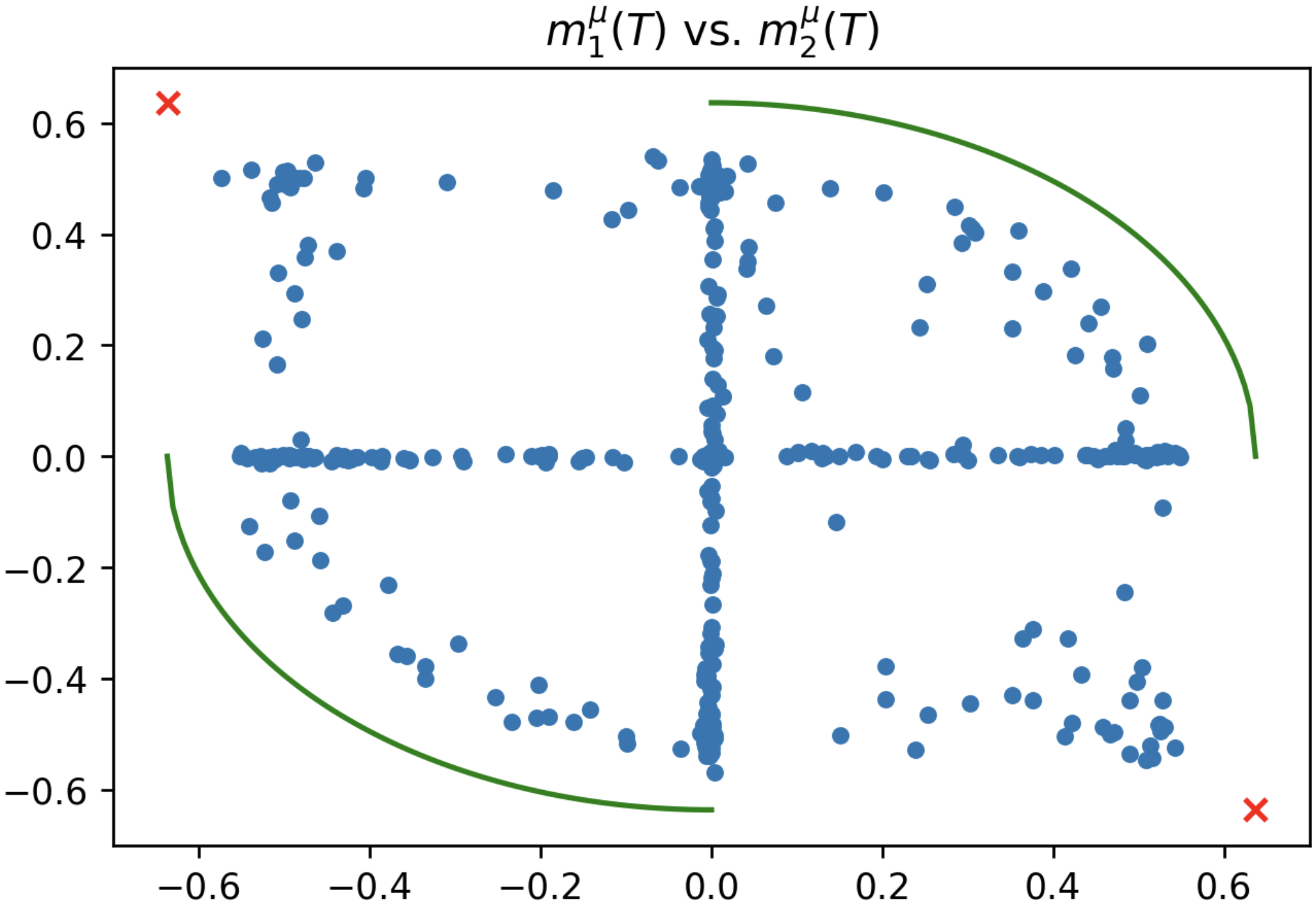}}   
    \subfigure[$\lambda = 500$]{\includegraphics[width=.24\linewidth]{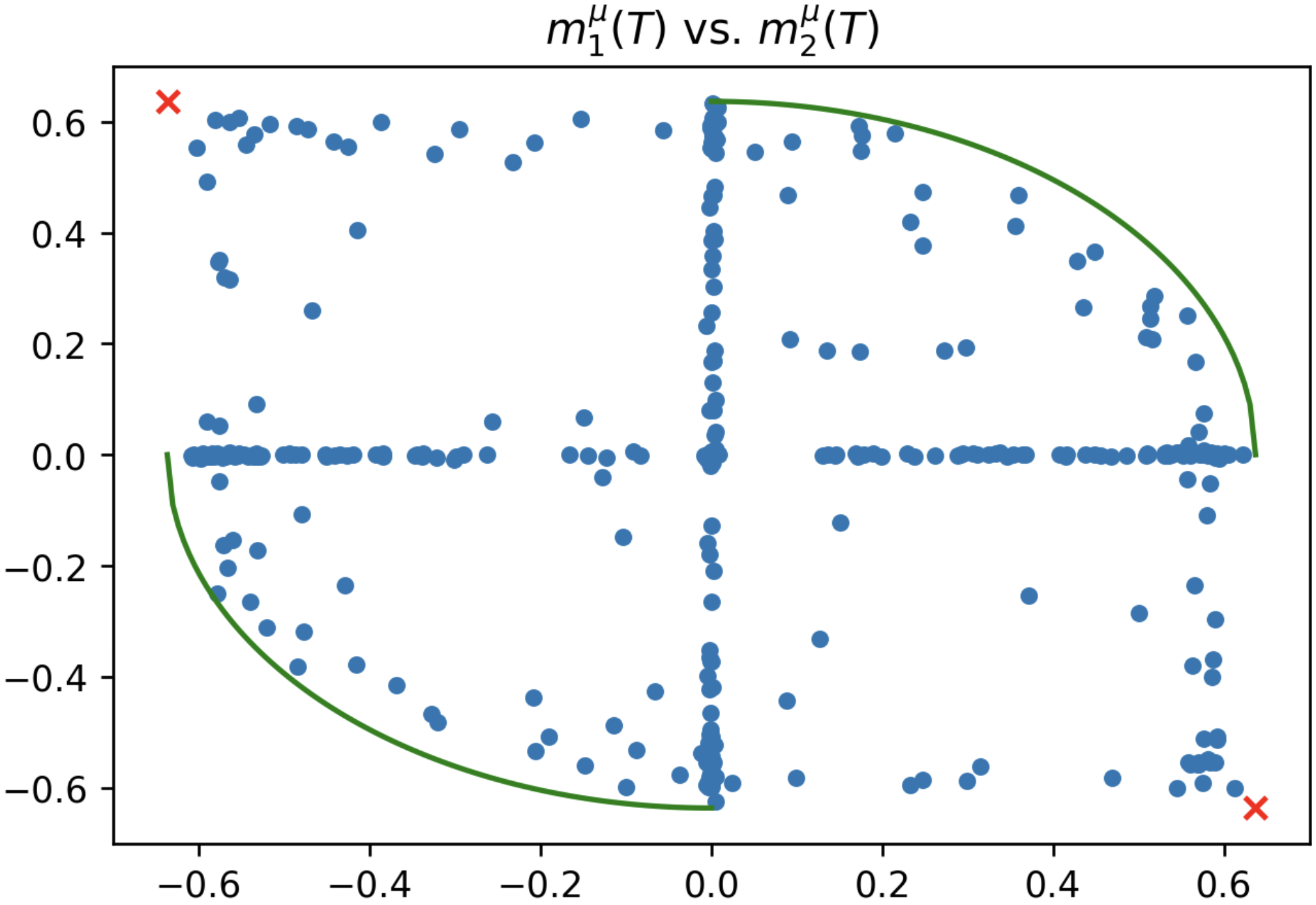}}    
  \subfigure[$\lambda = 1000$]{\includegraphics[width=.24\linewidth]{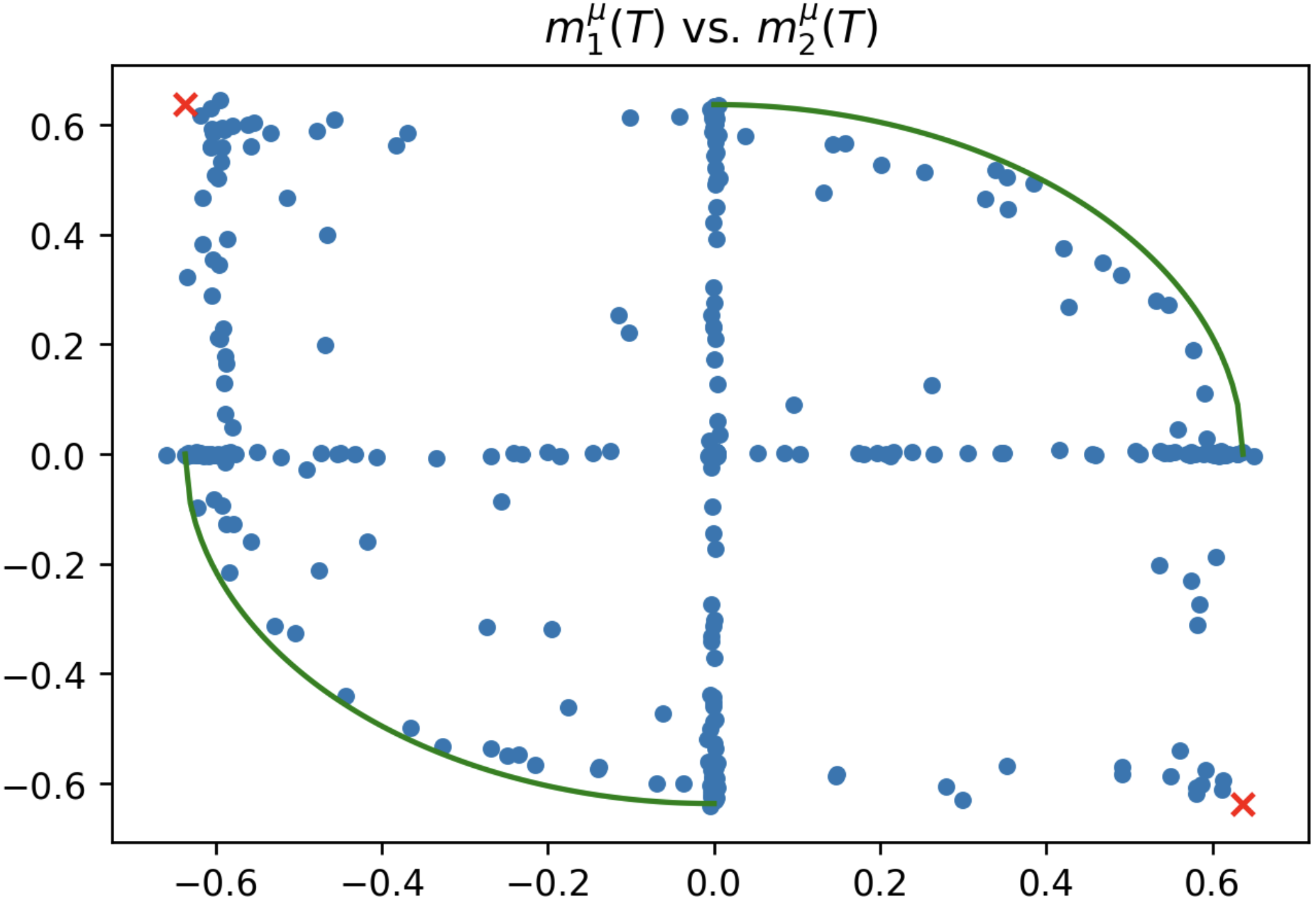}}    
  \caption{XOR in dimension $N=500$ with $\alpha = 0.1$ and $K=4$. Depicted are $(m_1^\mu,m_2^\mu)$ values that 500 runs of SGD converge to after $100N$ steps from a random Gaussian initialization. The ${\color{green!50!black}-}$ and ${\color{red}\times}$ are the unstable and stable fixed points of the $\lambda = \infty$ ballistic effective dynamics. This demonstrates that the fixed points of the limiting effective dynamics have the same qualitative structure at finite $\lambda$ as $\lambda = \infty$, and approach the $\lambda=\infty$ ones as $\lambda$ gets large.}\label{fig:GMM-endpoints-various-lambda.}\label{fig:XOR-endpoints-various-lambda}
\end{figure}

\begin{figure}
    \centering
  \subfigure[$\lambda=10$]{\includegraphics[width=.23\linewidth]{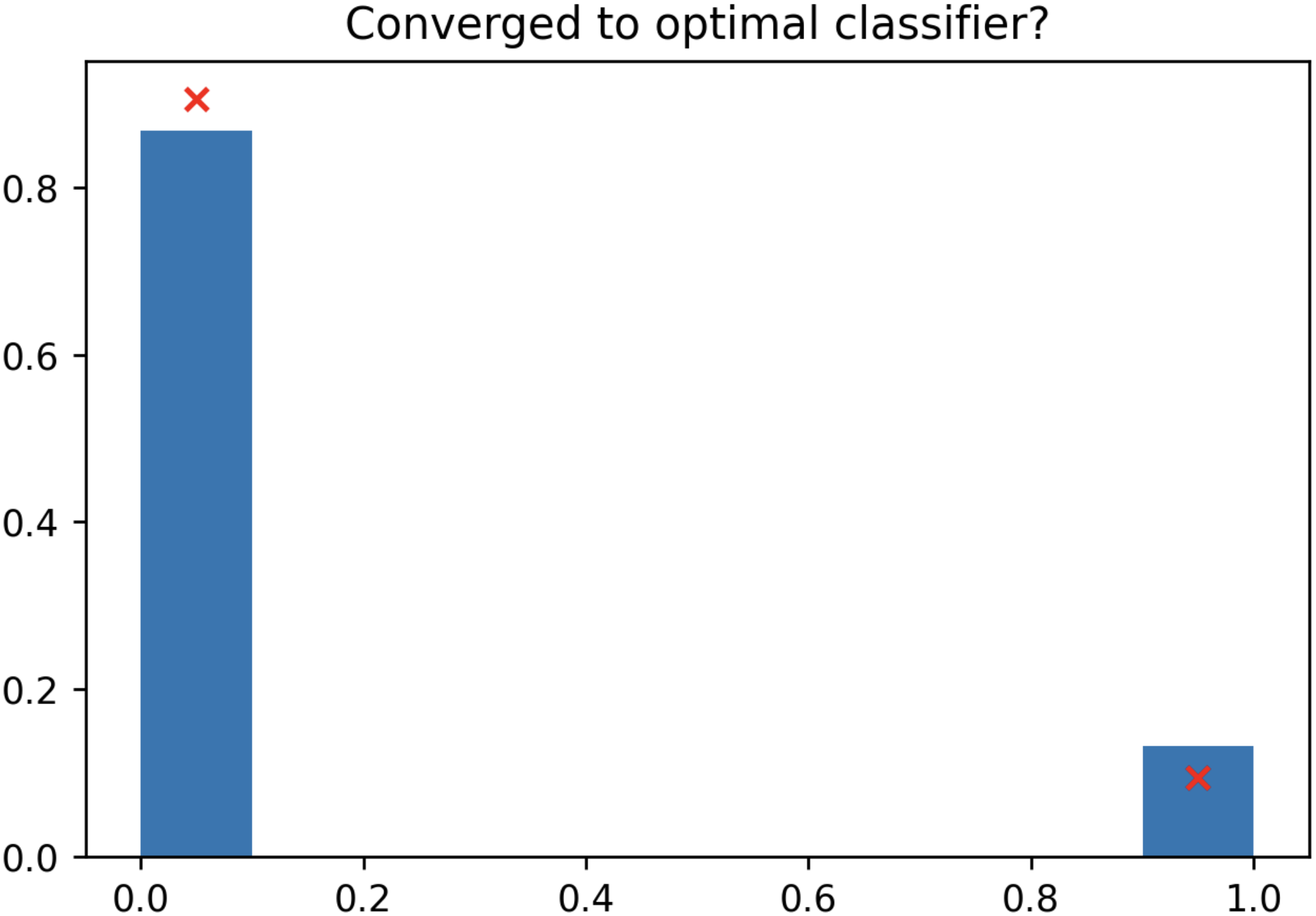}}
  \subfigure[$\lambda =100$]{\includegraphics[width=.23\linewidth]{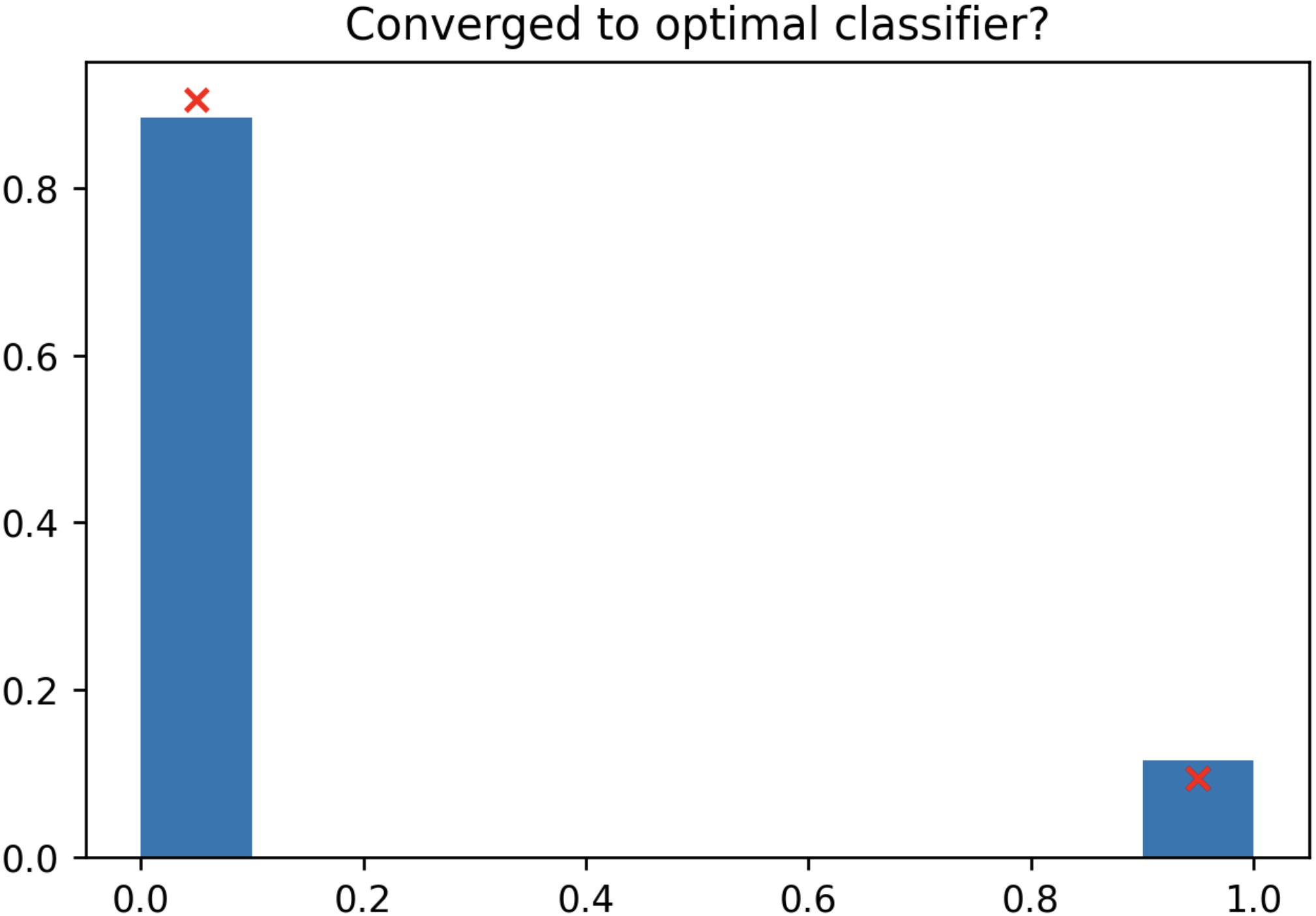}}   
    \subfigure[$\lambda = 500$]{\includegraphics[width=.23\linewidth]{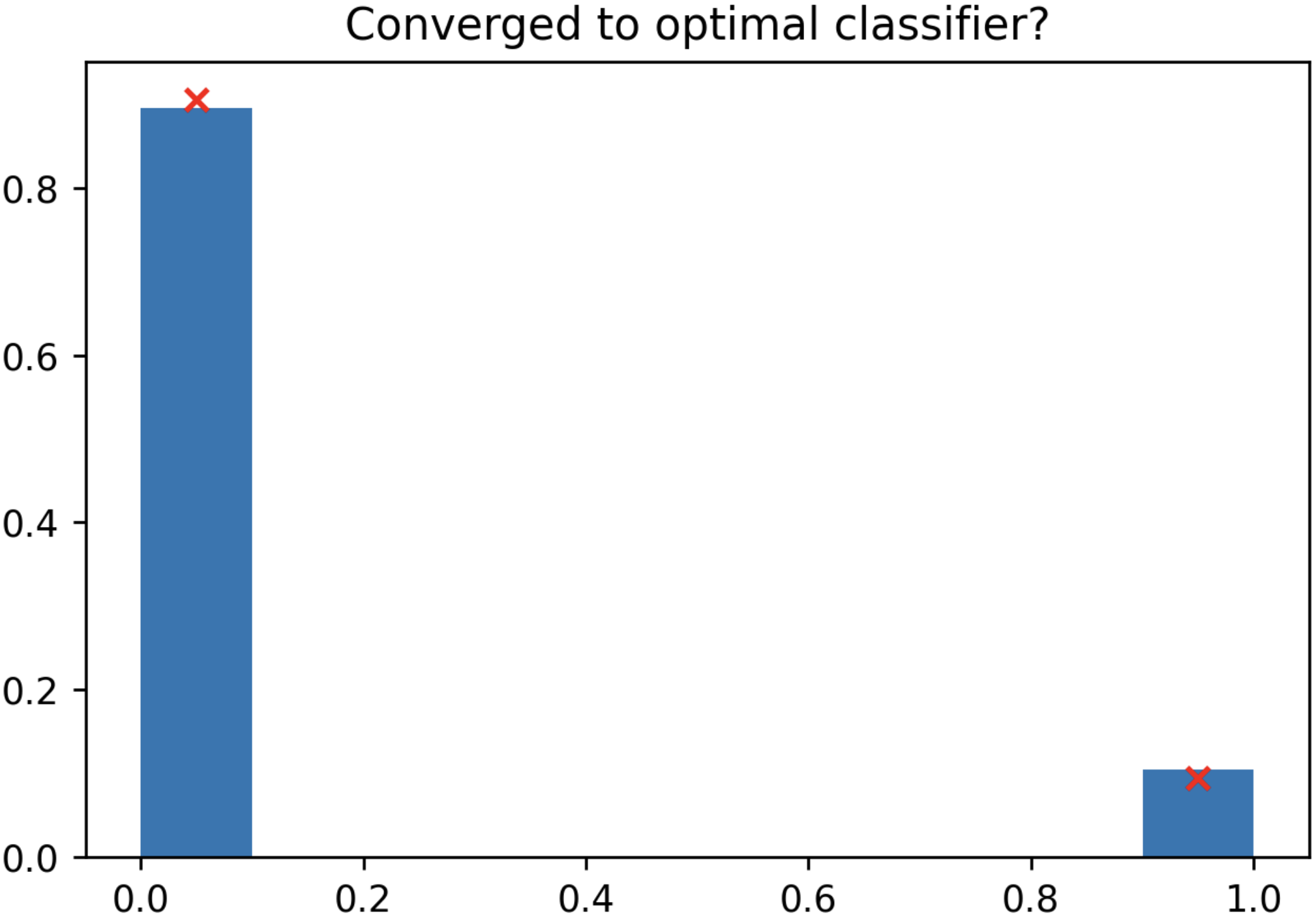}}    
  \subfigure[$\lambda = 1000$]{\includegraphics[width=.23\linewidth]{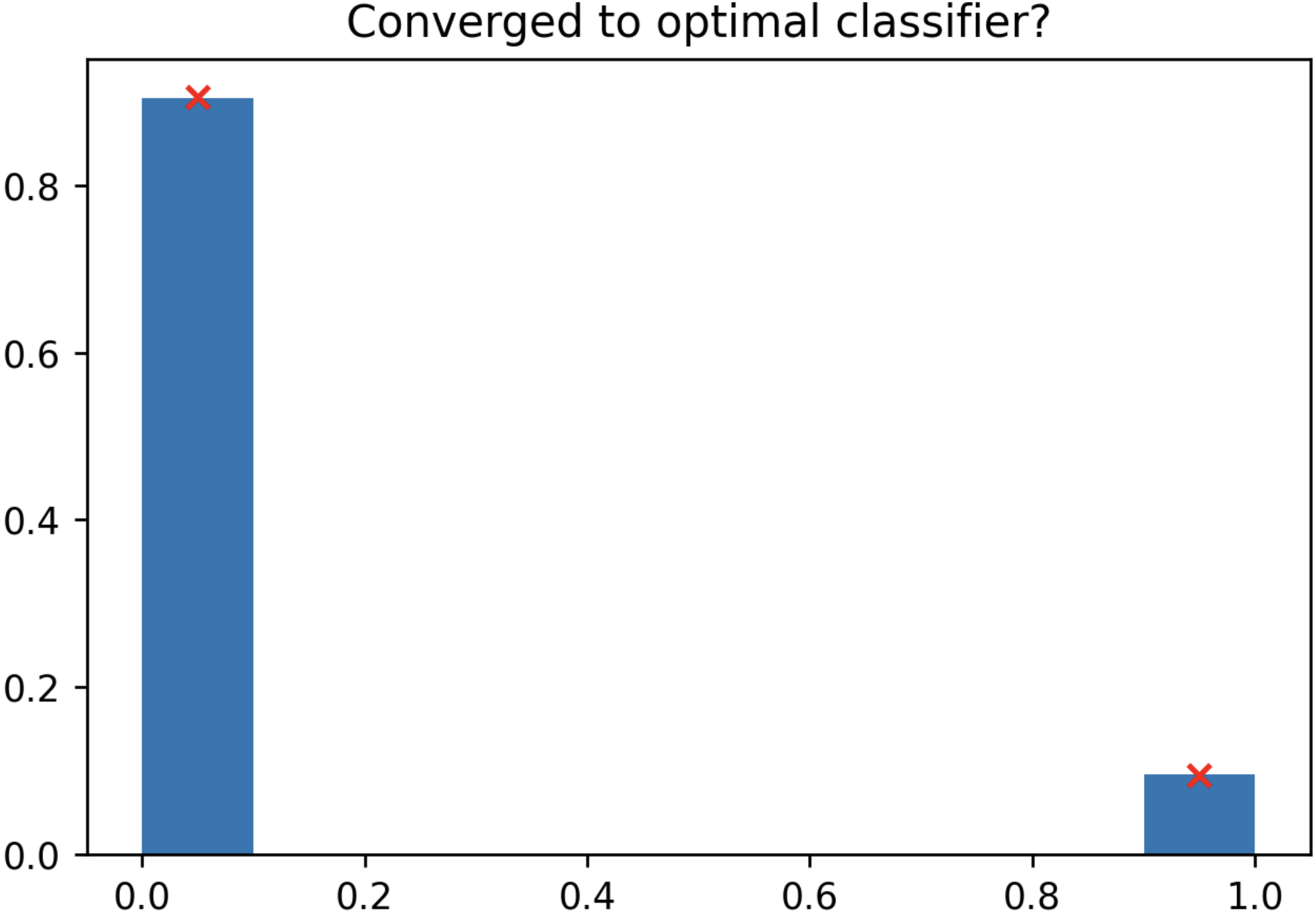}}    
  \caption{XOR in dimension $N=500$ with $\alpha = 0.1$ and $K=4$. The fraction of endpoints (SGD after $100N$ steps from a random Gaussian initialization) with $v$ having two positive entries and two negative entries, and with the consequent  correct signs on $m_i^\mu,m_i^\nu$, corresponding to the stable fixed points of the $\lambda = \infty$ dynamics; it matches the predicted $\frac{29}{32},\frac{3}{32}$ fractions.}\label{fig:XOR-isgood-various-lambda}
\end{figure}

\subsection{Low variance asymptotics}
As with the binary GMM, one can develop the large $\lambda$ limit of these asymptotics after $n\to\infty$.
The effective dynamics in this regime are noticeably more tractable. We defer the precise expressions of these
dynamics to Proposition~\ref{prop:xor-ballistic-ode} below. Let us instead classify the corresponding fixed points.

\begin{prop}\label{prop:fixpoints-xor}
The fixed points of the ODE system of Proposition~\ref{prop:xor-ballistic-ode} are classified as follows. If $\alpha>1/8$, then
the only fixed point is at $\mathbf{u}_{n}=\boldsymbol{0}$. 

If $0<\alpha<1/8$, then let $(I_{0},I_{\mu}^{+},I_{\mu}^{-},I_{\nu}^{+},I_{\nu}^{-})$
be any disjoint (possibly empty) subsets whose union is $\{1,...,K\}$.
Corresponding to that tuple $(I_{0},I_{\mu}^{+},I_{\mu}^{-},I_{\nu}^{+},I_{\nu}^{-})$,
is a set of fixed points that have $R_{ij}^{\perp}=0$ for all
$i,j$, and have 
\begin{enumerate}
\item $m_{i}^{\mu}=m_{i}^{\nu}=v_{i}=0$ for $i\in I_{0}$,
\item $m_{i}^{\mu}=v_{i}>0$ such that $\sum_{i\in I_{\mu}^{+}}v_{i}^{2}=\mbox{logit}(-4\alpha)$
and $m_{i}^{\nu}=0$ for all $i\in I_{\mu}^{+}$,
\item $-m_{i}^{\mu}=v_{i}>0$ such that $\sum_{i\in I_{\mu}^{-}}v_{i}^{2}=\mbox{logit}(-4\alpha)$
and $m_{i}^{\nu}=0$ for all $i\in I_{\mu}^{-}$,
\item $m_{i}^{\nu}=v_{i}<0$ such that $\sum_{i\in I_{\nu}^{+}}v_{i}^{2}=\mbox{logit}(-4\alpha)$
and $m_{i}^{\mu}=0$ for all $i\in I_{\nu}^{+}$,
\item $-m_{i}^{\nu}=v_{i}<0$ such that $\sum_{i\in I_{\nu}^{-}}v_{i}^{2}=\mbox{logit}(-4\alpha)$
and $m_{i}^{\mu}=0$ for all $i\in I_{\nu}^{-}$.
\end{enumerate}
In the $K=4$ case, these form $39$ connected sets of fixed points, and of which $4!=24$ are fixed points that are stable, corresponding to
the possible permutations in which each of $I_{\mu}^{+},I_{\mu}^{-},I_{\nu}^{+},I_{\nu}^{-}$
are singletons. 
\end{prop}
Similar to the binary GMM, in Figures~\ref{fig:XOR-endpoints-various-lambda}--\ref{fig:XOR-isgood-various-lambda}, we demonstrate numerically that the following predicted fixed points from the $\lambda\to\infty$ limit match those arising at finite large $n$ and $\lambda>0$.

{In the $K=4$ case, we can also exactly calculate the probability that the effective dynamics in the ballistic phase 
converges to a stable fixed point (as opposed to an unstable one). 
From a Gaussian initialization $\mu_n$ where $v_i \sim \cN(0,1)$ and $W_i \sim \cN(0,I_N/N)$ independently, 
this converges to~$\sfrac{3}{32}$. We refer the reader to Section~\ref{subsec:K=4-success-prob} for the proof. 

\subsection{Overparametrization in the XOR GMM}\label{subsec:overparametrization-xor}
Since the the derivations of the ballistic limiting equations apply for general $K$, we can also study the probability of ballistic convergence to a stable vs.\ unstable fixed point as one varies $K$. This addresses the regime of overparametrization for the XOR GMM since $K=4$ suffices to express a Bayes--optimal classifier. In this more generic setting, the probability of being in the ballistic domain of attraction of the stable fixed points (corresponding to the Bayes optimal classifiers) is
\begin{align}\label{eq:success-probability-overparametrization}
	\frac{1}{2^K}\sum_{k=2}^{K-2} \binom{K}{k} ( 1- 2^{1-k}) (1-2^{1+k - K})\,,
\end{align}
which goes to $1$ exponentially fast as $K$ grows. This clearly demonstrates the benefits of overparametrizaiton of the landscape in a concrete two-layer network: a random initialization is more likely to to contain the ``right" initial signature (corresponding to none of $I_\mu^+,I_\mu^-,I_\nu^+,I_\nu^-$ being empty at initialization) in order to be in the basin of a Bayes optimal classifier as the width grows, and as long as the right signature is present in the nodes at initialization, the SGD will ballistically converge to a global minimizer of the population loss. This is a rigorous example of the well-known lottery ticket hypothesis of~\cite{LotteryTicket}. Roughly speaking, the lottery ticket hypothesis proposes that the reason for the success of overparametrized networks is that they give more attempts for a sufficiently expressive subnetwork to be initialized well, and succeed at the task on its own.

\subsection{Diffusive limits at unstable fixed points}
As an example of the diffusions that can arise in the rescaled effective dynamics
at the unstable fixed points, let us consider the unstable fixed points in which $v$ has the correct
signature (two positive, two negative) but for each of those we are at a 
corresponding quarter-ring. By way of example, we can set $K=4$, or equivalently focus on a fixed point where all indices beyond the first four have $v_i = m_i^\mu = m_i^\nu =0$. Here, the dynamics effectively becomes
a pair of 2 two-layer GMM's on quarter-rings (as in Section~\ref{sec:binary-gmm}), that are anti-correlated. 
More precisely, let $(a_{1,\mu},a_{2,\mu})$ be such that $a_{1,\mu}^{2}+a_{2,\mu}^{2}=C_{\alpha}$
and $(a_{3,\nu},a_{4,\nu})$ such that $a_{3,\nu}^{2}+a_{4,\nu}^{2}=C_{\alpha}$, for $C_{\alpha}=- \mbox{logit}(4\alpha)$.
Take as fixed points about which we expand
to be $v_{i}=m_{i}^{\mu}=a_{i,\mu}>0$ 
and $v_{i}=m_{i}^{\nu}=a_{i,\nu}<0$ for $i=3,4$.
Namely, we let 
\begin{align*}
\tilde{v}_{i}= & \begin{cases}
\sqrt{N}(v_{i}-a_{i,\mu}) & i=1,2\\
\sqrt{N}(v_{i}-a_{i,\nu}) & i=3,4
\end{cases}\,, \qquad
 \begin{cases}
\tilde{m}_{i}^\mu = \sqrt{N}(m_{i}^\mu-a_{i,\mu}) & i=1,2\\
\tilde{m}_{i}^\nu = \sqrt{N}(m_{i}^\nu-a_{i,\nu}) & i=3,4
\end{cases}\,.
\end{align*}
(Set  $\tilde m_{i}^{\nu}=0$ for $i=1,2$ and $\tilde m_i^\mu = 0$ for $i=3,4$ in $\tilde \bu_n$ effectively removing those variables.)

\begin{figure}
    \centering
    \subfigure[$\lambda = 100$]{       \includegraphics[width=.23\linewidth]{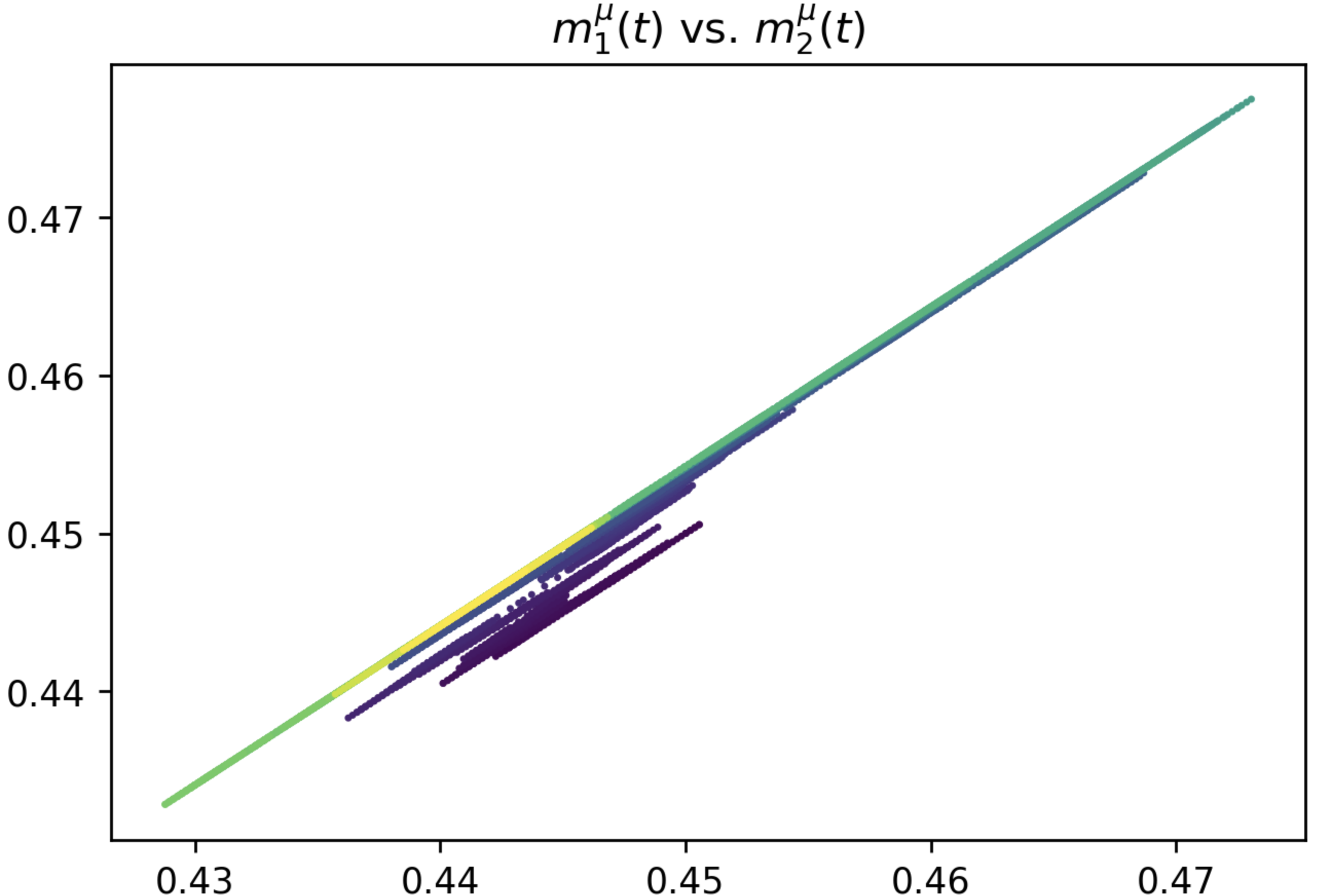}}    
  \subfigure[$\lambda = 100$]{       \includegraphics[width=.23\linewidth]{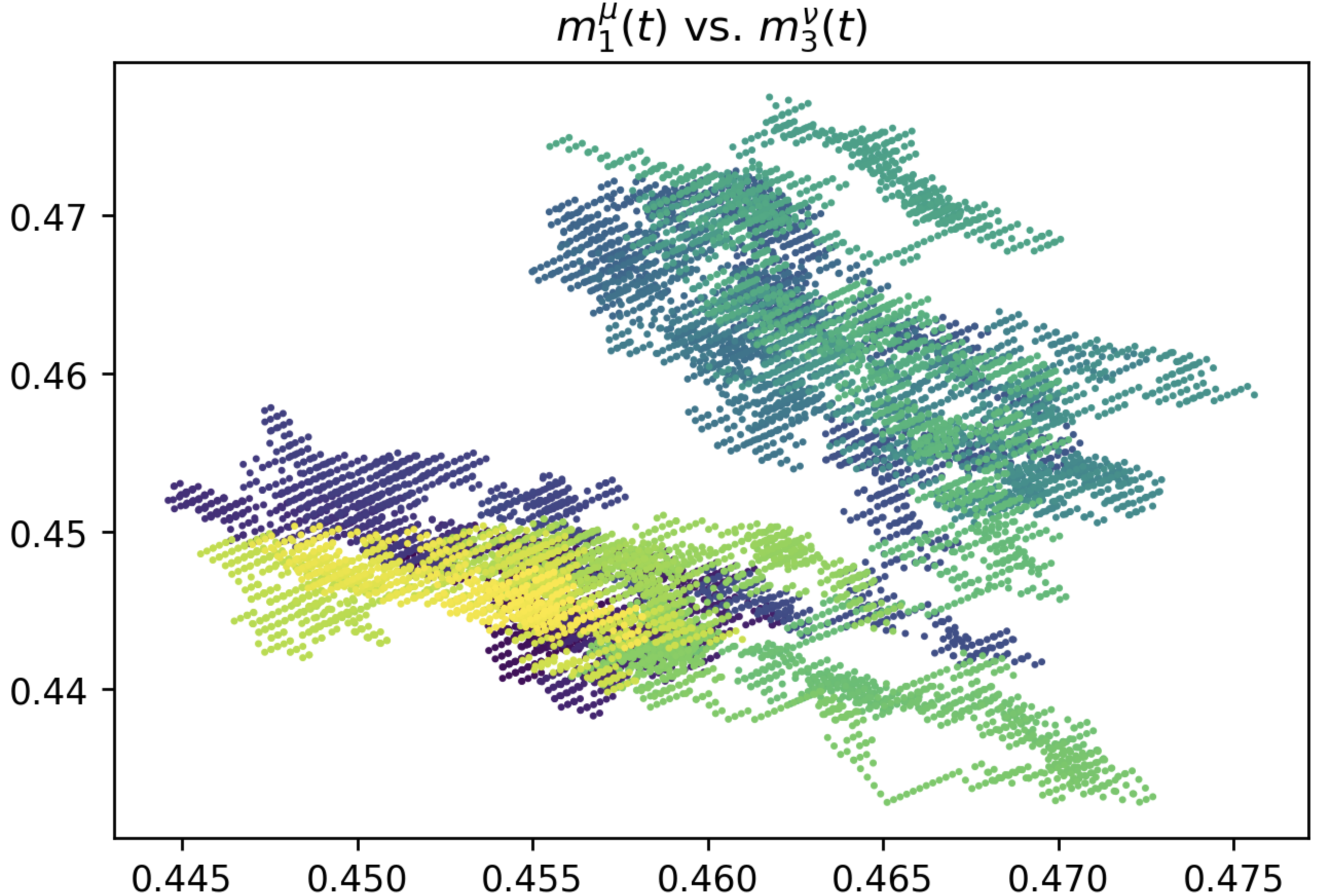}}    
    \subfigure[$\lambda = 1000$]{       \includegraphics[width=.23\linewidth]{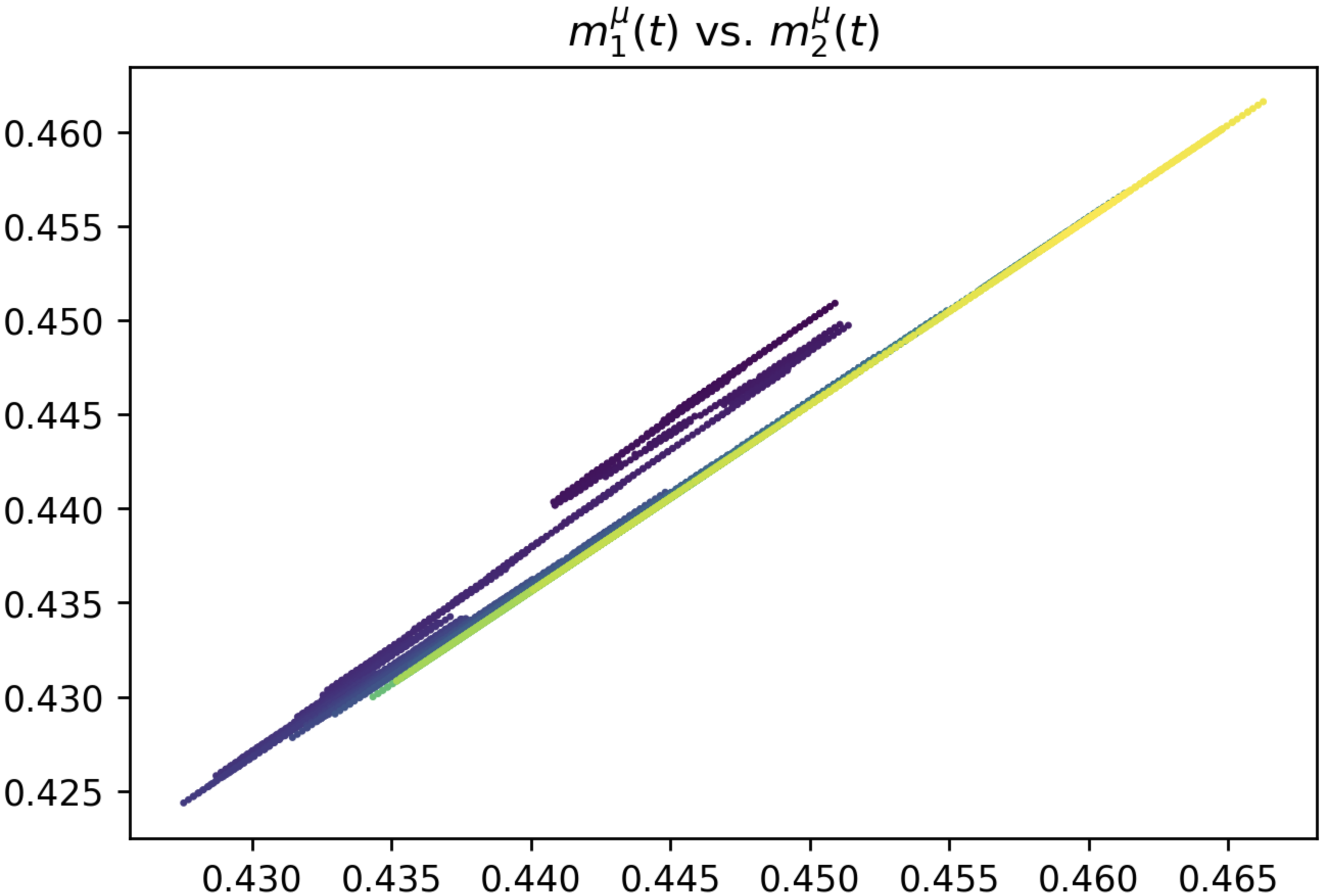}}    
  \subfigure[$\lambda = 1000$]{       \includegraphics[width=.23\linewidth]{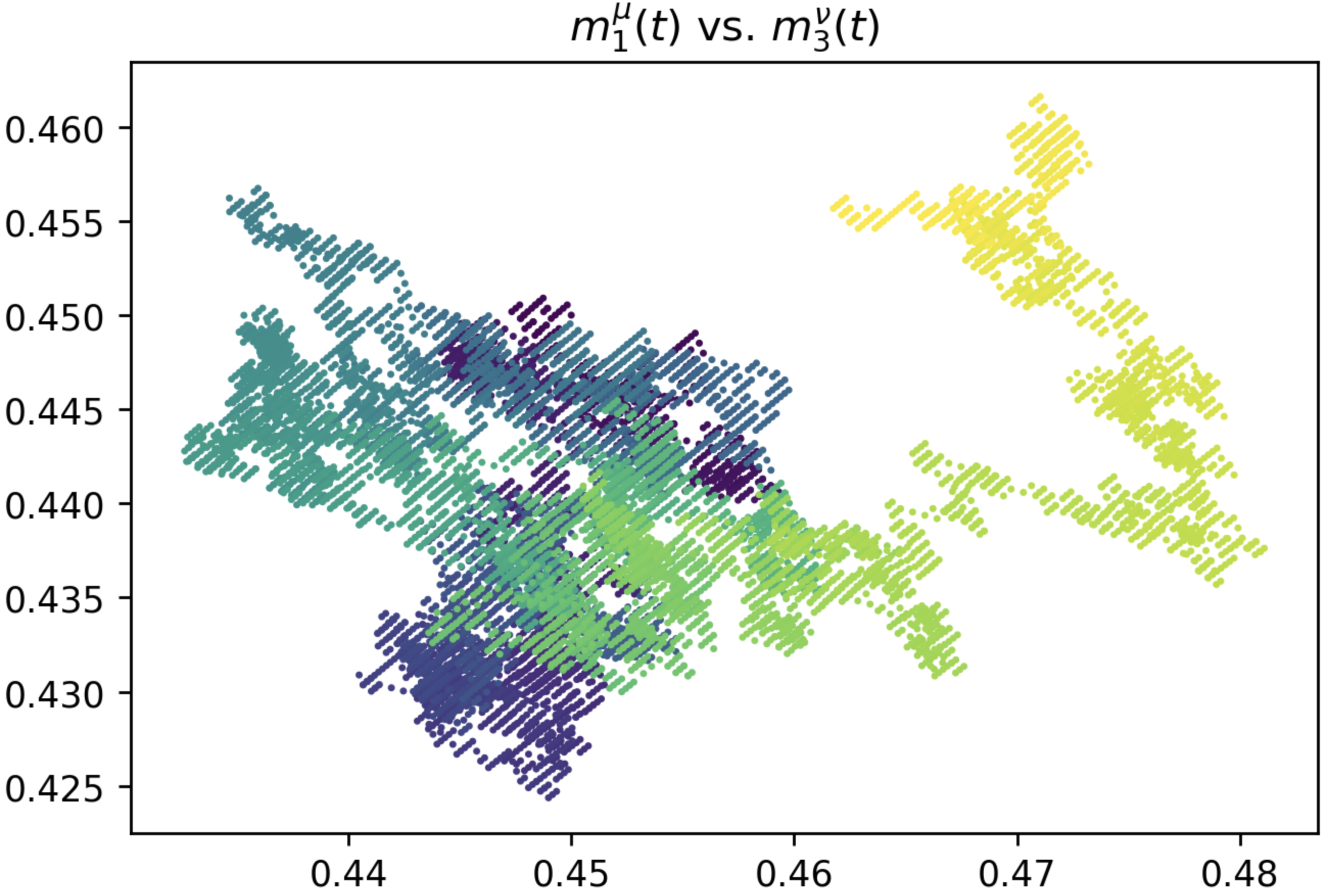}}    
  \caption{XOR GMM in dim.\ $N = 250$ and $\alpha = 0.1$ and $K=4$. (a) and (c) display the degenerate diffusive limits in the regime of Proposition~\ref{prop:XOR-diffusive-noiseless} in $(m_1^\mu,m_2^\mu)$ coordinates at $\lambda = 100$ and $\lambda =1000$. Conversely, (b) and (d) display the diffusive limits in the regime of Proposition~\ref{prop:XOR-diffusive-noiseless} in $(m_1^\mu,m_3^\nu)$, where the limiting diffusions are independent and are of rank 2.} 
    \label{fig:XOR-diffusions}
\end{figure}

\begin{prop}\label{prop:XOR-diffusive-noiseless}
Let $\delta_n = \sfrac{1}{N}$ and let $\tilde{\mathbf{u}}_{n}=(\tilde{v}_{i},\tilde{m}_{i}^{\mu},\tilde{m}_{i}^{\nu},R_{ij}^{\perp})$.
When $\lambda =\infty$, Theorem~\ref{thm:main} can be applied and $\tilde{\mathbf{u}}_{n}(t)$
converges to the solution of the SDE $d\mathbf{\tilde{u}}(t)= - \tilde{\mathbf{h}}(\tilde{\mathbf{u}})dt+\sqrt{\Sigma(\tilde{\mathbf{u}})}d\mathbf{B}_{t}$
where 
\begin{align*}
\tilde{h}_{\tilde{v}_{i}}= & \begin{cases}
\alpha(\tilde{v}_{i}-\tilde{m}_{i}^\mu) - a_{i,\mu}(\alpha-4\alpha^{2})\sum_{k=1,2}a_{k,\mu}(\tilde{v}_{k}+\tilde{m}_{k}^\mu) & i=1,2\\
\alpha(\tilde{v}_{i}-\tilde{m}_{i}^\nu) - a_{i,\nu}(\alpha-4\alpha^{2})\sum_{k=3,4}a_{k,\nu}(\tilde{v}_{k}+\tilde{m}_{k}^\nu) & i=3,4
\end{cases}\,,
\end{align*}
$\tilde h_{\tilde m_i^\mu}$ (resp., $\tilde h_{\tilde m_i^\nu}$) is like $h_{\tilde v_i}$ for $i=1,2$ (resp., $i=3,4$) with $\tilde v_i$ and $\tilde m_i^\mu$ (resp., $\tilde m_i^\nu$) swapped, 
$\tilde{h}_{R_{ij}^{\perp}}= 2 \alpha R_{ij}^{\perp}$, and $\tilde \Sigma$ is the constant rank-2 matrix whose non-zero entries are 
\begin{align*}
\tilde \Sigma_{\tilde{v}_{i}\tilde{v}_{j}} = \tilde \Sigma_{\tilde{m}_{i}^\mu \tilde{m}_{j}^\mu}= \tilde \Sigma_{\tilde{v}_{i}\tilde{m}_{j}^\mu} & = 	3\alpha^{2}a_{i,\mu}a_{j,\mu}\quad \mbox{if }  i,j\in \{1,2\}\,, \\ 
\tilde \Sigma_{\tilde{v}_{i}\tilde{v}_{j}} = \tilde \Sigma_{\tilde{m}_{i}^\nu \tilde{m}_{j}^\nu}= \tilde \Sigma_{\tilde{v}_{i}\tilde{m}_{j}^\nu} & = 	3\alpha^{2}a_{i,\nu}a_{j,\nu}\quad \mbox{if }  i,j\in \{3,4\}\,, \\
\tilde \Sigma_{\tilde{v}_{i}\tilde{v}_{j}}  = \tilde \Sigma_{\tilde{m}_{i}^\mu \tilde{m}_{j}^\nu}= \tilde \Sigma_{m_i^\mu,v_j} = \tilde \Sigma_{\tilde{v}_{i}\tilde{m}_{j}^\nu} & = 	- \alpha^{2}a_{i,\mu}a_{j,\nu}\quad \mbox{if }  i\in \{1,2\}, j\in \{3,4\} \,.
\end{align*}
\end{prop}
Numerical simulations in Figure~\ref{fig:XOR-diffusions} confirm these degenerate diffusive limits at finite $\lambda$.  

\part{Proofs}

\section{Proof of \prettyref{thm:main}}\label{sec:proof-of-convergence-thm}

In this section, we prove our main convergence result, namely Theorem~\ref{thm:main}. The drift terms can be seen from a Taylor expansion out to second order, with the role played by $\delta_n$-localizability being to justify neglecting certain negligible second order terms, as well as all higher order terms. The identification of the stochastic term is via the classical martingale problem~\cite{StroockVaradhan06} for summary statistics of stochastic gradient descent in the high-dimensional $n\to\infty$ limit.

\subsection*{Notational remark} For ease of notation, in the following we say that $f\lesssim g$ if there is some constant  $C>0$ such that 
$f\leq C g$ and that $f\lesssim_a g$ if there is some constant $C(a)>0$ depending only on $a$ such that
$f\leq C(a) g$. We will often suppress the dependence on $n$ in subscripts, when it is clear from context.

\begin{proof}[\textbf{\emph{Proof of \prettyref{thm:main}}}]
Our aim is to establish $\mathbf{u}_n\to \bu $ weakly as random variables on $C([0,\infty))$ where $\bu$ solves~\eqref{eq:effective-dynamics}. It is equivalent to show the same on $C([0,T])$ equipped with the sup-norm for every $T>0$. 

Let $\tau_K^n$ denote the exit time for the interpolated process $\bu_n(t)$ 
from $E_K^n$. Define its pre-image $E^*_{K,n}: = \bu_n^{-1}(E_K^n)$ and let $L^\infty_{K,n} = L^\infty(E^*_{K,n})$. 
For a function $f$, we use the shorthand $f_\ell$ to denote $f(X_\ell)$. 
By Taylor's theorem, we have that for any $C^3$ function $f$ and any $\ell\leq \tau_{K}^n/\delta$,
\begin{align}\label{eq:f-t-expression}
f_\ell & =f(X_{\ell-1}-\delta\nabla\Phi_{\ell-1}-\delta\nabla H^\ell_{\ell-1}) \nonumber \\
 & =f_{\ell-1} +\delta[A_{\ell}^{f}-A_{\ell-1}^{f}]+\delta[M_{\ell}^{f}-M_{\ell-1}^{f}]+O(\delta^{3}\norm{\nabla^{3}f}_{L^\infty_{K,n}}\cdot \norm{\nabla L}^{3}_{L^\infty_{K,n}})\,,
\end{align}
where $A_{\ell}^f$ and $M_\ell^f$ are defined by their increments as follows:
\begin{align*}
A_{\ell}^f-A_{\ell-1}^f & = {\big( -\cA_n + \delta \cL_n\big) f_{\ell-1}} +  \tfrac{1}{2} \left\langle \nabla\Phi\tensor\nabla\Phi,\nabla^{2}f\right\rangle _{\ell-1}\,, \\
M_{\ell}^{f}-M_{\ell-1}^{f} & = - \big\langle \nabla H^{\ell},\nabla f\big\rangle _{\ell-1} + \delta (\mathcal{E}_{\ell}^{f} - \mathcal E_{\ell-1}^f)\,, \\
\mathcal{E}_{\ell}^{f}  - \mathcal E_{\ell-1}^{f} & = \nabla^{2}f(\nabla\Phi,\nabla H^{\ell})_{\ell-1} +  \tfrac{1}{2} \big\langle \nabla^{2}f,\nabla H^{\ell}\tensor\nabla H^{\ell}-V \big\rangle_{\ell-1}\,,
\end{align*}
for $\cA_n = \langle \nabla \Phi,\nabla \rangle$, $\cL_n= \frac{1}{2} \sum_{i,j} V_{ij} \partial_i \partial_j$ and $V = \mathbb E [ \nabla H \otimes \nabla H]$ as in~\eqref{eq:A-L-operators}. 
Observe that $A_\ell^f$ is previsible (with respect to the filtration generated by $(Y_1,...,Y_{\ell-1})$), and $M_\ell^f$ is a martingale. 
We bound these for $f = u_j$ among $\mathbf{u}_n=(u_{1},...,u_{k})$. 

Recalling Definition~\ref{defn:localizable}, since $\mathbf u_n$ are $\delta_n$-localizable,
the error term in~\eqref{eq:f-t-expression} has
\[
\delta^{3}\sup_{x\in E_{K,n}^*}\E[\norm{\nabla^{3}u_{j}}\cdot\norm{\nabla L}^{3}] \lesssim\delta^{3}\norm{\nabla^{3}u_{j}}_{L^\infty_{K,n}}\left(\norm{\nabla\Phi}^{3}_{L^\infty_{K,n}}+\sup_{E^*_{K,n}} \E\norm{\nabla H}^{3}\right)\lesssim_K \delta^{3/2}\,.
\]
Since $\delta_n$ goes to $0$ as $n\to\infty$,  we may thus write $u_{j}(X_{\ell})$ as 
\[
u_{j}(X_\ell)=u_{j}(0) + \delta\sum_{\ell'\le \ell} \big(A^{u_j}_{\ell'}-A^{u_j}_{\ell'-1}\big)  + \delta\sum_{\ell'\le \ell} \big(M^{u_j}_{\ell'}-M^{u_j}_{\ell'-1}\big)+o(1)\,,
\]
where the last term is $o(1)$ in $L^{1}$ uniformly for $\ell\leq\tau_{K}/\delta$. 
Now let us define for $s\in[0,T]$, 
\begin{align*}
a_{j}'(s) & =A^{u_j}_{[s/\delta]}-A^{u_j}_{[s/\delta]-1} \qquad \mbox{and} \qquad b'_{j}(s) =M^{u_j}_{[s/\delta]}-M^{u_j}_{[s/\delta]-1}
\end{align*}
If we let 
$$a_{j}(s)=\int_{0}^{s}a'_{j}(s')ds'=a_{j}(\delta[s/\delta])+(s-\delta[s/\delta])(A^{u_j}_{[s/\delta]}-A^{u_j}_{[s/\delta]-1})$$
and similarly $b_{j}(s)=\int_{0}^{s}b'_{j}(s')ds'$, then recalling that $\bu_n(s)$ is the linear interpolation of $(u_j([s/\delta]))_{j}$,
we may write  
$$\mathbf{u}_n(s)= \mathbf{u}_n({0}) + \mathbf{a}_n(s)+\mathbf{b}_n(s)+o(1).$$
where $\mathbf a_n(s)  = (a_j(s))_{j}$ and $\mathbf b_n(s) = (b_j(s))_j$. 

We now prove that the sequence $(\mathbf{u}_{n}(s\wedge\tau_{K}^n ))$
is tight in $C([0,T])$ with limit points which
are $(1/4)$-Holder for each $K$. To this end, let us 
define $$\mathbf{v}_n(s)=\mathbf{u}_n({0}) + \mathbf{a}_n(s)+\mathbf{b}_n(s)\,.$$
As the $o(1)$ error above is uniform in $t$, we have that 
\[
\sup_{0\leq s\leq\tau_{K}^n}\norm{\mathbf{u}_n(s)-\mathbf{v}_n(s)}\to 0\,, \qquad \mbox{in $L^1$}\,.
\]
Thus it suffices to show the claimed tightness and Holder properties of limit points for $\mathbf{v}_n$ instead of~$\mathbf{u}_n$. We
aim to show that for all $0\le s,t\le T$, 
\begin{align}\label{eq:v-n-kolmogorov-continuity}
\E\norm{\mathbf{v}_n(s\wedge\tau_{K})-\mathbf{v}_n(t\wedge\tau_{K})}^{4}\lesssim_{K,T}(t-s)^{2},
\end{align}
from which we will get that the sequence $\mathbf{v}_{n}(s\wedge\tau_{K})$
is uniformly $1/4$-H\"older by Kolmogorov's continuity theorem. 
Evidently, for all $s,t$ we have that
\[
\|\mathbf{v}_n(s)-\mathbf{v}_n(t)\| \leq \|\mathbf{a}_n(s)-\mathbf{a}_n(t)\| + \|\mathbf{b}_n(s)-\mathbf{b}_n(t)\|\,.
\]
We control these terms in turn. We will do this coordinate wise and,
for readability, fix some $j\le k$ and let $u=u_j$, $a=a_j$, $b=b_j$ etc.

For the previsible term, we have
\begin{align}\label{eq:previsible-term}
\E\abs{a(s\wedge\tau_{K})& -a(t\wedge\tau_{K})}^{4} \nonumber \\ 
& \lesssim \E \big| \delta \sum_{\ell} \big((-\cA_n + \delta \cL_n) u\big)_\ell  \big|^4 +  \E\big|{ \delta^{2}  \sum_\ell \left\langle \nabla\Phi\tensor\nabla\Phi,\nabla^{2}u\right\rangle_\ell }\big|^{4},
\end{align}
where these sums are over steps $\ell$ ranging from $[s/\delta] \wedge \tau_K/\delta$ to $[t/\delta] \wedge \tau_K/\delta$. 

Let $\mathbf{h} = (h_j)_{j\le k}$ be as in~\eqref{eq:eff-drift}. 
Then
the first term in~\eqref{eq:previsible-term}
satisfies 
\begin{align*}
\E \big| \delta \sum_{\ell} \big((-\cA_n + \delta \cL_n) u\big)_\ell \big|^4 & \lesssim_K \E\abs{\delta\sum_\ell h_{j}(\mathbf u_n)_\ell}^{4}+o((t-s)^{4})\\
 & \leq (t-s)^{4}\left(\norm{h_{j}}_{L^{\infty}(E_{K}^n)}^{4}+o(1)\right) \lesssim_K (t-s)^4
\end{align*}
by continuity of $h_j$. 
For the second term in~\eqref{eq:previsible-term}, 
\begin{align*}
\E\abs{\delta^{2}\sum\left\langle \nabla\Phi\tensor\nabla\Phi,\nabla^{2}u\right\rangle_\ell }^{4} & \leq\delta^{8}\Big(\abs{((t-s)/\delta)}\sup_{x\in E_{K,n}^*}\norm{\nabla\Phi(x)}^{2}\sup_{x\in E_{K,n}^*}\norm{\nabla^{2}u(x)}_{\op}\Big)^{4}
\end{align*}
which is $\lesssim_K \delta^{2}(t-s)^{4}$ 
by items (1)--(2) $\delta_n$-localizability. (Applying this bound for $s=0,t=T$,
the last term in $a$ is vanishing in the limit for each $K$ whenever $\delta_n = o(1)$.)
Combining the above bounds yields
\[
\E\abs{a(s\wedge\tau_{K})-a(t\wedge\tau_{K})}^{4}\lesssim_{K}(t-s)^{4}.
\]

For the martingale term, notice that by Burkholder's inequality, 
\[
\E|b(s\wedge \tau_K) - b(t\wedge \tau_K)|^4 = 
\E\left[\left(\delta\sum(M^{u}_{\ell}-M^{u}_{\ell-1})\right)^{4}\right] \lesssim \E\left[\left(\delta^{2}\sum(M^{u}_{\ell}-M^{u}_{\ell-1})^{2}\right)^{2}\right]\,,
\]
where the sum again runs over steps $\ell$ ranging from $[s/\delta]\wedge \tau_K$ to $[t/\delta]\wedge \tau_K$. 
Repeatedly using the inequality $(x+y+z)^{2}\lesssim x^{2}+y^{2}+z^{2}$,
it suffices to bound the above quantity for each of the three terms
defining the martingale difference $M_\ell^u - M_{\ell-1}^u$ respectively. 

For the first term in that martingale difference, observe that
\begin{align}\label{eq:M-term-1}
\E\Big[\Big(\delta^{2}\sum_\ell \big\langle \nabla H^{\ell},\nabla u\big\rangle _{\ell -1}^{2}\Big)^{2}\Big] & =\delta^{4}\sum_{\ell,\ell'}\E\Big[\big\langle \nabla H^{\ell},\nabla u\big\rangle _{\ell-1}^{2}\big\langle \nabla H^{\ell'},\nabla u\big\rangle _{\ell'-1}^{2}\Big] \nonumber \\
 & \leq\Big(\delta\sum_\ell\Big(\delta^{2}\E\big\langle \nabla H^{\ell},\nabla u\big\rangle _{\ell -1}^{4}\Big)^{1/2}\Big)^{2}  \lesssim_{K}\left(t-s\right)^{2},
\end{align}
where in the second line we used Cauchy-Schwarz and in the last we used item (3) of $\delta_n$-localizability.

For the second term in the martingale difference, 
\begin{align}\label{eq:M-term-2}
\E\Big[\Big(\delta^{4}\sum_\ell & \big(\nabla^{2}u(\nabla\Phi,\nabla H^{\ell})_{\ell-1}\big)^{2}\Big)^{2}\Big]   \nonumber\\
 & \leq\delta^{6}(t-s)^{2} \sup_{x\in E^*_{K,n}} \norm{\nabla^{2}u(x)}_{\op}^4 \cdot\norm{\nabla\Phi(x)}^4 \cdot\E\norm{\nabla H(x)}^{4} \lesssim_K \delta^{2}(t-s)^2\,,
\end{align}
 by items (1)--(2) of $\delta_n$-localizability.  Finally, by the same reasoning, for the third term, 
\begin{align}\label{eq:M-term-3}
\E\Big[\Big(\delta^{4}\sum_\ell & \big\langle \nabla^{2}u,\nabla H^{\ell}\tensor\nabla H^{\ell}-V\big\rangle _{\ell-1}^{2}\Big)^{2}\Big]  \nonumber \\
& \lesssim\delta^{6}(t-s)^{2}\sup_{x\in E^*_{K,n}} \norm{\nabla^{2}u(x)}_{\op}^4 \cdot \E\big[\norm{\nabla H(x)}^{8}\big] \lesssim_K (t-s)^2\,.
\end{align}
All of the above terms are $O((t-s)^{2})$ since $0\le s,t\le T$. Thus we have the claimed~\eqref{eq:v-n-kolmogorov-continuity}, 
and
by Kolmogorov's continuity theorem, $(\mathbf{v}_n(s\wedge \tau_K))_s$, are uniformly
$\sfrac 14$-Holder and thus the sequence is tight with $\sfrac 14$-Holder limit points. 
Notice
furthermore that if we look at $(\mathbf{v}_n(t\wedge \tau_K)-\mathbf{a}_n(t\wedge \tau_K))_t$,
this sequence is also tight and the limits points are continuous martingales.
Let us examine their limiting quadratic variations.

Let $\mathbf v_n^K(t) = \mathbf v_n(t\wedge \tau_K)$ and define $\mathbf a_n^K(t)$ and $\mathbf b_n^K(t)$ analogously. Furthermore, let $\mathbf v^K(t)$, $\mathbf a^K(t)$ and $\mathbf b^K(t)$ be their respective limits which we have shown to exist and be $\sfrac{1}{4}$-Holder. 

We will compute the limiting quadratic variation for $\mathbf{b}^K(t)$. For ease of notation, let $\Delta M^{u_i}_{\ell} = M_{\ell}^{u_i}-M_{\ell-1}^{u_i}$ and $\Delta \mathcal E_\ell^{u_i} = \mathcal E_{\ell}^{u_i} - \mathcal E_{\ell-1}^{u_i}$. 
Notice first that for $1\le i,j\le k$, 
\begin{align*}
	b_{n,i}^{K}(t)b_{n,j}^{K}(t) - \int_{0}^{t} \delta \mathbb E\big[\Delta M^{u_i}_{[s/\delta]\wedge \tau_K} \Delta M_{[s/\delta]\wedge \tau_K}^{u_j} \big] ds\,,
\end{align*}
is a martingale. 
We therefore need to consider the limit as $n\to\infty$ of the integral above. 
Write
\begin{align}\label{eq:martingale-quadratic-variation}
	 \mathbb E[\Delta M_{\ell}^{u_i} \Delta M_{\ell}^{u_j}]  & = \langle \nabla u_i ,V\nabla u_j\rangle +  \delta \mathbb E [ \langle\nabla H^\ell,\nabla u_i \rangle_{\ell - 1} \Delta \mathcal E_{\ell}^{u_j}] + \delta\mathbb E [\Delta \mathcal E_{\ell}^{u_i}\langle \nabla H^\ell,\nabla u_j\rangle_{\ell-1}] \\
	 & \qquad + \delta^2 \mathbb E[\Delta \mathcal E_{\ell}^{u_i} \Delta\mathcal E_{\ell}^{u_j}]\,. \nonumber
\end{align}
 Consider the integrals of $\delta$ times each of these four terms separately. 
For the first term, 
\begin{align*}
\sup_{t\leq T}\Big | \int_{0}^{t}\delta\left\langle \nabla u_{i},V\nabla u_{j}\right\rangle _{[s/\delta]\wedge \tau_K} -  {\Sigma}_{ij}(\mathbf{v}_{n}^K(s))ds \Big | \leq T \sup_{x\in E_{K,n}^*}\abs{\delta\left\langle \nabla u_{i},V \nabla u_{j}\right\rangle(x) -{\Sigma}_{ij}(\mathbf{u}_n(x))}\,,
\end{align*}
goes to zero as $n\to\infty$ by the assumption in~\eqref{eq:diffusion-matrix}. 

We now reason that the integrals of $\delta$ times the other three terms in~\eqref{eq:martingale-quadratic-variation} all go to zero as $n\to\infty$. The second and third are identical: by Cauchy--Schwarz, 
\begin{align*}
	\sup_{x\in E_{K,n}^*} |\delta^2 \mathbb E[ \langle \nabla H,\nabla u_i\rangle \Delta \mathcal E_{\ell}^{u_j} ] |\le \delta^2 \mathbb E[\langle\nabla H,\nabla u_i\rangle^2]^{1/2} \mathbb E[(\Delta \mathcal E_{\ell}^{u_i})^2]^{1/2}\,.
\end{align*}
The first expectation contributes $\delta^{-1/2}$ by the first part of item (3) of localizability. Also, 
\begin{align}\label{eq:delta-E-square-moment}
	\mathbb E[(\Delta \mathcal E_{\ell}^{u_i})^2]^{1/2} \lesssim \mathbb E[\langle \nabla^2 u_i ,\nabla \Phi\otimes \nabla H\rangle^{2}]^{1/2}  + \mathbb E[\langle \nabla^2 u_i ,\nabla H\otimes \nabla H - V\rangle^2]^{1/2}\,.
\end{align}
The first of these terms is at most $\delta^{-1}$ as argued in~\eqref{eq:M-term-2}. The second is $o(\delta^{-3/2})$ by the  second part of item (3) in the definition of localizability. 
As such, we are able to conclude that 
\begin{align*}
	\sup_{t\le T} \Big| \int_0^t \delta^2 \mathbb E[ \langle \nabla H,\nabla u_i\rangle_{[s/\delta]\wedge \tau_K} \Delta \mathcal E_{[s/\delta]\wedge \tau_K}^{u_j} ] ds \Big |\,,
\end{align*} 
goes to zero as $n\to\infty$. 

The integral of $\delta$ times the fourth term in~\eqref{eq:martingale-quadratic-variation} is handled similarly using Cauchy--Schwarz and the bound of $o(\delta^{-3/2})$ on~\eqref{eq:delta-E-square-moment}. 

Altogether, we end up with 
\begin{align*}
	\lim_{n\to\infty} \sup_{i,j\le k} \sup_{t\le T}\Big| \int_{0}^{t} \delta \mathbb E[\Delta M_{[s/\delta]\wedge \tau_K}^{u_i} \Delta M_{[s/\delta]\wedge \tau_K}^{u_j}] ds - \int_{0}^t \Sigma_{ij} (\mathbf{v}_n^K(s))ds\Big| = 0\,.
\end{align*}
Thus, if we consider the continuous martingales given by $\mathbf b^K(t)$, its angle bracket is, by definition, given by 
$$\langle \mathbf b^K\rangle_t =\int_{0}^{t}\mathbf{\Sigma}(\mathbf{v}^{K}(s))ds \,.$$

By Ito's formula for continuous martingales (see, e.g., \cite[Theorem 5.2.9]{EthierKurtz86}),
we have that $f(\mathbf{v}_t)- \int_0^t \mathsf{L} f(\mathbf{v}_s) ds$ is a martingale for all 
 $f\in C_0^\infty(\R^k)$, where 
\[
\mathsf{L} = \frac{1}{2}\sum_{ij=1}^k \Sigma_{ij}\partial_i\partial_j-\sum_{i=1}^{k} h_i \partial_i.
\]
Since, by assumption, $\mathbf{h},\sqrt{\Sigma}$ are locally Lipschitz---and thus Lipschitz on $E_K$---this property uniquely characterizes
the solutions to \eqref{eq:effective-dynamics} (see, e.g., \cite[Theorem 6.3.4]{StroockVaradhan06}).
Thus $\mathbf{v}_K$ converges to the solution of \eqref{eq:effective-dynamics} stopped at $\tau_K$.
By a standard localization argument \cite[Lemmas 11.1.11-12]{StroockVaradhan06}, every limit point $\mathbf{v}(t)$ of $\mathbf{v}_n(t)$ solves the SDE~\eqref{eq:effective-dynamics} (using here that $E_{K}$ is an exhaustion by compact sets of $\R^k$).
\end{proof}

\section{Proofs for matrix and tensor PCA}\label{app:matrix-tensor-pca}
In this section, we prove the results of Section~\ref{sec:matrix-tensor-pca}. We will state them in the more general setting where we add a ridge penalty to the loss, so that for $\alpha\ge 0$ fixed, the loss is given by 
\begin{align}\label{eq:L-tensor-pca}
	 L (x,Y) =  - 2(\langle W,x^{\otimes k}\rangle+\lambda\langle x,v\rangle^{k})+ \norm{x}^{2k} + \tfrac{\alpha}{2} \|x\|^2 +c(Y)\,,
\end{align}
where $c(Y)$ only depends on $Y$. Note that $H(x) = -2\langle W,x^{\otimes k}\rangle$.

Our first aim is to establish Proposition~\ref{prop:matrix-tensor-effective-dynamics}, showing that the summary statistics $\mathbf{u}_{n}=(m,r_{\perp}^{2})$ satisfy the conditions of Theorem~\ref{thm:main} with the desired $\mathbf f, \mathbf g$ and $\Sigma$. We begin by checking localizability for $\mathbf{u}_n$. In what follows, 
for ease of notation we will denote $r^{2}=r_{\perp}^{2}$ and $R^2 = m^2 + r^2$.  In these coordinates, 
\begin{align}\label{eq:Phi-tensor-pca}
\Phi(x)=- 2\lambda m^{k}+ (r^{2}+m^{2})^{k} + \frac{\alpha}{2} (r^2 + m^2)+c'
\end{align}

\begin{lem}
	The distribution of $L(x,Y)$ depends only on $\mathbf{u}_n = (m,r^2)$. Furthermore, if $\lambda$ is fixed and $\delta_n = O(1/n)$, then  $\mathbf{u}_n$ is $\delta_n$-localizable for $E_K$ being the centered balls of radius $K$ in $\mathbb R^2$.  
\end{lem}

\begin{proof}
We check the items in Definition~\ref{defn:localizable} one by one, beginning with item (1). 
Express the derivatives for $\mathbf{u}_n$ as  
\begin{align}\label{eq:matrix-tensor-Jacobian}
\nabla m  =v\,, \qquad \nabla r^{2} =2(x-m v)\,. 
\end{align}
Notice that $\nabla^2 m = 0$, while $\nabla^2 r^2 = 2 (I - vv^T) $, whose operator norm is simply $2$, and $\nabla^\ell u_i=0$ for all $\ell \ge 3$.

For item (2), differentiating~\eqref{eq:Phi-tensor-pca}, $
\nabla\Phi  =\partial_{1}\phi\nabla m+ \partial_{2}\phi\nabla r^{2}$,
where 
\begin{align*}
\partial_{1}\phi & =- 2 \lambda km^{k-1}+ (2kR^{2k-2}+\alpha)m \qquad \partial_{2}\phi = kR^{2k-2}+\tfrac\alpha2.
\end{align*}
Notice that $\left\langle \nabla m,\nabla m\right\rangle =1,\left\langle \nabla m,\nabla r^{2}\right\rangle =0,$
and
$\left\langle \nabla r^{2},\nabla r^{2}\right\rangle =4r^{2}$. 
Consider $$\|\nabla \Phi\| \le |\partial_1 \phi| \|\nabla m\| + |\partial_2 \phi| \|\nabla r^2\|\,;$$ the bounding quantity is evidently a continuous function of $m,r^2$ and therefore as long as $x$ is such that $(m,r^2)\in E_K$, it is bounded by some $C(K)$. Next, if we consider 
$$
\mathbb E[\|\nabla H\|^8] 
\le C_k \mathbb E [ \|W(x,\ldots,x,\cdot)\|^8] \le \mathbb E\|W\|_{\op}^{8} \cdot R^{8k} \le C(k,K) n^{4}
$$
where the bound on the operator norm of an i.i.d.\ Gaussian $k$-tensor can be found, e.g., in~\cite[Lemma 4.7]{BGJ18a}. 
Moving on to item (3), by the same reasoning, for every $w$,  
\begin{align*}
\mathbb E[\langle \nabla H, w \rangle^4] \le 16 k \mathbb E [|W(w,x,\ldots,x)|^4]\le C(k,K) n^2 \|w\|\,.
\end{align*}
If $w=\nabla m = v$ then $\|w\|=1$ and if $w=\nabla r^2 = 2(x-m v)$ then $\|w\|\le C(K)$, so in both cases this is at most $C(k,K)n^2$. Finally, $\nabla^2 u$ is only non-zero if $u=r$ in which case it is $I - vv^T$. Then, 
\begin{align*}
	\mathbb E[\langle \nabla^2 r, \nabla H \otimes \nabla H - V\rangle^2]  \le 2 \mathbb E[\|\nabla H\|^4] \le C(k,K) n^2
\end{align*}
by the second item in the definition of localizability, and evidently the right-hand side is $O(\delta^{-2})$ if $\delta_n = O(1/n)$. 
\end{proof}

\begin{proof}[\textbf{\emph{Proof of Proposition~\ref{prop:matrix-tensor-effective-dynamics}}}]
Having checked localizability for $\mathbf{u}_n$, we apply Theorem~\ref{thm:main}. To compute $\mathbf f$, by the above, 
\begin{align*}
	f_{m} & = -2 \lambda km^{k-1} + (2 kR^{2k-2} + \alpha)m\,, \qquad f_{r^2}  = 2r^2 (2k R^{2k-2} +\alpha)\,.
\end{align*}
We next turn to calculating the corrector. For this, we first calculate the matrix $V = \mathbb E[\nabla H\otimes \nabla H]$. Recalling that $H = -2\langle W, x^{\otimes k}\rangle$ where $W$ is an i.i.d.\ Gaussian $k$-tensor, we have that 
\begin{align}\label{eq:matrix-tensor-V}
	V_{ij}  = \mathbb E [ \partial_i H \partial_j H ] 
	& = 4 k (k-1) x_i x_j  R^{2k-4} + 4 k  R^{2k-2}\mathbf 1\{i=j\}\,.
\end{align}
In particular, for $\delta = c_\delta /n$, we have  $\delta \mathcal L^\delta m  = 0$ and 
\begin{align*}
 \delta \mathcal L^\delta r^2 & 
  = \frac{4c_\delta}{n} k \Big((n-1) R^{2k-2} + (k-1) r^2 R^{2k-4}\Big) 
\end{align*}
from which we obtain in the limit that $n\to\infty$ that $g_m = 0$ and $g_{r^2} = 4 c_\delta k R^{2k-2}$. 

Together, these yield the ODE system of~\eqref{eq:matrix-tensor-ODE},
\begin{align*}
\dot{u}_{1}=  2 u_1 (\lambda ku_{1}^{k-2}-kR^{2k-2} -\alpha)\,,\qquad\qquad
\dot{u}_{2}=  - (4 u_{2} - 4 c_\delta)kR^{2k-2}- 2 \alpha u_2\,.
\end{align*}
which in the $\alpha = 0$ case matches Proposition~\ref{prop:matrix-tensor-effective-dynamics}. 
Finally, to see that $\Sigma = 0$, consider 
\begin{align}\label{eq:matrix-tensor-JVJ^T}
	J V  J^T = \begin{pmatrix} 	4 k(k-1)m^2 R^{2k-4} + 4 kR^{2k-2} & 4 k (k-1) m(R^2 - m)R^{2k-4} \\ 4 k (k-1) m(R^2 - m)R^{2k-4}  & 4 k(k-1)(R^2 - m)^2R^{2k-4}	\end{pmatrix}\,,
\end{align} 
which when multiplied by $\delta = O(1/n)$ evidently vanishes. 
\end{proof}

\subsection{The fixed points of Proposition~\ref{prop:matrix-tensor-effective-dynamics}}
We now turn to analyzing the ODE of Proposition~\ref{prop:matrix-tensor-effective-dynamics}.

\begin{proof}[\textbf{\emph{Proof of Proposition~\ref{prop:matrix-tensor-fixed-points}}}]
At the fixed points of the ODE in Proposition~\ref{prop:matrix-tensor-effective-dynamics},
\begin{align*}
\lambda ku_{1}^{k-1} & =\big(kR^{2k-2}+\alpha\big)u_{1}\,, \qquad \mbox{and} \qquad 2c_\delta kR^{2k-2} = \big(2kR^{2k-2}+\alpha\big)u_{2}\,.
\end{align*}
If $u_{1}=0$, then $R^2 = u_2$ and there are two possible fixed points: either $u_{2}=0$
or $u_{2}$ solves 
\[
ku_{2}^{k-2}(2 c_\delta - 2 u_{2})= \alpha.
\]
Notice that if $k=2$, this has a nontrivial solution of the form
$c_\delta -\frac{\alpha}{2}=u_{2}$, provided $\alpha<\alpha_c(2):= 2 c_\delta$, and if $k>2$,
this has a nontrivial solution provided $\alpha\le\max_{x\ge 0}kx^{k-2}(2c_\delta - 2x)$ at $c_\delta (k-2)x^{k-3}-(k-1)x^{k-2}=0$ i.e.,
$\frac{c_\delta (k-2)}{k-1}=x$. This gives 
\[
\alpha< \alpha_c(k) :=2 c_\delta^{k-1}  k(k-1)^{-(k-1)}(k-2)^{k-2}.
\]
Evidently when we take $\alpha = 0$, then its non-trivial solution is at $u_2 =1$ for all $k\ge 2$. 

Alternatively, if $u_{1}\neq0$ at a fixed point, then we can simplify further and get 
\begin{align*}
\lambda u_{1}^{k-2} & =R^{2k-2}+\alpha/k\,, \qquad \mbox{and} \qquad kR^{2k-2}=(kR^{2k-2}+\alpha)u_{2}\,,
\end{align*}
so that at the fixed point, 
\begin{align*}
u_{1}^{k-2} & = \frac{kR^{2k-2}+\alpha}{\lambda k}\,,\qquad \mbox{and}\qquad u_{2} =\frac{2 c_\delta kR^{2k-2}}{2kR^{2k-2}+\alpha}\,.
\end{align*}
For simplicity of calculations, set $\alpha =0$ as is the case in Proposition~\ref{prop:matrix-tensor-effective-dynamics}. 
Then, we simply get $u_2 = c_\delta$. In the case of $k=2$, we also find that there is a solution if and only if $\lambda>c_\delta$, in which case $R^{2} = \lambda$, from which together with $R^2 = u_1^2 + u_2$, we also get $u_1 = \pm\sqrt{\lambda -c_\delta}$. 

In the general case of $k>2$, we find that $R^{2} = {c_\delta}+\lambda^{ - \frac{2}{k-2}} R^{\frac{4(k-1)}{k-2}}$. 
This has real solutions (all of which have $R\ge u_2 = c_\delta$ as required) whenever $\lambda>\lambda_c(k)$ defined as 
\begin{align}\label{eq:lambda-c-k}
	\lambda_c(k): = \Big(\frac{c_\delta}{k}\Big)^{k/2} \Big(\frac{(2k-2)^{k-1}}{(k-2)^{(k-2)/2}}\Big)\,.
\end{align}
(Interpreting $0^0 =1$, this returns $\lambda_c(2) = c_\delta$.) 
With this $\lambda$, whenever $\lambda>\lambda_c(k)$, the equation for $R^2$ has exactly two real solutions, both of which are at least $c_\delta$ which we can denote by 
\begin{align*}
	\rho_\dagger(k,\lambda) & := \inf\{\rho\ge 1: \lambda^{-\frac{2}{k-2}} \rho^{\frac{2(k-1)}{k-2}} - \rho + c_\delta =0\}\,, \\
	 \rho_\star(k,\lambda) & := \sup\{\rho\ge 1: \lambda^{-\frac{2}{k-2}} \rho^{\frac{2(k-1)}{k-2}} - \rho +c_\delta =0\}\,.
\end{align*}
When $\lambda>\lambda_c(k)$, $\rho_\dagger <\rho_\star$ and when $\lambda = \lambda_c(k)$, the two are equal. Given this, we can then solve for $\tilde u_1$ at the corresponding fixed point, and find that they occur at 
\begin{align}\label{eq:m-dagger-m-star}
	m_\dagger (k,\lambda) = \sqrt{\rho_\dagger -c_\delta}\,, \qquad \mbox{and}\qquad m_\star(k,\lambda) = \sqrt{\rho_\star - c_\delta}\,,
\end{align}
as claimed.
\end{proof}

\subsection{Effective dynamics for the population loss}\label{subsec:tensorPCA-pop-loss-effective-dynamics}
In practice, one is interested in tracking the loss, or ideally, the generalization error. In this subsection, we add the generalization error $\Phi$ to our set of summary statistics and obtain limiting equations for its evolution from~\eqref{eq:Phi-fixed-points}. 

Recalling~\eqref{eq:Phi-tensor-pca}, the fact that $\Phi$ is a localizable summary statistic follows from the facts that $\|\nabla m\|,\|\nabla r^2\|\le C(K)$, and the fact that $\Phi$ is a smooth $n$-independent function of $m,r^2$. 
 
For simplicity of calculations let us stick to $\alpha = 0$. 
\begin{align*}
	f_\Phi = \langle \nabla \Phi,\nabla \Phi\rangle & =  4\lambda^2 k^2 m^{2(k-1)} -8\lambda k^2 m^k R^{2k-2} + 4 k^2 R^{4k-4} m^2 + 4k^2 r^2 R^{4k-4} \\ 
	& = 4k^2  m^2 \big( \lambda^2 m^{2(k-2)} - 2\lambda m^{k-2} R^{2k-2} + R^{4k-4}\big) + 4k^2 r^2 R^{4k-4}\,. 
\end{align*}
Next, consider the corrector for $\Phi$. For this, notice that 
\begin{align*}
	\tfrac{1}{2}\nabla^2 \Phi & = -\lambda k (k-1) m^{k-2} \nabla m^{\otimes 2} +  kR^{2k-2}\nabla m^{\otimes 2} +  k(k-1)R^{2(k-2)} (2 m\nabla m + \nabla r^2)\otimes \nabla m  \\
	&\qquad +  k (k-1) R^{2(k-2)} ( 2m \nabla m \otimes \nabla r^2 + \nabla r^2 \otimes \nabla r^2) + \tfrac{1}{2} \partial_2 \phi \nabla^2 r^2\,.
\end{align*}
Recalling $V$ from~\eqref{eq:matrix-tensor-V}, and taking $\delta = c_\delta /n$, all the terms in $\sum_{ij} V_{ij} \partial_i \partial_j \Phi$ vanish in the limit except the contribution from the $\nabla^2 r^2$, which yields $g_\Phi = \lim_{n\to\infty} \delta \cL^\delta \Phi = 4 c_\delta k^2  R^{4(k-1)}$
Finally, we wish to compute the volatility for the stochastic part of the evolution of $\Phi$. For this, consider $\nabla \Phi V \nabla\Phi^T$ and notice that all the entries of that matrix are continuous functions of $\mathbf{u}_n$ and thus go to zero when multiplied by $\delta = O(1/n)$.

\subsection{Diffusive limits at the equator}\label{app:diff-equator}
In this subsection, we develop the stochastic limit theorems for the rescaled observables about the axis $m=0$. 
Here we take as variables $(\tilde u_{1},\tilde u_{2})=(\sqrt{n}m,r^{2})$. For simplicity of presentation, we take $\alpha =0$ and $c_\delta =1$.

\begin{proof}[\textbf{\emph{Proof of Proposition~\ref{eq:rescaled-tensor-pca}}}]
We begin by checking localizability. The change from the original variables is in the $J$ matrix, in which now $\nabla \tilde u_1 = \sqrt n \nabla m = \sqrt n v$. This does not affect items (1)--(2) of localizability; for item (3), notice that 
$$\mathbb E[\langle \nabla H,\nabla m\rangle^4] =n^2 \mathbb E[\langle \nabla H,v\rangle^4] \le n^2 \mathbb E[W_{1,...,1}^4]  \le C n^2\,.$$
The second part of item (3) is unchanged since $\nabla^2 \tilde u_1 =0$. 

Computing the drifts, 
\begin{align*}
\langle \nabla\Phi,\nabla\tilde u_{1}\rangle= & -2 k \lambda \sqrt n m^{k-1} +2 k \sqrt n R^{2k-2} m = - 2k \lambda n^{ - \frac{k-2}{2}} \tilde u_1^{k-1} + 2k (r^2 + ({\tilde u_1^2}/{n}))^{k-1} \tilde u_1\,, \\
\langle \nabla \Phi,\nabla r^{2}\rangle= & 4 k r^{2} R^{2k-2} = 4 k r^2 (r^2 + ({\tilde u_1^2}/{n}))^{k-1}\,.
\end{align*}
Taking limits as $n\to\infty$, as long as $\lambda$ is fixed in $n$, we see that $\mathbf{f}$ is given by 
\begin{align*}
	f_{\tilde u_1} = \begin{cases} -2 k \lambda  \tilde u_1^{k-1} + 2 k\tilde u_2^{k-1} \tilde u_1 & k=2 \\ 2 k\tilde u_2^{k-1}\tilde u_1 & k\ge 3 \end{cases}\,, \qquad\mbox{and}\qquad f_{\tilde u_2} = 4 k\tilde u_2^k\,.
\end{align*}

We turn to obtaining the correctors in these rescaled coordinates. Evidently $\delta \cL \tilde u_1 =0$ still by linearity of $\tilde u_1$. Following the calculation for the corrector, it is now given by $g_{\tilde u_2} = 4 k \tilde u_2^{k-1}$. 

Next we consider the volatility of the stochastic process one gets in the  limit. Recalling $JVJ^T$ from~\eqref{eq:matrix-tensor-JVJ^T}, and noticing that the rescaling $J\to\tilde J$ multiplies its $(1,1)$-entry by $n$ and its off-diagonal entries by $\sqrt n$, we find that in the new coordinates,  
\begin{align}\label{eq:matrix-tensor-JVJ^T-rescaled}
	\tilde J V  \tilde J^T = \begin{pmatrix} 4 k(k-1)\tilde u_1^2 R^{2k-4} + 4 k n R^{2k-2} & 4 k (k-1) \tilde u_1(R^2 - m)R^{2k-4} \\ 4 k (k-1) \tilde u_1 (R^2 - m)R^{2k-4}  & 4 k(k-1)(R^2 - m)^2 R^{2k-4}	\end{pmatrix}
\end{align} 
Multiplying by $\delta = 1/n$ and taking the limit as $n\to\infty$, the only entry of this matrix that survives is from $\Sigma_{11}$ where we get $\Sigma_{11} = 4 k\tilde u_2^{k-1}$ as claimed.
\end{proof}

Regarding the discussion in the $k\ge 3$ case of~\eqref{eq:diffusive-m-ballistic-r}, when $\lambda_n = \Lambda n^{(k-2)/2}$, observe that the first term in $\langle \Phi,\nabla \tilde u_1\rangle$ above would not vanish and would instead converge to $-4 k\Lambda \tilde u_1^{k-1}$.

\subsection{Diffusive limit for the radius}\label{app:diff-radius}
We now show how to rescale the radial term $r^2$  to obtain a diffusive limit for $r^2$ about $r^2 = 1$. (For readability, we take the case $c_\delta =1$ though an analogous result works for general $ c_\delta$.) To this end, consider $\tilde{\mathbf{u}}_n = (\tilde u_1, \tilde u_2) = (\sqrt{n} m, \sqrt{n} (r^2 -1))$. 
Now $J$ is in terms of   $\nabla \tilde u_1 = \sqrt{n} \nabla m$ and $\nabla \tilde u_2 = \sqrt{n} \nabla u_2$. 
Let us verify localizability for $\tilde{\mathbf{u}}_n$; the only changes as compared to the previous subsection are those entailing $\tilde u_2$. 

For item (1), $\|\nabla^2 \tilde u_2\|_{\op}= O(\sqrt n)$ and $\nabla^3 \tilde u_2 =0$. For the first part of item (3), 
\begin{align*}
	\mathbb E[\langle \nabla H,\nabla \tilde u_2\rangle^4] = n^2 \mathbb E[\langle \nabla H, 2(x-mv)\rangle^4] \lesssim n^2(R^{4k} + m^4) \mathbb E[W_{1,...,1}^4])\lesssim_K n^2\,,
\end{align*}  
where we used in the first inequality that the law of $H$ is rotation invariant and $H$ is a $k$-homogenous function. 
For the second part of item (3), 
\begin{align*}
	 \mathbb E[\langle \nabla^2 \tilde u_2, \nabla H\otimes \nabla H - V\rangle^2] \le n \mbox{Var}(\|\nabla H\|^2)\,.
\end{align*}
We now express 
\begin{align*}
	\mbox{Var}(\|\nabla H\|^2) = \sum_{i} \mbox{Var}((\partial_i H)^2) + \sum_{i\ne j} \mbox{Cov}((\partial_i H)^2,(\partial_j H)^2)\,.
\end{align*}
The $\partial_i H$ are Gaussian with mean zero, and by~\eqref{eq:matrix-tensor-V}, variance $C_k' R^{2(k-2)} x_i^2 + C_k R^{2(k-1)}$ and covariance $C_k x_i x_j R^{2(k-2)}$. Recall the following fact about Gaussians: if $X,Y$ are Gaussians with variances $\sigma^2$ and covariance $t$, then $\mbox{Cov}(X^2,Y^2)\le C t^2\sigma^4$ for some universal constant $C$. Also, $\mbox{Var}(X^2)\le C\sigma^4$. 
Applying this to $\partial_i H$, we get 
\begin{align*}
	\mbox{Var}(\|\nabla H\|^2) \lesssim_K n + \sum_{i, j} x_i^2 x_j^2 \lesssim_K n\,.
\end{align*}
Combined with the above, this gives a bound of $n^2 = O(\delta^{-2})$ on the second part of item (3).

We now calculate the resulting drifts. For $\mathbf{f}$, write 
\begin{align*}
	\cA_n u_1  & = -2 k \lambda \mathbf{1}_{k=2}\tilde u_1^{k-1} + 2k r^{2(k-1)} \tilde u_1= -2 k \lambda \mathbf{1}_{k=2}\tilde u_1^{k-1} + 2k (1+n^{-1/2}\tilde u_2)^{k-1} \\ 
	\cA_n \tilde u_2 & = 4kn^{1/2} r^2(r^2 + (\tilde u_1^2/n))^{k-1} = 4kn^{1/2}(1+n^{-1/2}\tilde u_2)(1+n^{-1/2}\tilde u_2 + n^{-1} \tilde u_1^2)^{k-1} \\ 
	& \qquad \qquad \qquad \qquad \qquad \qquad \quad = 4kn^{1/2} + 4 k^2 \tilde u_2 \mathbf + o(1)
\end{align*}
We next calculate the prelimits of the corrector. 
Evidently $\delta \mathcal L \tilde u_1 = 0$ still by linearity of $\tilde u_1$ and 
\begin{align*}
	\delta \mathcal L \tilde u_2 = \sqrt{n} \delta \mathcal L^\delta r^2  = \frac{4}{\sqrt n} k\Big((n-1)R^{2k-2} + (k-1)(1+n^{-1/2} \tilde u_2)R^{2k-4}\Big) 
\end{align*}
Combining terms and sending $n\to\infty$, we obtain 
\begin{align*}
	f_{\tilde u_1}-g_{\tilde u_1}  = -2 k  \lambda \mathbf{1}_{k=2} \tilde u_1^{k-1} + 2k\,, \qquad\mbox{and}\qquad f_{\tilde u_2} - g_{\tilde u_2}  = 4k \tilde u_2 . 
\end{align*}
It remains to compute the volatility of the stochastic process one gets in the limit. Recalling $JVJ^T$ from~\eqref{eq:matrix-tensor-JVJ^T} and noticing that the rescaling $J$ to $\tilde J$ has now multiplied all four of its entries by $n$, we find that in the new coordinates, 
\begin{align}\label{eq:matrix-tensor-JVJ^T-doubly-rescaled}
	\tilde J V  \tilde J^T = \begin{pmatrix} 4 k(k-1)\tilde u_1^2 R^{2k-4} + 4 k n R^{2k-2} & 4 k (k-1) n^{1/2} \tilde u_1(R^2 - m)R^{2k-4} \\ 4 k (k-1) n^{1/2} \tilde u_1 (R^2 - m)R^{2k-4}  & 4 k(k-1) n (R^2 - m)^2 R^{2k-4}	\end{pmatrix}\,.
\end{align} 
Multiplying by $\delta = 1/n$ and taking the limit as $n\to\infty$, the two entries of this matrix that survive are $\Sigma_{11}$ and $\Sigma_{22}$, where $
	\Sigma_{11} = 4k$ and $\Sigma_{22} = 4k(k-1)$. 
All in all, we obtain \eqref{eq:matrix-tensor-PCA-double-diffusive-limit}.

\section{Proofs for the binary Gaussian mixture model}\label{sec:binary-gmm-proofs}

Recall the cross-entropy loss for the binary GMM with SGD from~\eqref{eq:Loss-bgmm}, and recall the set of summary statistics $\mathbf u_n$ from~\eqref{eq:u-bgmm}. 
 
\begin{lem}\label{lem:summary-stats-bGMM}
The distribution of $L((v,W))$ depends only on $\bu_n$ from~\eqref{eq:u-bgmm}.
In particular, we have that $\Phi(x)=\phi(\bu_n)$ for some
$\phi$. Furthermore, $\bu_n$ satisfy the bounds in item (1) of Definition~\ref{defn:localizable} if $E_K$ is the ball of radius $K$ in $\mathbb R^{2N+2}$.
\end{lem}

\begin{proof}
Let $X_\mu \sim \mathcal N(\mu,I/\lambda)$ and $X_{-\mu} \sim \mathcal N(-\mu, I/\lambda)$. Then, notice that 
\begin{align*}
L((v,W)) \stackrel{d}= \begin{cases} -v\cdot g(WX_\mu)+\log(1+e^{v\cdot g(WX_\mu)})+p(v,W) & \mbox{w. prob. $1/2$}\\ 
		\log(1+e^{v\cdot g(- W X_{\mu})})+p(v,W) & \mbox{w. prob. $1/2$} \end{cases}\,.
\end{align*}
Next, notice that as a vector, $
(W_1 X_\mu, W_2 X_\mu)$ is distributed as $(m_1 + Z_{1,\mu} m_1 + Z_{1,\perp}, m_2 + Z_{2,\mu} m_2 +  Z_{2,\perp})$,
where $Z_{1,\mu},Z_{2,\mu}$ are i.i.d.\ $\cN(0,\lambda^{-1})$, and $Z_{1,\perp},Z_{2,\perp}$ are jointly Gaussian with means zero and covariance 
\begin{align}\label{eq:Z-perp-covariance}
\lambda^{-1} \left[\begin{array}{cc}
R_{11}^{\perp} & R_{12}^\perp\\
R_{12}^\perp & R_{22}^{\perp}
\end{array}\right]
\end{align}
Similarly, the distribution of $WX_{-\mu}$ also only depends on $(m_i, R_{ij}^{\perp})_{i,j}$. Finally, 
\begin{align*}
	p(v,W) = \frac{\alpha}{2} \big(v_1^2 + v_2^2 +m_1^2 + R_{11}^\perp + m_2^2 + R_{22}^{\perp}\big)
\end{align*}
Therefore, at any point $(v,W)$, the law of $L((v,W))$, and thus $\Phi$, is simply a function of $\bu_n(v,W)$. 
To see that the summary statistics satisfy the bounds of item (1) in Definition~\ref{defn:localizable}, write $\nabla = (\partial_{v_1},\partial_{v_2},\nabla_{W_1},\nabla_{W_2})$. Then  
\begin{align}\label{eq:bgmm-J}
J = (\nabla u_\ell)_\ell = 
\left[\begin{array}{ccccccc}
1 & 0& 0& 0& 0& 0& 0\\
0 & 1 & 0& 0& 0& 0& 0\\
0& 0& \mu& 0& W_2^\perp & 2W_1^\perp& 0\\
0& 0& 0& \mu& W_1^\perp & 0& 2W_2^\perp
\end{array}\right]^{\mathsf{T}}
\end{align}
For the higher derivatives, evidently we only have second derivatives in the last 3 variables
each of which is given by a block diagonal matrix where only one block is non-zero and is given by an identity matrix.
The third derivatives of all elements of $\bu_n$ are zero.
\end{proof}

We can now express the loss, the population loss, and their respective derivatives and they (their laws at a fixed point) will evidently only depend on the summary statistics. One arrives at the  following expressions for $\nabla L$ by direct calculation from~\eqref{eq:Loss-bgmm}. 
\begin{align}
\nabla_{v_{i}}L&=(W_{i}\cdot X)\mathbf{1}_{W_i \cdot X\ge 0}\big(-y+\sigma(v\cdot g(WX)\big)+\alpha v_{i} \label{eq:bgmm-nabla-v-L}\\
\nabla_{W_{i}}L&=v_{i} X \mathbf{1}_{W_{i}\cdot X\ge0}\big(-y+ \sigma(v\cdot g(WX)) \big)+\alpha W_{i}\label{eq:bgmm-nabla-W-L}
\end{align}
In what follows, for an arbitrary vector $w\in \mathbb R^N$, we use the notation 
\begin{align}\label{eq:bgmm-A}
	\mathbf{A}_{i}  = \mathbb E\big[ X \mathbf{1}_{W_i\cdot X\ge 0} \big(-y+\sigma(v\cdot g(WX)\big)\big]
\end{align}
(Notice that if $w \in \{\mu, W_i, W_i^\perp\}$, then $\mathbf A_i \cdot w$ is only a function of $\bu_n$ by the same reasoning as used in Lemma~\ref{lem:summary-stats-bGMM}.)
Then, we can also easily express 
\begin{align}
	\nabla_{v_{i}} \Phi = W_i \cdot \mathbf{A}_i+\alpha v_{i} \qquad \nabla_{W_{i}} \Phi =v_{i} \mathbf A_i+\alpha W_{i}\label{eq:bgmm-nabla-W-Phi}
\end{align}
and for $H=L-\Phi$,  
\begin{align}
	\nabla_{v_{i}} H &= W_{i}\cdot \Big(X\mathbf{1}_{W_i \cdot X\ge 0}\big(-y+\sigma(v\cdot g(WX)\big) - \mathbf{A}_i\Big)\,,\label{eq:bgmm-nabla-v-H}\\
\nabla_{W_{i}} H &=v_{i} \Big(X\mathbf{1}_{W_i \cdot X\ge 0}\big(-y+\sigma(v\cdot g(WX)\big) - \mathbf{A}_i\Big)\,. \label{eq:bgmm-nabla-W-H}
\end{align}
Finally, the matrix $V$ can be expressed as follows: 
\begin{align}\label{eq:bgmm-V}
	V_{v_i,v_j} & = \mathbb E\big[ (W_i\cdot X)(W_j\cdot X)\mathbf 1_{W_i\cdot X\ge 0} \mathbf 1_{W_j\cdot X\ge 0} (-y+\sigma(v\cdot g(WX)))^2\big] - (W_i\cdot \mathbf{A}_i)(W_j\cdot \mathbf{A}_j) \nonumber \\
		V_{v_i,W_j} & = v_j \mathbb E\big[(W_i\cdot X)X \mathbf 1_{W_i\cdot X\ge 0} \mathbf 1_{W_j\cdot X\ge 0}(-y+\sigma(v\cdot g(WX)))^2 \big] - v_j (W_i\cdot \mathbf{A}_i) \mathbf A_j \nonumber \\ 
	V_{W_i,W_j} & = v_i v_j \mathbb E\big[ X^{\otimes 2} \mathbf 1_{W_i\cdot X\ge 0} \mathbf 1_{W_j\cdot X\ge 0}(-y+\sigma(v\cdot g(WX)))^2 \big] - v_i v_j \mathbf A_i\otimes\mathbf A_j\,.
\end{align}

Let us conclude this subsection with the following simple preliminary bounds that will be useful towards establishing the conditions of $\delta_n$-localizability from Definition~\ref{defn:localizable}, and the promised limiting equations. {The proofs of these are straightforward using Gaussianity and are provided in Section~\ref{sec:technical-lemmas} for completeness.}
 
 \begin{lem}\label{lem:bgmm-A-bound}
 Fix $w\in \mathbb R^n$. We have $
 	\mathbb E[|X \cdot w|^8] \lesssim (w\cdot \mu)^8 +  \|w\|^8 \lambda^{-4}$ and $\|\mathbf{A}_i\| \le C(\bu_n)$. 
 \end{lem}

\begin{lem}\label{lem:large-lambda-est}
For each $i$, for every $R_{ii}^\perp<\infty$ and every $m_{i}>0$, we have 
\begin{align}\label{eq:large-lambda-prob}
\lim_{\lambda\to\infty}\mathbb{P}\big(W_{i}\cdot X_{\mu}<0)=0\,.
\end{align}
For every $v_{i},R_{ij}^\perp$ and $m_{i}\ne 0$ for $i,j = 1,2$, we have 
\begin{align}\label{eq:large-lambda-sigmoid}
\lim_{\lambda\to\infty}\mathbb{E}\big[\big|\sigma(v\cdot g(WX_{\mu}))-\sigma(v\cdot g(m))\big|\big]=0\,.
\end{align}
\end{lem}

 \begin{fact}\label{fact:moment-bound}
 Fix $\mu\in S^{N-1}(1)$, and let $g(x)= x\vee 0$ and $X_\mu \sim \mathcal N(\mu,I/\lambda)$.
There is a function $C:\mathbb R^2 \to\mathbb R_+$ such that for all $\lambda>0$, $\theta\in \R$,
and $(v_i,W_i)\in \R\times \R^N$,
\[
\E[\exp(\theta v_i g(W_i\cdot X_\mu))] \leq \exp\big(\theta v_i m_i + \tfrac{1}{2\lambda } \theta^2 v_i^2 R_{ii}^\perp\big)\,.
\]
 \end{fact}

\subsection{Verifying the conditions of Theorem~\ref{thm:main} for fixed $\lambda$}
Throughout this section we will take $\mu=e_1$. By rotational invariance of the problem, 
this is without loss of generality, and only simplifies certain expressions. The $\delta_n$-localizability can be seen by application of the moment bounds listed above.

\begin{lem}\label{lem:bgmm-localizable}
For $\delta_n=O(1/N)$ and any fixed $\lambda$, 
the 2-layer GMM with observables $\bu_n$ is $\delta_n$-localizable for $E_K$ being balls of radius $K$ about the origin in $\mathbb R^7$.
\end{lem}

\begin{proof}
The condition on $\bu_n$ was satisfied per Lemma~\ref{lem:summary-stats-bGMM}. 
Recalling $\nabla \Phi$ from~\eqref{eq:bgmm-nabla-W-Phi}, one can verify that the norm of each of the four terms in $\nabla \Phi$ is individually bounded, using the Cauchy--Schwarz inequality together with the bound of Lemma~\ref{lem:bgmm-A-bound} on $\|\mathbf{A}_i\|$.  

Next, consider bounding $\mathbb{E}[\|\nabla H\|^{8}]$ by $\sum_{i=1,2}\mathbb{E}[|\nabla_{v_{i}}H|^{8}]+\mathbb{E}[\|\nabla_{W_{i}}H\|^{8}]$, 
and recall the expressions for $\nabla H$ from~\eqref{eq:bgmm-nabla-v-H}--\eqref{eq:bgmm-nabla-W-H}. 
Using the trivial bound $|\sigma(x)|\le 1$, and the inequality $(a+b)^8\le C(a^8 + b^8)$, for $i\in \{1,2\}$, the first term is at most $C ( \mathbb E[ |X\cdot W_i|^8] + \|W_i\|^8 \|\mathbf{A}_i\|^8)$
which is bounded by a constant depending continuously on $\bu_n$ per Lemma~\ref{lem:bgmm-A-bound}. If we let $Z$ be a standard Gaussian, the quantity $\mathbb{E}[\|\nabla_{W_{i}}H\|^{8}]$ is controlled by  
\begin{align*}
C\Big(v_i^8 \mathbb{E}\Big[\|X\mathbf{1}_{W_{i}\cdot X\ge0}\sigma(-v\cdot g(WX))\|^{8}\Big] + v_i^8 \|\mathbf{A}_i\|^8 \Big)
\le C |v_{i}|^{8} \Big(1+\frac{\E\norm{Z}^8}{\lambda^{4}}\Big)\,.
\end{align*}
Using the well-known bound that $\mathbb E[\|Z\|^8]\le N^{4}$, and the fact that $\delta = O(1/N)$, we see that this is at most $C\delta^{-4}$. 
We next verify the claimed bound that 
\begin{align}\label{eq:bgmm-localizable-item-3}
\delta_{n}^{2}\sup_{i}\sup_{x\in \bu_n^{-1}(E_{K})}\mathbb{E}[\langle\nabla H,\nabla u_{i}\rangle^{4}] & \le C(K)\,.
\end{align}
When $u_i$ is $v_i$, this is simply a fourth moment bound on $\nabla_{v_i} H$, which follows from the $8$'th moment by Jensen's inequality. When $u_i$ is $m_i$, or $R_{ij}^\perp$, the bound follows from
\begin{align*}
	\mathbb E[\langle \nabla_{W_i} H, w\rangle^4] \le  C |v_i|^4 \big(\mathbb E[|X\cdot w|^4] + \|w\|^4 \|\mathbf{A}_i\|^4\big)\,,
\end{align*} 
for choices of $w$ being either $\mu$ in which case $\|w\|=1$ or $W_i^\perp$ in which case $\|w\|=R_{ii}^\perp$. For each $K$, this is at most some constant $C(K)$ using the two bounds of Lemma~\ref{lem:bgmm-A-bound}. 

Finally, consider the quantity $\mathbb E[\langle \nabla^2 u, \nabla H\otimes \nabla H - V\rangle^2]$. 
This is only non-zero for $u \in \{R_{ij}^\perp\}$ for which $\nabla^2 u$ is a block-identity matrix, having operator norm at most $2$ in all cases. Therefore, this quantity is at most $4\mathbb E[\|\nabla H\|^4]$ which is at most $N^2$ by the above proved second item in the definition of localizability. This is therefore $O(\delta_n^{-2}) = o(\delta_n^{-3})$ as needed. 
\end{proof}

\begin{proof}[\textbf{\emph{Proof of Proposition~\ref{prop:bgmm-ballistic}}}]
The convergence of the population drift to $\mathbf{f}$ from Proposition~\ref{prop:bgmm-ballistic} follows by taking the inner products of  $\nabla L$ from~\eqref{eq:bgmm-nabla-W-Phi} with the rows of $J$ from~\eqref{eq:bgmm-J}, and noticing that $\mathbf{A}_i^\mu$ from~\eqref{eq:bgmm-A-B} is exactly $\mathbf{A}_i\cdot \mu$ and $\mathbf{A}_{ij}^\perp$ from~\eqref{eq:bgmm-A-B} is exactly $\mathbf{A}_i\cdot W_{j}^\perp$. 

Next consider the convergence of the correctors to the claimed $\mathbf{g}$. 
The variables $u\in\{v_{1},v_{2},m_{1},m_{2}\}$ are linear so $\mathcal{L}_{n} u=0$ and for these, $\mathbf{g}_{u}=0$. For $u= R_{ij}^{\perp}$ for $i,j\in \{1,2\}$, the relevant entries in $V$ are those corresponding to $W_i^\perp$ and $W_j^\perp$. For ease of notation, in what follows let $\pi = \sigma(v\cdot g(WX))$. 

For ease of calculation taking $\mu = e_1$, we have $
\cL_n R_{ij}^\perp = \sum_{k\ne 1} V_{W_{ik},W_{jk}}$, 
which by~\eqref{eq:bgmm-V}, and the choice of $\delta_n = c_\delta/N$, is given by 
\begin{align}\label{eq:bgmm-corrector-calc}
	\delta_n \cL_n R_{ij}^\perp & = \frac{c_\delta}{N} \sum_{k\ne 1} v_i v_j \Big(\mathbb E\big[ (X\cdot e_k)^2 \mathbf 1_{W_i \cdot X\ge 0}\mathbf{1}_{W_j\cdot X\ge 0} (-y+\pi)^2\big] -  (\mathbf A_i \cdot e_k )(\mathbf A_j \cdot e_k)\Big)\nonumber \\ 
	& = \frac{c_\delta}{N} v_i v_j \Big( \mathbb E \big[ \|X^\perp\|^2 \mathbf 1_{W_i \cdot X\ge 0} \mathbf 1_{W_j\cdot X\ge 0} (-y+\pi)^2 \big] - \langle \mathbf{A}_i - \mathbf{A}_i^\mu \mu, \mathbf{A}_j- \mathbf{A}_j^\mu \mu\rangle \Big)\,.
\end{align}
Consider the two terms separately. First, rewrite $\frac{1}{N}  \mathbb E [ \|X^\perp\|^2 \mathbf 1_{W_i \cdot X\ge 0} \mathbf 1_{W_j\cdot X\ge 0} (-y+\pi)^2 ]$ as 
\begin{align*}
	 \mathbb E\big[\big(\tfrac{1}{N} \|X^\perp\|^2 -\lambda^{-1}\big) \mathbf 1_{W_i \cdot X\ge 0} \mathbf 1_{W_j\cdot X\ge 0} (-y+\pi)^2\big] + \lambda^{-1} \mathbf B_{ij}\,.
\end{align*}
Of course the second term is exactly what we want to be $g_{u}$, so we will show the first term here goes to zero. By Cauchy--Schwarz, if $Z\sim \cN(0,I-e_1^{\otimes 2})$, the first term above is at most $
	\lambda^{-1} \mathbb E[(\frac{\|Z\|^2}{N}  - 1)^2]^{1/2} \le \frac{2}{\lambda \sqrt{N}}$,
where we used the fact that for a standard Gaussian, $g\sim \cN(0,1)$, we have $\mathbb E[(g^2 -1)^2]= 2$. It remains to show the inner product term in~\eqref{eq:bgmm-corrector-calc} goes to zero as $n\to\infty$. For this term, rewrite 
\begin{align*}
\frac{1}{N} \langle \mathbf{A}_i - \mathbf{A}_i^\mu \mu, \mathbf{A}_j- \mathbf{A}_j^\mu \mu\rangle = \frac{1}{N} \mathbb E\big[(X_1^\perp\cdot X_2^\perp) \mathbf 1_{W_i \cdot X_1\ge 0} \mathbf 1_{W_j\cdot X_2\ge 0} (-y+\pi_1)(-y+\pi_2)\big]\,,
\end{align*}
where $X_1,X_2$ are i.i.d.\ copies of $X$, and $\pi_1,\pi_2$ are the corresponding $\sigma(v\cdot g(WX_1))$ and $\sigma(v\cdot g(WX_2))$. By Cauchy--Schwarz, if $Z,Z'$ are i.i.d.\ $\cN(0,I-e_1^{\otimes 2})$, this is at most $
\frac{1}{\lambda N} \mathbb E [ (Z\cdot Z')^2]^{1/2}\le \frac{1}{\lambda\sqrt{N}}$. 
This term therefore also vanishes as $n\to\infty$, yielding the desired limit for the corrector, 
\begin{align*}
	g_{R_{ij}^{\perp}} = \frac{c_\delta v_i v_j}{\lambda} \mathbb E\big[\mathbf 1_{W_i \cdot X\ge 0} \mathbf 1_{W_j\cdot X\ge 0} (-y+\pi)^2\big] = \frac{c_\delta v_i v_j}{\lambda} \mathbf{B}_{ij}\,.
\end{align*} 
which we emphasize is only a function of $\bu_n$. 
We lastly need to show that the diffusion matrix $\Sigma_n$ goes to zero as $n\to\infty$ when $\delta_n = O(1/n)$. This is straightforward to see by considering any element of $JVJ^T$ and using Cauchy--Schwarz together with the two bounds of Lemma~\ref{lem:bgmm-A-bound} to bound it in absolute value by some $C(K)$ independent of $n$. Then when multiplying by any $\delta_n = o(1)$, this entire matrix will evidently vanish. 
\end{proof}

\subsection{The small-noise limit of the effective dynamics}
{One can now take a $\lambda\to\infty$ limit to arrive at the ODE system of Proposition~\ref{prop:bgmm-ballistic-noiseless}. 

\begin{proof}[\textbf{\emph{Proof of Proposition~\ref{prop:bgmm-ballistic-noiseless}}}]
We begin with considering $\lim_{\lambda\to\infty}\mathbf{A}_{i}^\mu$:
its limiting value will depend on the signs of both $m_{1}$ and $m_{2}$.
We can express $\mathbf{A}_{i}^\mu$ from~\eqref{eq:bgmm-A-B} as 
\begin{align*}
\mathbb{E}[(X\cdot\mu) \mathbf{1}_{W_{i}\cdot X\ge0}(-y+\sigma(v\cdot g(WX)))] & =  \frac{1}{2}\mathbb{E}\Big[(X_{\mu}\cdot\mu)\mathbf{1}_{W_{i}\cdot X_{\mu}\ge0}(-1+\sigma(v\cdot g(WX_{\mu})))\Big] \\ 
& \qquad +\frac{1}{2}\mathbb{E}\Big[(-X_{\mu}\cdot\mu)\mathbf{1}_{W_{i}\cdot X_{\mu}\le0}\sigma(v\cdot g(-WX_{\mu}))\Big]\,.
\end{align*}
We claim that the two terms on the right-hand side converge
to $-\frac{1}{2}\mathbf{1}_{m_{i}>0}\sigma(-v\cdot g(m))$ and $-\frac{1}{2}\mathbf{1}_{m_{i}<0}\sigma(v\cdot g(-m))$ respectively.
This follows by e.g., writing the difference as
\begin{align}\label{eq:A-i-dot-mu-noiseless}
\mathbb{E}\Big[(X_{\mu}\cdot\mu)\mathbf{1}_{W_{i}\cdot X_\mu \ge0}\sigma(-v\cdot g(WX_{\mu}))\Big] & -\mathbf{1}_{m_{i}\ge0}\sigma(-v\cdot g(m)) \\ 
&=  \mathbb{E}\Big[(X_{\mu}\cdot\mu-1)\mathbf{1}_{W_{i}\cdot X_{\mu}\ge0}\sigma(-v\cdot g(WX_{\mu}))\Big] \nonumber \\ 
& \qquad +\mathbb{E}\Big[(\mathbf{1}_{W_{i}\cdot X_{\mu}\ge0}-\mathbf{1}_{m_{i}\ge0})\sigma(-v\cdot g(WX_{\mu}))\Big] \nonumber \\
&\qquad+\mathbf{1}_{m_{i}\ge0}\mathbb{E}\Big[\sigma(-v\cdot g(WX_{\mu}))-\sigma(-v\cdot g(m))\Big]\,. \nonumber
\end{align}
Call these three terms $I,II$, and $III$. 
For $I$, we use the fact that $\mathbb E[|X_\mu \cdot \mu -1|]$ goes to zero as $\lambda \to\infty$;  $II$ is evidently bounded by $\mathbb{P}(W_{i}\cdot X_{\mu}<0)$ when $m_i>0$ or its symmetric counterpart when $m_i<0$---both vanishing as $\lambda \to\infty$ per~\eqref{eq:large-lambda-prob} in Lemma~\ref{lem:large-lambda-est}; finally, $III$ goes to zero as $\lambda\to\infty$ by~\eqref{eq:large-lambda-sigmoid} in Lemma~\ref{lem:large-lambda-est}.

Putting the above together, we find that 
\begin{align*}
\lim_{\lambda\to\infty}\mathbf{A}_{i}^\mu = & -\frac{1}{2}\mathbf{1}_{m_{i}>0}\sigma(-v\cdot g(m))-\frac{1}{2}\mathbf{1}_{m_{i}<0}\sigma(v\cdot g(-m))\,,
\end{align*}
at which point, we see that if $m_{1},m_{2}\ge0$, this becomes $\frac{1}{2}\sigma(-v\cdot m)$,
as it is if $m_{1},m_{2}\le0$. If $m_{1}\ge0$ and $m_{2}\le0$,
then you get $\lim_{\lambda} \mathbf{A}_{1}^\mu =-\frac{1}{2}\sigma(-v_{1}m_{1})$ and
$\lim_{\lambda} \mathbf{A}_{2}^\mu =-\frac{1}{2}\sigma(-v_{2}m_{2})$ and likewise if
$m_{1}\le0$ and $m_{2}\ge0$. 

Next consider the limit as $\lambda\to\infty$ of $\mathbf{A}_{ij}^\perp$ from~\eqref{eq:bgmm-A-B},
which we claim converges to $0$. Write 
\begin{align}\label{eq:A-ij-perp-noiseless}
\mathbf{A}_{ij}^\perp & =  -\frac{1}{2}\mathbb{E}\Big[(X_{\mu}\cdot W_{j}^{\perp})\mathbf{1}_{W_{i}\cdot X\ge0}\sigma(-v\cdot g(WX_{\mu}))\Big] \\ 
& \quad -\frac{1}{2}\mathbb{E}\Big[(X_{\mu}\cdot W_{j}^{\perp})\mathbf{1}_{W_{i}\cdot X_{\mu}<0}\sigma(v\cdot g(-WX_{\mu}))\Big]\,. \nonumber
\end{align}
These two terms are bounded similarly. The absolute value of the first of these is bounded by $(1/2)\mathbb E[|X_\mu \cdot W_j^\perp|]$ which is at most $(1/2)\sqrt{R_{jj}^\perp} \lambda^{-1/2}$ by~\eqref{lem:bgmm-A-bound}. The second is analogously bounded. These evidently go to zero as $\lambda\to\infty$. 

Finally, since $|\mathbf{B}_{ij}|\le 1$,  the quantity $g_{R_{ij}^\perp} = c_\delta \frac{v_i v_j}{\lambda} \mathbf{B}_{ij}$ evidently goes to zero as $\lambda\to\infty$. 
\end{proof}

\begin{rem}
The above argument used $m_i\ne 0$ for the limit of $\mathbf{A}_i^\mu$. If one considers the cases when $m_{i}=0$, the limiting drifts still apply. For this, it suffices to show that if $m_i = 0$, then $\mathbf{A}_i^\mu$ converges to zero. Without loss of generality, suppose $m_1 =0$ and consider  
\begin{align*}
	\mathbf{A}_1 \cdot \mu = \mathbb{E}\big[Z_{1,\mu} \mathbf{1}_{Z_{1,\perp}\ge 0} \sigma(-v\cdot g(Z_{1,\perp} ,m_2 Z_{2,\mu} + Z_{2,\perp}))\big]\,.
\end{align*}
This is zero independently of $\lambda$ by independence of $Z_{1,\mu}$ from the other Gaussians in the expectation. 
\end{rem}

Evidently, every fixed point must have $R_{ij}^{\perp} = 0$. Furthermore, if we let $u_i = v_i - m_i$, then 
\begin{align*}
	\dot u_i = \begin{cases}- \frac{u_i}{2} \sigma(-v\cdot m) - \alpha u_i & m_1m_2>0 \\ - \frac{u_i}{2} \sigma(-v_i m_i) - \alpha u_i  & \mbox{else} \end{cases}\,,
\end{align*}
and therefore every fixed point of the ODE system must have $u_i =0$, which is to say $v_i = m_i$. Therefore, it suffices to characterize the fixed points in terms of $(v_1,v_2)$ as claimed. This reduces to $v_i \sigma(-\|v\|^2) = 2 \alpha v_i v_1 v_2 >0$ if $v_1v_2>0$ and $v_i\sigma(-v_i^2) = 2\alpha v_i$ otherwise.  
Observe first that the point $(v_1,v_2) = (0,0)$ is a fixed point of this system. If $(v_1,v_2)\ne 0$, then dividing out by $v_i$, the above reduces to $\sigma(-\|v\|^2) = 2\alpha$ if $v_1v_2>0$ and $\sigma(-v_i^2) = 2\alpha$ otherwise. 
Recalling that $C_\alpha = -\logit(2\alpha) = \log(1-2\alpha) - \log(2\alpha)$ we obtain the claimed set of fixed points by inverting these equations (they only have a solution if $\alpha<1/4$).

In order to study the stability of the various fixed points, notice first that the ODE system of Proposition~\ref{prop:bgmm-ballistic-noiseless} is a gradient system for the $\lambda =\infty$ population loss, 
\begin{align*}
	\Phi(v,m) = \frac{1}{2} \Big( \log ( 1+e^{ - v\cdot g(m)}) + \log ( 1+ e^{ v\cdot g(-m)})\Big) + \frac{\alpha}{2} \sum_{i=1,2}(v_i^2 + m_i^2 +R_{ii}^\perp)\,.
\end{align*}
Since it is a gradient system, with only the specified fixed points, the stability of a fixed point can be deduced by showing it is the minimizer of $\Phi$. In particular, the values of $\Phi$ at its critical points are given by $\Phi_0 =  \log 2$ at $v_1 = v_2 =0$, $\Phi_{+} = \frac{1}{2}( \log 2 + \log( 1+e^{ - C_\alpha}) + \alpha C_\alpha$ when $v_1v_2>0$, and $\Phi_- = \log(1+e^{ - C_\alpha}) + 2\alpha C_\alpha$ when $v_1 v_2<0$. It is a simple calculus exercise to show that the smallest of these is $\Phi_0$ when $\alpha>1/4$ and $\Phi_-$ when $\alpha<1/4$. 

To show that each of the other critical points are all unstable, one can find a direction along which the dynamical system is locally repelled from it. For instance, we will show that the ring of fixed points with $v_i = m_i$ and $R_{ij}^\perp = 0$ with $v_1v_2 \le 0$ is unstable, by showing a repelling direction arbitrarily close to the point $v_1 = -\sqrt{C_\alpha}$, $v_2 = 0$. If $v_1 = - \sqrt{C_\alpha}$ and $v_2 = \epsilon>0$, then $\dot v_2$ there reduces to $\epsilon(\frac{\sigma(-\epsilon^2)}{2} - \alpha)$, and as long as $\alpha<1/4$, there exists $\epsilon>0$ such that $\sigma(-\epsilon^2) >2\alpha$ so $\dot v_2>0$ for all $\epsilon$ small enough.

\subsection{Rescaled effective dynamics around unstable fixed points}
In this section, we consider scaling limits of the rescaled effective dynamics in their noiseless limit, where the rescaling is about the unstable set of fixed points given by the quarter circle $v_1^2 + v_2^2 = C_\alpha$ per item (2) of Proposition~\ref{prop:bgmm-ballistic-noiseless}. Let $\delta_n = \sfrac{c_\delta}{N}$, and fix $(a_1,a_2)\in \mathbb R_+^2$ with $a_1^2 + a_2^2 =C_\alpha$, and let $\bu_n$ be the variables of~\eqref{eq:u-bgmm} with $v_i,m_i$ replaced by $\tilde v_i = \sqrt{N} (v_i - a_i)$ and $\tilde m_i = \sqrt{N} (m_i - a_i)$. 
 
\begin{proof}[\textbf{\emph{Proof of Proposition~\ref{prop:bgmm-diffusive-noiseless}}}]
We start by considering the drift process for these rescaled variables. Notice that the rescaling induces the transformation $\tilde J$ multiplying $J$ by $\sqrt N$ in its entries corresponding to $v_i,m_i$. The fact that the rescaled variables satisfy the conditions of Theorem~\ref{thm:main} follows as in Lemma~\ref{lem:bgmm-localizable} with the only distinction arising in the bound on~\eqref{eq:bgmm-localizable-item-3}, where previously we did not use the $\delta_n^2$ factor---in the new coordinates, the factor of $\sqrt N$ raised to the fourth power is cancelled out by $\delta_n^2$ as long as $\delta_n = O(1/N)$.  

For the population drift of the new variables, if the variables $\tilde v_i,\tilde m_i$ are in a ball of radius $K$ in $\mathbb R^4$ (which we take to be our $E_K$), the signs of $m_i$ agree, and therefore 
\begin{align*}
	f_{\tilde v_i} =-  \sqrt{N} \frac{v_i}{2} \sigma(-v\cdot m) +  \alpha \sqrt{N} m_i  \qquad\mbox{and}\qquad f_{\tilde m_i} =   - \sqrt{N} \frac{m_i}{2} \sigma(-v\cdot m) +  \alpha \sqrt{N} v_i\,.
\end{align*}
We wish to claim that these expressions have consistent limits when $\tilde v_i,\tilde m_i$ are localized to $E_K$ for fixed $K$. notice that in $m_i = a_i +  N^{-1/2} \tilde m_i$ and $v_i = a_i + N^{-1/2} \tilde v_i$, and using $\sum a_j^2 = C_\alpha$, 
\begin{align*}
	v\cdot m = C_\alpha + N^{-1/2} \sum_{j=1,2}  a_j (\tilde v_j + \tilde m_j)+ O(1/n)\,.
\end{align*}
Now Taylor expanding the sigmoid function, and using the definition of $C_\alpha$, we get 
\begin{align*}
	\sigma (-v\cdot m) & = \sigma(-C_\alpha) + (v\cdot m - C_\alpha) \sigma(-C_\alpha) (1-\sigma(-C_\alpha)) + O(n^{-1})  \\ 
	& = 2\alpha + N^{-1/2}a_j \Big(\sum_{j=1,2} \big( \tilde v_j + \tilde m_j\big) (2\alpha)(1-2\alpha)\Big) + O(n^{-1})\,.
\end{align*}
Plugging these into the earlier expressions for $f_{\tilde v_i}$, we see that 
\begin{align*}
	f_{\tilde v_i}  & =   - \frac{ N^{1/2}a_i +  \tilde m_i}{2} \Big(2\alpha + \frac{a_j}{N^{1/2}} \sum_{j=1,2} \big(\tilde v_j + \tilde m_j\big) (2\alpha)(1-2\alpha) + O\Big(\frac{1}{n}\Big)\Big)  + \alpha (n^{1/2} a_i  +  \tilde v_i) \\ 
	& =  - \alpha \tilde m_i + \alpha \tilde v_i - a_i (\alpha -2\alpha^2) \sum_{j=1,2} a_j (\tilde v_j + \tilde m_j) + O(n^{-1/2})\,.
\end{align*}
Taking the limit as $n\to\infty$, this yields exactly the population drift claimed for the $\tilde v_i$ variable. The calculation for $f_{\tilde m_i}$ is analogous, and the equations for $R_{ij}^\perp$ are evidently unchanged by the transformation of $v_i,m_i$ to $\tilde v_i,\tilde m_i$. Furthermore, these variables are still linear so no corrector is introduced. 

We now turn to computing the limiting diffusion matrix $\Sigma$ in the new variables $\tilde v_i,\tilde m_i$. We first use the following expression for the matrix $V$ when $\lambda = \infty$, by taking the $\lambda=\infty$ in~\eqref{eq:bgmm-V}:
\begin{align*}
V_{v_i,v_j} & =  \frac{m_{i}m_{j}}{4}\cdot\begin{cases}
\sigma(-v\cdot m)^{2} & m_1m_2>0\\
\sigma(-v_i m_i)\sigma(-v_jm_j) & \mbox{else}
\end{cases}\,, 
\end{align*}
with similar expressions for $V_{v_i,W_j}$ and $V_{W_i,W_j}$. 
Rewriting in $\tilde v$ and $\tilde m$, we see that in $E_K$,  
\begin{align*}
	V_{v_i, v_j}  = \alpha^2 a_i a_j + O(n^{-1/2})\,, \qquad V_{v_i,W_j} = \mu (\alpha^2 a_i a_j  + O(n^{-1/2}))\,,  
\end{align*}
\begin{align*}
V_{W_i,W_j} = \mu^{\otimes 2} (\alpha^2 a_i a_j  + O(n^{-1/2}))\,.
\end{align*}
Now multiplying this on both sides by $\tilde J$, for the $\tilde \bu_n$ variables, the two factors of $\sqrt{N}$ from $\tilde J$ cancel out with the choice of $\delta_n = 1/N$, and in the $n\to\infty$ limit, leave $\tilde{\Sigma}_{v_{i}v_{j}}=\tilde{\Sigma}_{m_{i}m_{j}}=\tilde{\Sigma}_{v_{i}m_{j}}= \alpha^{2}a_{i}a_{j}$
as claimed. 
\end{proof}

\section{Proofs for the XOR Gaussian mixture model}\label{sec:xor-proofs}
Fix two orthogonal vectors $\mu,\nu \in \mathbb R^N$ and recall the cross-entropy loss with penalty $p(v,W) = \frac{\alpha}{2} (\|v\|^2 + \|W\|^2)$. For the XOR GMM with SGD, the cross-entropy loss is given by 
\begin{align}\label{eq:XOR-GMM-loss}
	L(v,W)= -y v\cdot g(WX) +\log\big(1+e^{v\cdot g(WX)}\big)+p(v,W)
\end{align}
where if the class label $y = 1$, then $X$ is a symmetric binary Gaussian mixture with means $\pm \mu$, and if $y=0$, then $X$ is a symmetric Gaussian mixture with means $\pm \nu$. This has the same form as the loss for the 2-layer binary GMM, and we will find many similarities in the below between them. Indeed, the only difference is in the distribution of $X$ conditionally on the class label $y$ as described, and the fact that $v$ is now in $\mathbb R^K$ and $W = (W_i)_{i=1,...,K}$ is now a $K\times N$ matrix.  In what follows we take $n = KN + K$. As such, all the formulae of~\eqref{eq:bgmm-nabla-v-L}--~\eqref{eq:bgmm-V} also hold for the XOR GMM, but with the law of $(y,X)$ now understood differently.  

\begin{rem}
We could also have added a bias at each layer, however the Bayes
classifier in this problem is an ``X'' centered at the origin so we can safely take the biases to be~0.
\end{rem}

\subsection{Summary statistics and localizability}
Recall the set of summary statistics $\bu_n$ from~\eqref{eq:XOR-GMM-summary-stats}. The next lemma shows that $\bu_n$ form a good set of summary statistics. 

\begin{lem}\label{lem:XOR-summary-stats}
	The distribution of $L((v,W))$ depends only on $\bu_n$ from~\eqref{eq:XOR-GMM-summary-stats}.
In particular, we have that $\Phi(x)=\phi(\bu_n)$ for some
$\phi$. Furthermore, $\bu_n$ satisfy the bounds in item (1) of Definition~\ref{defn:localizable} with an exhaustion by balls of $\mathbb R^{KN+K}$.
\end{lem}

\begin{proof}
Let $X_w = \cN(w,I/\lambda)$ for $w\in \{\mu,-\mu,\nu,-\nu\}$. Notice that the law of $L$ at a fixed point $(v,W) \in \mathbb R^{K+KN}$ can be written as 
\begin{align}\label{eq:XOR-loss}             
L((v,W)) \stackrel{d}= \begin{cases} -v\cdot g(WX_\mu)+\log(1+e^{v\cdot g(WX_\mu)})+p(v,W) & \mbox{w. prob. $1/4$}\\  
				-v\cdot g(WX_{-\mu})+\log(1+e^{v\cdot g(WX_{-\mu})})+p(v,W) & \mbox{w. prob. $1/4$}\\ 
		\log(1+e^{v\cdot g(W X_{\nu})})+p(v,W) & \mbox{w. prob. $1/4$} \\
		\log(1+e^{v\cdot g(W X_{-\nu})})+p(v,W) & \mbox{w. prob. $1/4$}  \end{cases}
\end{align}
Next, notice that as a vector 
\begin{align*}
	WX_\iota= (m_i +Z_{i,\iota}m_i^\iota  + Z_{i\perp})_{i=1,...,K} \qquad \mbox{for $\iota \in \{\mu,\nu\}$}\,,
\end{align*}
where $Z_{i,\iota}$ are i.i.d.\ $\cN(0,\lambda^{-1})$ and $(Z_{i\perp})$ are jointly Gaussian with covariance matrix 
\begin{align*}
	\mbox{Cov} (Z_{i\perp}, Z_{j\perp}) = \lambda^{-1} R_{ij}^{\perp}\,.
\end{align*}
Similarly, the law of $W X_{-\iota}$ depends only on $(m_i^\iota, R_{ij}^{\perp})$. Finally, 
\begin{align*}
	p(v,W) = \tfrac{\alpha}{2} \sum_{i=1,...,K} \big( v_i^2 + R_{ii}^{\perp}\big)\,.
\end{align*}
Therefore, at a fixed point $(v,W)$ the law of $L(v,W)$ is only a function of $\bu_n(v,W)$. 

To see that the summary statistics satisfy the bounds of item (1) in Definition~\ref{defn:localizable}, note that the non-zero entries of $J= (\nabla u_\ell)_{\ell}$ are as follows. 
\begin{align}\label{eq:XOR-J}
	\partial_{v_i} v_i = 1\,, \qquad \nabla_{W_i} m_i^\mu = \mu\,, \qquad \nabla_{W_i} m_i^\nu = \nu\,, \qquad \nabla_{W_i} R_{jk}^{\perp} = W_j^\perp \delta_{ij} + W_k^\perp \delta_{ik}\,,
\end{align}
where $\delta_{ij}$ is $1$ if $i=j$ and $0$ otherwise. For higher derivatives, we only have second derivatives in the $R_{jk}^\perp$ variables, each of which is given by a block diagonal matrix where only one block is non-zero and it is twice an identity matrix. Thus the operator norm of these second derivatives is $2$. The third derivatives of all elements of $\bu_n$ are zero. 
\end{proof}

In the following, let 
\begin{align*}
	\mathbf{A}_i = \mathbb E\big[X\mathbf 1_{W_i \cdot X\ge 0} \big(-y+\sigma(v\cdot g(WX))\big)\big]\,.
\end{align*}
By the same reasoning as in Lemma~\ref{lem:XOR-summary-stats}, if $w\in \{\mu,\nu, W_i, W_i^\perp\}$, then $w\cdot \mathbf{A}_i$ is only a function of $\bu_n$. We then also have the conclusions of Lemma~\ref{lem:bgmm-A-bound} for $X$ distributed according to the XOR GMM by simply decomposing it into two mixtures, and we will therefore appeal to this lemma meaning its analogue for the XOR GMM. 

\begin{lem}\label{lem:xor-localizable}
For $\delta=O(1/N)$ and any fixed $\lambda$, 
the 2-layer XOR GMM with observables $\bu_n$ is $\delta_n$-localizable for $E_K$ being balls of radius $K$ about the origin in $\mathbb R^{4K + \binom{K}{2}}$. \end{lem}

\begin{proof}
	The condition on $\bu_n$ was satisfied per Lemma~\ref{lem:XOR-summary-stats}. Recalling $\nabla \Phi$ from~\eqref{eq:bgmm-nabla-W-Phi}, one can verify that the norm of each of the four terms in $\nabla \Phi$ is individually bounded, using the Cauchy--Schwarz inequality together with the bound of Lemma~\ref{lem:bgmm-A-bound} on $\|\mathbf{A}_i\|$, naturally adapted to XOR.  The remaining estimates are also analogous to the proof of Lemma~\ref{lem:bgmm-localizable} with the analogue of Lemma~\ref{lem:bgmm-A-bound} applied. 
\end{proof}

\subsection{Effective dynamics for the XOR GMM}

\begin{proof}[Proof of Proposition~\ref{prop:XOR-effective-dynamics-finite-lambda}]
	The convergence of the population drift to $\mathbf{f}$ from Proposition~\ref{prop:bgmm-ballistic} follows by taking the inner products of  $\nabla L$ from~\eqref{eq:bgmm-nabla-W-Phi} with the rows of $J$ from~\eqref{eq:XOR-J}, and noticing that $\mathbf{A}_i^\mu$ is exactly $\mathbf{A}_i\cdot \mu$, $\mathbf{A}_i^\nu$ is exactly $\nu \cdot \mathbf{A}_i$, and $\mathbf{A}_{ij}^\perp$ is exactly $\mathbf{A}_i\cdot W_{j}^\perp$. 
	
	We next consider the population correctors. The fact that $g_{v_i} = g_{m_i^\mu} = g_{m_i^\nu} =0$ follows from the fact that the Hessians of $v_i, m_i^\mu, m_i^\nu$ are zero. For the corrector $g_{R_{ij}^\perp}$ for $1\le i\le j\le K$, the relevant entries of $V$ are those corresponding to $W_i^\perp$ and $W_j^\perp$. For ease of notation, in what follows let $\pi = \sigma(v\cdot g(WX))$. 
	
	Similar to the calculation of~\eqref{eq:bgmm-corrector-calc}, 
	\begin{align*}
		\delta_n \cL_n R_{ij}^\perp = \frac{c_\delta}{N} v_i v_j \Big( & \mathbb E\big[\|X^\perp\|^2 \mathbf 1_{W_i\cdot X\ge 0}\mathbf{1}_{W_j\cdot X\ge 0} (\pi-y)^2\big] \\
		& \,\,\, - \langle \mathbf{A}_i -\mathbf{A}_i^\mu \mu -\mathbf{A}_i^\nu \nu, \mathbf{A}_j-\mathbf{A}_j^\mu \mu -\mathbf{A}_j^\nu \nu\rangle\Big)\,.
	\end{align*}
	By the same arguments on the concentration of the norm of Gaussian vectors as used in the binary GMM case, then we deduce from this that 
	\begin{align*}
			g_{R_{ij}^{\perp}} = \frac{c_\delta v_i v_j}{\lambda} \mathbb E\big[\mathbf 1_{W_i \cdot X\ge 0} \mathbf 1_{W_j\cdot X\ge 0} (-y+\pi)^2\big] = \frac{c_\delta v_i v_j}{\lambda} \mathbf{B}_{ij}\,. 
	\end{align*}
	
	Finally, let us establish that the limiting diffusion matrix is all-zero whenever $\delta_n = o(1)$. This follows exactly as it did in the proof of Proposition~\ref{prop:bgmm-ballistic}. 
\end{proof}

\subsection{Small noise limit of the effective dynamics}

The aim of this section is to establish the following small-noise $\lambda\to\infty$ limit of the effective dynamics ODE of Proposition~\ref{prop:XOR-effective-dynamics-finite-lambda}. This will again be quite similar to the analogous proofs for the binary GMM in Section~\ref{sec:binary-gmm-proofs}, and when these similarities are clear we will omit details. 

\begin{prop}\label{prop:xor-ballistic-ode}
In the $\lambda\to\infty$ limit, the ODE  from Proposition~\ref{prop:XOR-effective-dynamics-finite-lambda} converges to 
\begin{align*}
\dot{v}_{i}& =  \frac{m_{i}^{\mu}}{4}\Big(\mathbf{1}_{m_{i}^{\mu}\ge0}\sigma(-v\cdot g(m^{\mu}))-\mathbf{1}_{m_{i}^{\mu}<0}\sigma(-v\cdot g(-m^{\mu}))\Big) \\ & \qquad -\frac{m_{i}^{\nu}}{4}\Big(\mathbf{1}_{m_{i}^{\nu}\ge0}\sigma(v\cdot g(m^{\nu}))-\mathbf{1}_{m_{i}^{\nu}<0}\sigma(v\cdot g(-m^{\nu}))\Big)-\alpha v_{i}\,,\\
\dot{m}_{i}^{\mu} & =  \frac{v_{i}}{4}\Big(\mathbf{1}_{m_{i}^{\mu}\ge0}\sigma(-v\cdot g(m^{\mu}))-\mathbf{1}_{m_{i}^{\mu}<0}\sigma(-v\cdot g(-m^{\mu}))\Big)-\alpha m_{i}^{\mu}\,,\\
\dot{m}_{i}^{\nu} & =  -\frac{v_{i}}{4}\Big(\mathbf{1}_{m_{i}^{\nu}\ge0}\sigma(-v\cdot g(m^{\nu}))-\mathbf{1}_{m_{i}^{\nu}<0}\sigma(-v\cdot g(-m^{\nu}))\Big)-\alpha m_{i}^{\nu}\,,
\end{align*}
and $\dot{R}_{ij}^{\perp} =  -2 \alpha R_{ij}^{\perp}$ 
for $1\le i\le j\le K$. 
\end{prop}

\begin{proof}
	Let us begin with convergence of $\mathbf{A}_i^\mu$. We claim that it converges to 
	\begin{align*}
		\lim_{\lambda\to\infty} \mathbf{A}_i^\mu = -\frac{1}{4} \mathbf 1_{m_i^\mu >0} \sigma(-v\cdot g(m^{\mu})) -\frac{1}{4}\mathbf{1}_{m_i^\mu<0} \sigma(v\cdot g(-m))\,.
	\end{align*} 
	In order to see this, expand 
	\begin{align*}
		\mathbf{A}_i &  = \frac{1}{4} \mathbb E\big[- X_\mu \mathbf 1_{W_i \cdot X_\mu \ge 0} (\sigma( - v\cdot g(WX_\mu)))\big] - \frac{1}{4} \mathbb E\big[ X_{-\mu} \mathbf 1_{W_i \cdot X_{-\mu} \ge 0} (\sigma( - v\cdot g(WX_{-\mu})))\big] \\
		& \qquad + \frac{1}{4} \mathbb E\big[X_\nu \mathbf 1_{W_i \cdot X_\nu \ge 0} (\sigma( v\cdot g(WX_\nu)))\big] + \frac{1}{4} \mathbb E\big[ X_{-\nu} \mathbf 1_{W_i \cdot X_{-\nu} \ge 0} (\sigma( v\cdot g(WX_{-\nu})))\big]\,.
	\end{align*}
	The point will be that when taking the inner product with $\mu$, the first two terms here contribute to the limit and the latter two vanish, while when taking the inner product with $\nu$, the first two terms vanish in the $\lambda\to\infty$ limit while the latter two contribute. 
	
	Consider e.g., the first of the four terms above, and inner product with $\mu$. In this case, consider  
\begin{align*}
\mathbb{E}\big[(X_{\mu}\cdot\mu)\mathbf{1}_{W_{i}\cdot X_\mu \ge0}\sigma(-v\cdot g(WX_{\mu}))\big] & -\mathbf{1}_{m_{i}^\mu \ge0}\sigma(-v\cdot g(m^\mu))\,,
\end{align*}
which is precisely the quantity that was exactly shown to go to zero as $\lambda \to\infty$ in~\eqref{eq:A-i-dot-mu-noiseless}. To see that the third and fourth terms above go to zero when taking their inner product with $\mu$, observe that they become
\begin{align*}
	\big|\mathbb E\big[(X_\nu \cdot \mu) \mathbf 1_{W_i\cdot X_\nu \ge 0 } \sigma(v\cdot g(WX_\nu))\big]\big|\le \mathbb E[|X_\nu \cdot \mu|]\,,
\end{align*}
which by orthogonality of $\mu$ and $\nu$ is at most $\lambda^{-1/2}$ by the reasoning of Lemma~\ref{lem:bgmm-A-bound}, therefore vanishing as $\lambda\to\infty$. Together with its analogue for $X_{-\nu}$, this implies the claim for the convergence of $\mathbf{A}_i^\mu$, as well as its analogous limit of $\mathbf{A}_i^\nu$. 

We next consider the limit as $\lambda\to\infty$ of $\mathbf{A}_{ij}^\perp$, which we claim goes to $0$. Using the expansion of $\mathbf{A}_i$ from earlier in this proof, we can consider $\mathbf{A}_{ij}^\perp = \mathbf{A}_i \cdot W_j^\perp$ as four terms having the form of the terms in~\eqref{eq:A-ij-perp-noiseless}, which were there showed to go to zero as $\lambda \to\infty$. Since $W_j^\perp$ here is orthogonal both to $\mu$ and $\nu$, the same proof applies. 

Finally, in order to see that the limit as $\lambda\to\infty$ of $g_{R_{ij}^\perp} = c_\delta \frac{v_i v_j}{\lambda} \mathbf{B}_{ij}$ is zero, which follows from the fact that $|\mathbf{B}_{ij}|\le 1$. 
\end{proof}

\begin{prop}\label{prop:fixpoints-xor}
The fixed points of the ODE system of Proposition~\ref{prop:xor-ballistic-ode} are classified as follows. If $\alpha>1/8$, then
the only fixed point is at $\mathbf{u}_{n}=\boldsymbol{0}$. 

If $0<\alpha<1/8$, then let $(I_{0},I_{\mu}^{+},I_{\mu}^{-},I_{\nu}^{+},I_{\nu}^{-})$
be any disjoint (possibly empty) subsets whose union is $\{1,...,K\}$.
Corresponding to that tuple $(I_{0},I_{\mu}^{+},I_{\mu}^{-},I_{\nu}^{+},I_{\nu}^{-})$,
is a set of fixed points that have $R_{ij}^{\perp}=0$ for all
$i,j$, and have 
\begin{enumerate}
\item $m_{i}^{\mu}=m_{i}^{\nu}=v_{i}=0$ for $i\in I_{0}$,
\item $m_{i}^{\mu}=v_{i}>0$ such that $\sum_{i\in I_{\mu}^{+}}v_{i}^{2}=\mbox{logit}(-4\alpha)$
and $m_{i}^{\nu}=0$ for all $i\in I_{\mu}^{+}$,
\item $-m_{i}^{\mu}=v_{i}>0$ such that $\sum_{i\in I_{\mu}^{-}}v_{i}^{2}=\mbox{logit}(-4\alpha)$
and $m_{i}^{\nu}=0$ for all $i\in I_{\mu}^{-}$,
\item $m_{i}^{\nu}=v_{i}<0$ such that $\sum_{i\in I_{\nu}^{+}}v_{i}^{2}=\mbox{logit}(-4\alpha)$
and $m_{i}^{\mu}=0$ for all $i\in I_{\nu}^{+}$,
\item $-m_{i}^{\nu}=v_{i}<0$ such that $\sum_{i\in I_{\nu}^{-}}v_{i}^{2}=\mbox{logit}(-4\alpha)$
and $m_{i}^{\mu}=0$ for all $i\in I_{\nu}^{-}$.
\end{enumerate}
In the $K=4$ case, these form $39$ connected sets of fixed points, and of which $4!=24$ are fixed points that are stable, corresponding to
the possible permutations in which each of $I_{\mu}^{+},I_{\mu}^{-},I_{\nu}^{+},I_{\nu}^{-}$
are singletons. 
\end{prop}

\begin{proof}
 Evidently, any fixed point must have $R_{ij}^{\perp}=0$ for all $i,j$. Furthermore, the point $v_i = m_i^\mu = m_i^\nu =0$ for $i=1,...,K$ evidently forms a fixed point of the system. Now suppose there is some fixed point with $v_i =0$ for some $i$; in that case, it must be that  $m_i^\mu=0$ and $m_i^\nu=0$. Therefore, we can select a subset $I_0$ of $\{1,...,K\}$ such that $v_i= m_i^\mu =m_i^\nu=0$ for $i\in I_0$. 
 
 For any such choice of $I_0$, consider next, $i\notin I_0$. We first claim that if $v_i > 0$ at a fixed point, then $m_i^\mu \in\{\pm v_i\}$ and $m_i^\nu =0$, whereas if $v_i<0$ then $m_i^\nu \in \{\pm v_i\}$ and $m_i^\mu=0$. To see this, notice that at any fixed point, 
 \begin{align*}
 	4 \alpha m_i^\mu &  = v_i \Big(\mathbf{1}_{m_{i}^{\mu}\ge0}\sigma(-v\cdot g(m^{\mu}))-\mathbf{1}_{m_{i}^{\mu}<0}\sigma(-v\cdot g(-m^{\mu}))\Big)\,,\\
	4\alpha m_i^\nu & =  -v_i \Big(\mathbf{1}_{m_{i}^{\nu}\ge0}\sigma(-v\cdot g(m^{\nu}))-\mathbf{1}_{m_{i}^{\nu}<0}\sigma(-v\cdot g(-m^{\nu}))\Big)\,.
 \end{align*}
Since $\sigma$ is non-negative, if $v_i>0$, the sign of the right-hand side of the first equation is the same as the sign of $m_i^\mu$ so it can have a non-zero solution, while the sign of the right-hand side of the second equation is the opposite of the sign of $m_i^\nu$, so any such fixed point must have $m_i^\nu=0$. To see that $m_i^\mu =\pm v_i$ at such a fixed point, now set $m_i^\nu=0$ and take the fixed point equations for $v_i$ and $m_i^\mu$, dividing one by $v_i$ and the other by $m_i^\mu$ to see that 
\begin{align*}
	4\alpha \frac{v_i}{m_i^\mu} = 4\alpha \frac{m_i^\mu}{v_i}\,, \qquad \mbox{or} \qquad v_i^2 = (m_i^\mu)^2\,,
\end{align*}
as claimed. The fixed points having $v_i<0$ are solved symmetrically. 

Our classification now reduces to understanding the possible values taken by $(v_1,...,v_K)$ given their signs (when non-zero). Fix a partition $(I_0,I_\mu^+,I_\mu^-,I_\nu^+,I_\nu^-)$ of $\{1,...,K\}$ and consider the set of fixed points having $m_i^\mu = m_i^\nu =v_i=0$ for $i\in I_0$, $m_i^\mu =v_i>0$ on $I_\mu^+$ and so on as designated by Proposition~\ref{prop:fixpoints-xor}; by the above any fixed point is of this form. It remains to check that the values of $v_i$ on each of these sets are as described by the proposition. 

In order to see this, fix e.g., $i\in I_\mu^+$. Then, $m_i^\mu = v_i$ and $m_i^\nu =0$, and so the fixed point equations reduce to 
\begin{align*}
	4\alpha v_i = v_i \sigma(-v\cdot g(m^\mu))\,, \qquad \mbox{or} \qquad 4\alpha = \sigma\Big( - \sum_{j\in I_\mu^+} v_j^2\Big)\,,
\end{align*}
since the only coordinates where $g(m^\mu)$ will be non-zero are $j\in I_\mu^+$, where $m_j^\mu = v_j$. Inverting the sigmoid function, this implies exactly the claimed $\sum_{j\in I_\mu^+} v_j^2 = \mbox{logit}(-4\alpha)$. The cases of $I_\mu^-,I_\nu^+,I_\nu^-$ are analogous, concluding the proof.

{The count of the number of connected components of fixed points this forms is sensitive to $K$, so for concreteness let us perform it when $K=4$. We first notice that the fixed point at $(0,...,0)$ is disconnected from all others. Fixed points corresponding to some $(I_0,...,I_\nu^-)$ are part of the same connected component of fixed points if one goes from one to the other by moving an element of $I_{\iota}^\eta$ (for some $\iota \in \{\mu,\nu\}$ and $\eta\in \{\pm\}$ to $I_0$ without making $I_\iota^\eta$ empty, or by moving an element of $I_0$ to a non-empty $I_\iota^\eta$.

We turn now to studying the stability of these various sets of fixed points. Observe that in the $\lambda\to\infty$ limit, the dynamical system of Proposition~\ref{prop:xor-ballistic-ode} is a gradient system for the population loss
\begin{align*}
	\Phi & =  \frac{1}{4} \Big(\log (1+e^{ - v\cdot g(m^\mu)}) + \cdots + \log (1+e^{ - v\cdot g(-m^\nu)})\Big) \\
	& \quad + \frac{\alpha}{2} \sum_{i} (v_i^2 + (m_i^\mu)^2 + (m_i^\nu)^2 + R_{ii}^\perp)\,.
\end{align*}
At a fixed point (which necessarily has $v_i = m_i$, $R_{ii}^\perp = 0$, and is characterized by the partition of $\{1,...,4\}$ into $I_\mu^+,I_\mu^-,I_\nu^+,I_\nu^-$, this reduces to 
\begin{align*}
	\Phi = \frac{1}{4} \Big( \log(1+e^{ - \sum_{i\in I_\mu^+} v_i^2}) + \cdots + \log (1+e^{ - \sum_{i\in I_\nu^-} v_i^2})\Big) + \alpha \sum_{i} v_i^2
\end{align*}
At this point, noticing that $\sum_{i\in I_\mu^+} v_i^2$ is equal to $C_\alpha = -\mbox{logit}(4\alpha)$ if $I_\mu^+$ is non-empty and $0$ if it is empty, and similarly for $I_\mu^-,I_\nu^+,I_\nu^-$, this turns into a simple optimization problem over the number of non-empty $I_\mu^+,I_\mu^-,I_\nu^+,I_\nu^-$. 
Just as in the binary GMM case, it becomes evident that when $\alpha>1/8$, this is minimized at $v_i =0$ for all $i$ (i.e., they are all empty and $I_0 = \{1,...,4\}$, whereas when $\alpha<1/8$ the above is minimized when every one of $I_\mu^+,I_\mu^-,I_\nu^+,I_\nu^-$ are all non-empty. This yields the global minima of $\Phi$ in these coordinates, and ensures the fixed points we claimed were stable are indeed stable. 

To show the instability of any other connected set of fixed points, the reasoning goes just as in the binary GMM case: consider a small perturbation of the specified critical region in the direction of the stable fixed points and it can be seen by examining the drifts directly, that the dynamical system has a repelling direction.}
\end{proof}

\begin{rem}\label{rem:overparametrization-modifications}
When $K>4$, the counting of connected components of fixed points of course changes. However, what is still clear by an identical calculation is that the sets of fixed points minimizing $\Phi$ will still be $(0,...,0)$ when $\alpha>1/8$ and will be all fixed points that have all four of $I_\mu^+,...,I_\nu^-$ being non-empty if $\alpha<1/8$. Notice that when $\alpha<1/8$ and  $K>4$, even the set of stable fixed points become connected to form a single stable manifold.  
\end{rem}

\subsection{\sfrac{3}{32}-probability of ballistic convergence to an optimal classifier}\label{subsec:K=4-success-prob}
We now reason that when $K=4$ the ballistic effective dynamics of Proposition~\ref{prop:xor-ballistic-ode} is such that under an uninformative Gaussian initialization, the probability of being in a basin of attraction of one of the 24 stable fixed points is $3/32$. Begin by noticing that if the first layer weights are initialized as $W_i \sim \mathcal N(0,I_N/N)$ independently for $i=1,...,4$ and the second layer weights $v_i$ are independent standard Gaussians, then the projection onto the coordinate system $(v_i, m_i^\mu, m_i^\nu, R_{ij})$ is given by $$ \lim (\mathbf{u}_n)_*\mu_n =  \mathcal N(0,1)^{\otimes 4} \otimes \delta_0^{\otimes 4} \otimes \delta_0^{\otimes 4}\otimes \delta_{I_4}$$
{{The $\delta$-functions at zero for $m_i^\mu, m_i^\nu$ however cause some trouble because of the indicator functions on the sign of $m_i^\mu$ and $m_i^\nu$ in the equations of Proposition~\ref{prop:xor-ballistic-ode}. 

In order to handle this, we can instead consider the pre-limit as a mixture (over all the possible signings $\epsilon_i^\mu,\epsilon_i^\nu\in \{-1,+1\}$ of $m_i^\mu, m_i^\nu$) of initializations where $m_i^\mu \sim \epsilon_i^\mu |Z|$ for $Z$ being Gaussian of variance $1/N$. For any such signing $\mathbf{\epsilon}$, we take the limit per Proposition~\ref{prop:xor-ballistic-ode} to obtain the limiting ODE's with the indicators taking their values corresponding to the signings $\mathbf{\epsilon}$. Thus the limit with the Gaussian initialization can be thought of as the equal mixture over the same signings $\mathbf{\epsilon}$ of the various ODE's obtained from Proposition~\ref{prop:xor-ballistic-ode} with the various indicators taking values $0$ or $1$.  With that in mind, can interpret the initial $m_i^\mu(0),m_i^\nu(0)$ as random variables that take values $0^+$ and $0^-$ with probability $1/2$ each, the superscript being the signing dictating which indicator should be $1$.}}  

Under the flow of Proposition~\ref{prop:xor-ballistic-ode}, if $v_i(0)$ is positive, then $m_i^\nu$ stays fixed at zero, and if $m_i^\mu(0) = 0^-$ then $m_i^\mu$ becomes negative infinitesimally quickly, whereas if $m_i^\mu(0) = 0^+$ then it becomes positive infinitesimally quickly. At any rate, the sign of $v_i$ never changes to negative from such an initialization, and similarly if $v_i(0)$ is negative, the sign of $v_i$ will never change to positive. As such, in order to have  a chance at being in the basin of attraction of one of the stable fixed points outlined in Proposition~\ref{prop:fixpoints-xor}, it must be the case that two of $(v_i(0))_i$ have positive sign and two of them have negative sign; evidently this has probability $\binom{4}{2}/2^4 = 3/8$. 

Given that two of $v_i(0)$ are positive, and two of them are negative---say without loss of generality that $i=1,2$ are the coordinates in which it is positive, and $i=3,4$ are the coordinates in which it is negative---then the dynamical system for $(v_1,v_2,m_1^\mu,m_2^\mu)$ is exactly the ballistic limit of the two-layer GMM studied in Section~\ref{sec:binary-gmm}, for which we found that the probability of converging to a good classifier is $1/2$. Similarly, the dynamical system for $(v_2,v_4,m_3^\nu,m_4^\nu)$ independently gives a further probability $1/2$ of converging to \emph{its} good classifier. Together, these yield a probability of $3/32$ of converging to one of the $4!$ many optimal classifiers for the XOR GMM. 

\begin{rem}\label{rem:overparametrization-success-probability}
Generically, if $K \ge 4$, by a similar reasoning to the above, in order to fall in the basin of attraction of the stable fixed points, it must be the case that the initialization has some four indices each of which initially belong to $I_\mu^+,I_\mu^-,I_\nu^+,I_\nu^-$. This is the probability that $v_i(0)$ are positive for at least two indices, and negative for at least two indices, and then among the indices at which $v_i(0)$ is positive, there is at least one index where $m_i^\mu$ is positive and one where it is negative, and similarly with $v_i(0)$ negative and $m_i^\nu$. Doing this combinatorial calculation out, we find that the probability of being in a good initialization is exactly the expression in~\eqref{subsec:overparametrization-xor}. This is easily seen to go to $1$ exponentially fast as $K\to\infty$ since the initial choice of $v_i(0)$'s will typically have around $K/2$ positive and $K/2$ negative coordinates, and with exponentially high probability those will have both positive and negative $m_i^\mu$ and $m_i^\nu$. 
\end{rem}

\subsection{Diffusive limit on critical submanifolds}
We now consider scaling limits of the rescaled effective dynamics in their noiseless limit, where the rescaling is about the unstable set of fixed points given by the product of two quarter circles where $I_\mu^+ = \{1,2\}$ and $I_\nu^+ = \{3,4\}$ (if $K>4$, examine the fixed point in which all coordinates after the first four are in $I_0$). In what follows, fix $(a_{1,\mu},a_{2,\mu})\in \mathbb R_+^2$ with $a_{1,\mu}^2 + a_{2,\mu}^2 =C_\alpha$, and $a_{3,\nu}^2 + a_{4,\nu}^2 = C_\alpha$, and  let $\bu_n$ be the variables of~\eqref{eq:u-bgmm} with $v_i,m_i^\mu,m_i^\nu$ replaced by 
\begin{align*}
\tilde{v}_{i}= & \begin{cases}
\sqrt{N}(v_{i}-a_{i,\mu}) & i=1,2\\
 - \sqrt{N}(v_{i}-a_{i,\nu}) & i=3,4
\end{cases}
\end{align*}
and
\begin{align*}
\tilde{m}_{i}^\mu=  \begin{cases}
\sqrt{N}(m_{i}^\mu-a_{i,\mu}) & i=1,2\\
0 & i=3,4
\end{cases}\,, \qquad
\tilde{m}_{i}^\nu=  \begin{cases}
0 & i=1,2\\
\sqrt{N}(m_{i}^\nu-a_{i,\nu}) & i=3,4
\end{cases}\,.
\end{align*}
By the choices of $\tilde m_i^\mu =0$ and $\tilde m_i^\nu=0$, we mean that we formally mean that we remove those variables from $\tilde \bu_n$, and for us now $E_K$ will be the ball of radius $K$ in the other coordinates, and the point $\{0\}$ for $(\tilde m_i^\mu)_{i=3,4}$ and $(\tilde m_i^\nu)_{i=1,2}$. 

\begin{proof}[\textbf{\emph{Proof of Proposition~\ref{prop:XOR-diffusive-noiseless}}}]
	The fact that the rescaled variables $\tilde \bu_n$ satisfy the conditions of Theorem~\ref{thm:main} follows as in Lemma~\ref{lem:xor-localizable} with the only distinction arising in the bound on~\eqref{eq:bgmm-localizable-item-3}, where previously we did not use the $\delta_n^2$ factor, but is still satisfied using $\delta_n = O(1/n)$. 
	
	We next consider the population drift of the new variables $\tilde v_i, \tilde m_i^\mu$ and $\tilde m_i^\nu$. If we take these variables to be in $E_K$, and recall the population drifts etc.\ in the $\lambda = \infty$ setting from Proposition~\ref{prop:xor-ballistic-ode},  for $i=1,2$, we have $f_{\tilde v_i}$ is the $n\to\infty$ limit of 
	\begin{align*}
		 \sqrt{N} \frac{m_i^\mu}{4}\sigma( - v\cdot g(m^\mu)) - \sqrt{N} \alpha v_i
	\end{align*}
	If we then use the expansion 
	\begin{align*}
		v\cdot g(m^\mu) &  = C_\alpha + N^{-1/2} \sum_{j=1,2} a_{j,\mu} (\tilde v_j + \tilde m_j^\mu) + O(1/n) 
	\end{align*}
	from which we obtain 
	\begin{align*}
	\sigma( - v\cdot g(m^\mu)) & = \sigma( - C_\alpha) + \frac{1}{\sqrt{N}}\Big( \sum_{j=1,2} a_{j,\mu} (\tilde v_j + \tilde m_j^\mu)  \Big)(4\alpha)(1-4\alpha) + O(\tfrac{1}{n}) 
	\end{align*}
	Plugging these in, and taking the $n\to\infty$ limit we find that for $i=1,2$, 
	\begin{align*}
		f_{\tilde v_i} = \alpha(\tilde v_i - \tilde m_i^\mu) - a_{i,\mu} (\alpha - 4\alpha^2) \sum_{k=1,2} a_{k,\mu}(\tilde v_k + \tilde m_k^\mu)\,.
	\end{align*}
	By a similar reasoning, for $i=3,4$, we have 
	\begin{align*}
		f_{\tilde v_i} = \alpha(\tilde v_i - \tilde m_i^\nu) - a_{i,\nu} (\alpha - 4\alpha^2) \sum_{k=3,4} a_{k,\nu}(\tilde v_k + \tilde m_k^\nu)\,.
	\end{align*}
	The claimed equations for $f_{\tilde m_i^\mu}$ when $i=1,2$ and $f_{\tilde m_i^\nu}$ when $i=3,4$ hold by analogous reasoning, and the equations for $f_{R_{ij}^{\perp}}$ are evidently unaffected by the change of variables to $\tilde \bu_n$. Regarding the population correctors, they are also unaffected (all zero) since the variables that were changed in $\tilde\bu_n$ are all linear. 
	
	It remains to compute the volatility matrix in the coordinates $v_i, \tilde m_i^\mu,\tilde m_i^\nu$. We first use the following expression for the matrix $V$ when $\lambda = \infty$, by taking $\lambda=\infty$ in~\eqref{eq:bgmm-V}. If $i,j\in \{1,2\}$, then 
	\begin{align*}
		V_{v_i,v_j} = \begin{cases} \frac{3}{16} m_i^\mu m_j^\mu \sigma( - v\cdot m^\mu)^2  & i,j\in \{1,2\} \\ 
					\frac{3}{16} m_i^\nu m_j^\nu \sigma(  v\cdot m^\nu)^2 & i,j\in \{3,4\}	\end{cases}
	\end{align*}
	and if $i\in \{1,2\}$ and $j\in \{3,4\}$, then 
	\begin{align*}
		V_{v_i,v_j} =  - \frac{1}{16} m_i^\mu m_j^\nu \sigma( - v\cdot m^\mu) \sigma(v\cdot m^\nu)
	\end{align*}
	When considering $\Sigma_{v_i,v_j}$ we multiply this by $N$ coming from $\tilde J$ and $\tilde J^T$, but also multiply by $\delta = 1/N$, so that taking the limit as $n\to\infty$, we get
\begin{align*}
	\tilde \Sigma_{v_i,v_j} = \begin{cases} 3\alpha^2  a_{i,\mu} a_{j,\mu} & i,j\in \{1,2\} \\ 
					3\alpha^2  a_{i,\nu} a_{j,\nu}  & i,j\in \{3,4\} \\ 
					 - 3\alpha^2 a_{i,\mu} a_{j,\nu} & i\in \{1,2\}, j\in \{3,4\}	\end{cases}\,.
\end{align*}
	By a similar reasoning, if $i,j\in \{1,2\}$, then 
	\begin{align*}
		V_{v_i,W_j} \cdot \mu =  \frac{3}{16} v_j m_i^\mu \sigma( - v\cdot m^\mu)^2 \qquad i,j\in \{1,2\}  \\
		V_{v_i,W_j} \cdot \nu = \frac{3}{16} v_j m_i^\nu \sigma(  v\cdot m^\nu)^2 \qquad  i,j\in \{3,4\}
	\end{align*}
	and if $i\in \{1,2\}$ and $j\in \{3,4\}$, then 
	\begin{align*}
		V_{v_i, W_j} \cdot \nu =  - \frac{1}{16} v_j m_i^\mu \sigma(- v\cdot m^\mu) \sigma(v\cdot m^\nu)\,.
	\end{align*}
	Taking the limit as $n\to\infty$, we again recover the claimed limiting diffusion matrix, and similar calculations yield the same for $\Sigma_{\tilde m_i^\mu,\tilde m_j^\mu}$, $\Sigma_{\tilde m_i^\nu,\tilde m_j^\nu}$ and $\Sigma_{\tilde m_i^\mu,\tilde m_j^\nu}$, concluding the proof. 
	\end{proof}

	\section{Proofs of technical lemmas for Gaussian mixtures}\label{sec:technical-lemmas}
	In this section, we establish the technical bounds on Gaussian moments in Lemmas~\ref{lem:bgmm-A-bound}--\ref{lem:large-lambda-est}.
	 \begin{proof}[\textbf{\emph{Proof of Lemma~\ref{lem:bgmm-A-bound}}}]
For the first bound, let $Z\sim \cN(0,I)$ and consider 
\begin{align*}
	\mathbb E[ |X\cdot w|^8] = \frac{1}{2} \mathbb E[(w\cdot \mu + \lambda^{-1/2}w\cdot Z)^8]  + \frac{1}{2}  \mathbb E[(-w\cdot \mu + \lambda^{-1/2}w\cdot Z)^8] \,.
\end{align*}
The quantities in the expectations are at most some universal constant times $(w\cdot \mu)^8 + \lambda^{-4} (w\cdot Z)^8$. 
To bound the expectation of the second term here, notice that $w\cdot Z$ is distributed as $\cN(0,\|w\|^2)$ implying the desired. 

The bound on $\mathbf{A}_i$ goes as follows. Evidently it suffices to let $X_\mu = \mu+\lambda^{-1/2} Z$ for $Z \sim \cN(0,I)$, and prove the bound on the norm of  
\begin{align*}
	\mathbb E[ X_\mu \mathbf 1_{W_i \cdot X_\mu \ge 0} (-1 + \sigma(g(W X_\mu)))] = \mathbb E[ (\mu + \lambda^{-1/2} Z) \mathbf 1_{W_i \cdot X_\mu \ge 0} (-1 + \sigma(g(W X_\mu)))]\,.
\end{align*}
Now decompose $Z$ as $Z_\mu \mu + Z_{1,\perp} W_1^\perp + Z_{2,\perp} W_2^\perp + Z_{3}$, where $Z_\mu\sim \cN(0,1)$ is independent of $(Z_{1,\perp},Z_{2,\perp})$ which is distributed as $\cN(0,A)$ with $A$ given by~\eqref{eq:Z-perp-covariance}, which is independent of $Z_3$ distributed as a standard Gaussian vector orthogonal to the subspace spanned by $(\mu, W_{1}^\perp, W_2^\perp)$. By independence of $Z_3$ from the indicator and the argument of the sigmoid, all those terms contribute nothing to the expectation, and therefore,  
\begin{align*}
	\|\mathbf{A}_i\|^2 \le \sum_{w\in \{\mu, W_1^\perp, W_2^\perp\}} \mathbb E [ (X\cdot w)^2 \mathbf 1_{W_i \cdot X\ge 0} (-y+\sigma(g(WX)))] \le (1+R_{11}^{\perp} + R_{22}^{\perp}) (1+\lambda^{-1})\,.
\end{align*}
Here, we used the first inequality of the lemma. This yields the desired. 
\end{proof}

\begin{proof}[\textbf{\emph{Proof of Lemma~\ref{lem:large-lambda-est}}}]
The proof of~\eqref{eq:large-lambda-prob} is easily seen by rewriting the probability in question as  
\[
\mathbb{P}(W_{i}\cdot X_{\mu}<0)=\mathbb{P}\big(\cN(0,\lambda^{-1}) <- m_{i}({m_{i}^{2}+R_{ii}^\perp})^{-1/2}\big)=e^{-m_{i}^{2}\lambda/2 (m_{i}^{2}+R_{ii}^\perp)}\,,
\]
so that as long as $m_{i}>0$ this goes to zero as $\lambda\to\infty$.

We turn to~\eqref{eq:large-lambda-sigmoid}. Consider 
\begin{align*}
\mathbb{E}\big[\big|\sigma(v\cdot g(WX_{\mu}))-\sigma(v\cdot g(m))\big|\big] &\le\mathbb{E}\Big[\big|e^{v\cdot g(WX_{\mu})}-e^{v\cdot g(m)}\big|\Big]\\
&\le\mathbb{E}\big[\big|e^{v_{1}g(W_{1}\cdot X_{\mu})}e^{v_{2}g(W_{2}\cdot X_{\mu})}-e^{v_{1}g(m_{1})}e^{v_{2}g(m_{2})}\big|\big]\,.
\end{align*}
This in turn is bounded by 
\begin{align}\label{eq:sigmoid-difference-two-terms}
\mathbb{E}\big[e^{v_{2}g(W_{2}X_{\mu})}\big|e^{v_{1}g(W_{1}X_{\mu})}-e^{v_{1}g(m_{1})}\big|\big]+e^{v_{1}g(m_{1})}\mathbb{E}\big[\big|e^{v_{2}g(W_{2}X_{\mu})}-e^{v_{2}g(m_{2})}\big|\big]\,.
\end{align}
Applying Cauchy--Schwarz to the first term, it suffices to establish
the following bounds
\begin{align*}
\mathbb{E}\big[e^{2v_{i}g(W_{i}X_{\mu})}\big]\le C\,, \qquad\mbox{and} \qquad \lim_{\lambda\to\infty}\mathbb{E}\big[\big(e^{v_{i}g(W_{i}X_{\mu})}-e^{v_{i}g(m_{i})}\big)^{2}\big]=0\,.
\end{align*}
To demonstrate the first of these inequalities, notice that 
\[
\mathbb{E}\Big[e^{2v_{i}g(W_{i}X_{\mu})}\Big]\le\mathbb{E}\Big[e^{2v_{i}|W_{i}X_{\mu}|}\Big]\le C\,.
\]
uniformly over $\lambda$, per Fact~\ref{fact:moment-bound}. 
For the second desired bound, expand $e^{v_{i}g(W_{i}\cdot X_{\mu})}-e^{v_{i}g(m_{i})}$ as 
\begin{align*}
 \big(e^{v_i (W_i\cdot X_\mu) \mathbf 1_{W_i\cdot X_\mu\ge 0}} - e^{ v_i(W_i\cdot X_\mu) \mathbf 1_{m_i \ge 0}}\big) + \big(e^{ v_i (W_i\cdot X_\mu) \mathbf 1_{m_i \ge 0}} - e^{ v_i m_i \mathbf 1_{m_i\ge 0}} \big)\,.
\end{align*}
It suffices to show the expectation of the square of each of these goes to zero as $\lambda\to\infty$. First, 
\begin{align*}
	\mathbb E\big[\big(e^{v_i (W_i\cdot X_\mu) \mathbf 1_{W_i\cdot X_\mu\ge 0}} - e^{ v_i (W_i \cdot X_\mu) \mathbf 1_{m_i \ge 0}}\big)^2\big] \le (1\vee e^{v_i (W_i\cdot X_\mu)}) \mathbb E[\mathbf 1_{W_i\cdot X_\mu \ge 0} - \mathbf 1_{m_i \ge 0}]\,.
\end{align*}
If $m_i \ne 0$, the expectation on the right goes to zero by~\eqref{eq:large-lambda-prob}. Second, 
\begin{align*}
	\mathbb E\big[\big(e^{ v_i (W_i\cdot X_\mu) \mathbf 1_{m_i \ge 0}} - e^{ v_i m_i \mathbf 1_{m_i\ge 0}} \big)^2\big] \le \mathbb E\big[(e^{v_i (W_i \cdot X_\mu)} - e^{ v_i m_i})^2 \mathbf 1_{m_i\ge 0}\big]\,.
\end{align*}
When $m_i <0$, this is evidently zero; when $m_i >0$, if $G_\lambda \sim \cN(0,I/\lambda)$, this is 
\begin{align*}
	e^{2v_i  m_i} \mathbb E\big[ (e^{v_i (W_i\cdot G_\lambda)} -1)^2\big]\,.
\end{align*}
which goes to zero as $O(\lambda^{-1})$ when $\lambda \to\infty$, by the explicit formula for the moment generating function of the Gaussian $W_i \cdot G_\lambda$, whose variance is $(m_i^2 + R_{ii}^\perp) \lambda^{-1}$.  
\end{proof}

\subsection*{Acknowledgements}
The authors thank the anonymous referees for their useful comments and suggestions. The authors thank F.\ Krzakala, L.\ Zdeborova, and B.\ Loureiro for interesting conversations and suggestions, especially suggesting we investigate the role of overparametrization in the XOR GMM. The authors thank M.\ Sellke for pointing out the relationship to the lottery ticket hypothesis. The authors also thank M.\ Glasgow for a careful reading and helpful suggestions. R.G.\ acknowledges the support of NSF DMS-2246780 and the Miller Institute for Basic Research in Science. A.J.\ acknowledges the support of the Natural Sciences and Engineering Research Council of Canada (NSERC) and the Canada Research Chairs programme. Cette recherche a \'et\'e enterprise gr\^ace, en partie, au 
soutien financier du Conseil de recherches en sciences naturelles et en g\'enie du Canada (CRSNG),  [RGPIN-2020-04597, DGECR-2020-00199], et du Programme des chaires de recherche du Canada.

\bibliographystyle{plain}
\bibliography{Effective-dynamics.bib}
\end{document}